%% file: iclr2021_conference.tex
\documentclass{article} % For LaTeX2e
\usepackage{iclr2021_conference,times}

\usepackage[colorlinks=true, linkcolor=blue, citecolor=blue,urlcolor=black]{hyperref}      
\usepackage{url}
\usepackage{graphicx}
\usepackage{wrapfig}

\title{Linear Last-iterate Convergence in Constrained Saddle-point Optimization}

\author{%
  \textbf{Chen-Yu Wei, Chung-Wei Lee, Mengxiao Zhang, Haipeng Luo} \\
  University of Southern California \\
  \texttt{\{chenyu.wei,leechung,mengxiao.zhang,haipengl\}@usc.edu}
}

\iclrfinalcopy % Uncomment for camera-ready version, but NOT for submission.
\input{header}
\input{commands}

\begin{document}

\maketitle

\begin{abstract}

Optimistic Gradient Descent Ascent (\alg) and Optimistic Multiplicative Weights Update (\OMWU) for saddle-point optimization have received growing attention due to their favorable last-iterate convergence.
However, their behaviors for simple bilinear games over the probability simplex are still not fully understood --- previous analysis lacks explicit convergence rates, only applies to an exponentially small learning rate, %~\citep{daskalakis2019last},
or requires additional assumptions such as the uniqueness of the optimal solution.
%while for \alg, to our knowledge, there is no linear last-iterate convergence result for constrained bilinear games.

In this work, we significantly expand the understanding of last-iterate convergence for \alg and \OMWU in the constrained setting.
Specifically, for \OMWU in bilinear games over the simplex, we show that when the equilibrium is unique, \emph{linear} last-iterate convergence is achieved with a learning rate \major{whose value is set to a \emph{universal constant}, } 
improving the result of \citep{daskalakis2019last} under the same assumption. %(generalization for \OMWU on simplexes can be added here)
We then significantly extend the results to more general objectives and feasible sets for the projected \alg algorithm, by introducing a sufficient condition under which \alg exhibits concrete last-iterate convergence rates with a \major{constant} learning rate \major{whose value only depends on the smoothness of the objective function. }
We show that bilinear games over {\it any polytope} satisfy this condition and \alg converges exponentially fast even {\it without the unique equilibrium assumption}.
Our condition also holds for strongly-convex-strongly-concave functions, recovering the result of \citep{hsieh2019convergence}.
Finally, we provide experimental results to further support our theory. %Interestingly, for bilinear games on probability simplex, although \OMWU is often considered as a better algorithm for regret minimization, we observe that \alg often leads to faster last-iterate convergence, which would be interesting for further investigation.  

%Optimistic Gradient Descent Ascent (\alg) algorithm for saddle-point optimization has received growing attention due to its favorable last-iterate convergence. 
%However, its behavior for simple two-player matrix games is still not fully understood ---
%previous analysis lacks explicit convergence rates, only applies to exponentially small learning rate, or requires additional conditions such as uniqueness of the optimal solution. 
%
%In this work, we significantly expand the understanding of \alg, introducing a set of sufficient conditions under which \alg exhibits concrete last-iterate convergence rates with a constant learning rate.
%Specifically, we show that matrix games satisfy these conditions and \alg converges exponentially fast without any additional assumptions.
%More generally, our conditions hold for smooth bilinear functions and strongly-convex-strongly-concave functions over a constrained set.
%We provide experimental results to further support our theory.

%To further demonstrate the significance of our results for matrix games, 
%we greatly generalize the ideas to finite-horizon stochastic/Markov games and provide the first algorithm that simultaneously ensures 1) linear last-iterate convergence when playing against itself and 2) low regret when playing against an arbitrary slowly-changing opponent.
\end{abstract}
\input{relatedwork}
\input{preliminaries}
\input{algorithm}
\input{experiments-main}
%\input{sto-game}
%\input{general_results}

%\section*{Broader Impact}
%This work is mostly theoretical, and we do not foresee any negative ethical or societal outcomes.
%Researchers working on optimization in general might benefit from our work and find our techniques useful for other problems.
%More broadly, our results might lead to better algorithms and understanding for training GANs.

\subsubsection*{Acknowledgments}
\major{The authors would like to thank the anonymous reviewers for providing highly constructive comments which bring about significant improvement of the result during the rebuttal phase.} CL would like to thank Yu-Guan Hsieh for many helpful discussions on error bounds and metric subregularity. The authors are supported by NSF Awards IIS-1755781 and IIS-1943607. 

\bibliography{ref}
\bibliographystyle{iclr2021_conference}

\newpage
\appendix

\input{appendix}

\end{document}

%% file: header.tex
\usepackage{prettyref}
\newcommand{\pref}[1]{\prettyref{#1}}

\newcommand{\savehyperref}[2]{\texorpdfstring{\hyperref[#1]{#2}}{#2}}
\newrefformat{eq}{\savehyperref{#1}{Eq. \textup{(\ref*{#1})}}}
\newrefformat{eqn}{\savehyperref{#1}{Equation~\ref*{#1}}}
\newrefformat{lem}{\savehyperref{#1}{Lemma~\ref*{#1}}}
\newrefformat{def}{\savehyperref{#1}{Definition~\ref*{#1}}}
\newrefformat{line}{\savehyperref{#1}{Line~\ref*{#1}}}
\newrefformat{thm}{\savehyperref{#1}{Theorem~\ref*{#1}}}
\newrefformat{corr}{\savehyperref{#1}{Corollary~\ref*{#1}}}
\newrefformat{cor}{\savehyperref{#1}{Corollary~\ref*{#1}}}
\newrefformat{sec}{\savehyperref{#1}{Section~\ref*{#1}}}
\newrefformat{app}{\savehyperref{#1}{Appendix~\ref*{#1}}}
\newrefformat{assum}{\savehyperref{#1}{Assumption~\ref*{#1}}}
\newrefformat{ex}{\savehyperref{#1}{Example~\ref*{#1}}}
\newrefformat{fig}{\savehyperref{#1}{Figure~\ref*{#1}}}
\newrefformat{alg}{\savehyperref{#1}{Algorithm~\ref*{#1}}}
\newrefformat{rem}{\savehyperref{#1}{Remark~\ref*{#1}}}
\newrefformat{conj}{\savehyperref{#1}{Conjecture~\ref*{#1}}}
\newrefformat{prop}{\savehyperref{#1}{Proposition~\ref*{#1}}}
\newrefformat{proto}{\savehyperref{#1}{Protocol~\ref*{#1}}}
\newrefformat{prob}{\savehyperref{#1}{Problem~\ref*{#1}}}
\newrefformat{claim}{\savehyperref{#1}{Claim~\ref*{#1}}}
\newrefformat{que}{\savehyperref{#1}{Question~\ref*{#1}}}
\newrefformat{op}{\savehyperref{#1}{Open Problem~\ref*{#1}}}

%% file: commands.tex
\usepackage{xspace}
\usepackage{amsmath}
\usepackage{amsthm}
\usepackage{bbm}
\usepackage{algorithm}
\usepackage[algo2e, vlined, ruled]{algorithm2e}
\usepackage{algorithmic}
\usepackage{mathrsfs}
\usepackage{amssymb}
\usepackage{bm}
\usepackage{makecell}
\usepackage{tabulary}
\usepackage{xcolor}

\definecolor{Green}{rgb}{0.13, 0.65, 0.3}

\usepackage[]{color-edits}
\addauthor{HL}{Green}

\DeclareMathOperator*{\argmax}{argmax}
\DeclareMathOperator*{\argmin}{argmin}
\newcommand{\KL}{\text{\rm KL}}

\newcommand{\calT}{\mathcal{T}}

\newcommand{\calV}{\mathcal{V}}
\newcommand{\calA}{\mathcal{A}}

\newcommand{\calZ}{\mathcal{Z}}

\newcommand{\calS}{\mathcal{S}}
\newcommand{\calM}{\mathcal{M}}

\newcommand{\order}{\mathcal{O}}

\newcommand{\calP}{\mathcal{P}}

\newcommand{\calX}{\mathcal{X}}
\newcommand{\calY}{\mathcal{Y}}

\newcommand{\inner}[1]{\langle#1\rangle}
\newcommand{\rsi}{\text{SP-MS}\xspace}
\newcommand{\der}{\mathrm{d}}
\newcommand{\typeone}{\text{\rsi}\xspace}
\newcommand{\typetwo}{\text{\rsi}\xspace}
\newcommand{\rsilong}{\text{Saddle-Point Metric Subregularity}\xspace}

\newcommand{\hatv}{\widehat{\bv}}
\newcommand{\Holder}{{H{\"o}lder}\xspace}

\newcommand{\calN}{\mathcal{N}}
\newcommand{\supp}{\text{supp}}

\newcommand{\dist}{\mathrm{dist}^2}
\newcommand{\distp}[1]{\mathrm{dist}_{#1}^2}

\newcommand{\xp}{\widehat{\bm{x}}}
\newcommand{\yp}{\widehat{\bm{y}}}
\newcommand{\zp}{\widehat{\bm{z}}}
\newcommand{\xpp}{\widehat{x}}
\newcommand{\ypp}{\widehat{y}}
\newcommand{\zpp}{\widehat{z}}
\newcommand{\alg}{\textsc{OGDA}\xspace}
\newcommand{\EG}{\textsc{EG}\xspace}
\newcommand{\OMWU}{\textsc{OMWU}\xspace}
\newcommand{\MWU}{\textsc{MWU}\xspace}
\newcommand{\GDA}{\textsc{GDA}\xspace}

\newcommand{\modify}[1]{{#1}}
\newcommand{\major}[1]{{#1}}
\newcommand{\minor}[1]{#1}
\newcommand{\zero}{\bm{0}}

\newcommand{\va}{\bm{a}}

\newcommand{\x}{\bm{x}}
\newcommand{\y}{\bm{y}}
\newcommand{\z}{\bm{z}}
\newcommand{\g}{\bm{g}}
\newcommand{\n}{\bm{n}}
\newcommand{\bu}{\bm{u}}
\newcommand{\bv}{\bm{v}}
\newcommand{\G}{\bm{G}}
\newcommand{\ai}{\bm{a}_{i}}
\newcommand{\aj}{\bm{a}_{j}}
\newcommand{\ci}{\bm{c}_{i}}

\newtheorem*{theorem*}{Theorem}
\newtheorem{theorem}{Theorem}
\newtheorem{assumption}{Assumption}
\newtheorem{lemma}[theorem]{Lemma}

\newtheorem*{lemma*}{Lemma}
\newtheorem{definition}{Definition}

\theoremstyle{definition}

%% file: relatedwork.tex
%!TEX root=iclr2021_conference.tex
\section{Introduction}

Saddle-point optimization in the form of $\min_{\x} \max_{\y} f(\x,\y)$ dates back to \citep{neumann1928theorie}, where the celebrated minimax theorem was discovered.
Due to advances of Generative Adversarial Networks (GANs)~\citep{goodfellow2014generative} (which itself is a saddle-point problem), the question of how to find a good approximation of the saddle point, especially via an efficient iterative algorithm, has recently gained significant research interest.
Simple algorithms such as Gradient Descent Ascent (\GDA) and Multiplicative Weights Update (\MWU) are known to cycle and fail to converge even in simple bilinear cases (see e.g.,~\citep{bailey2018multiplicative} and \citep{cheung2019vortices}).

Many recent works consider resolving this issue via simple modifications of standard algorithms, usually in the form of some extra gradient descent/ascent steps.
This includes Extra-Gradient methods (\EG)~\citep{liang2019interaction,mokhtari2019unified}, Optimistic Gradient Descent Ascent (\alg)~\citep{daskalakis2018training,gidel2018variational,mertikopoulos2019optimistic}, Optimistic Multiplicative Weights Update (\OMWU)~\citep{daskalakis2019last,Lei2020Last}, and others. In particular, \alg and \OMWU are suitable for the repeated game setting where two players repeatedly propose $\x_t$ and $\y_t$ and receive only $\nabla_{\x} f(\x_t, \y_t)$ and $\nabla_{\y} f(\x_t, \y_t)$ respectively as feedback, with the goal of converging to a saddle point or equivalently a Nash equilibrium using game theory terminology.
\modify{One notable benefit of \alg and \OMWU is that they are also \emph{no-regret} algorithms with important applications in online learning, especially when playing against adversarial opponents~\citep{chiang2012online, rakhlin2013optimization}. 
}

Despite considerable progress, especially those for the unconstrained setting,
the behavior of these algorithms for the constrained setting, where $\x$ and $\y$ are restricted to \minor{closed convex sets} $\calX$ and $\calY$ respectively, is still not fully understood.
This is even true when $f$ is a bilinear function and $\calX$ and $\calY$ are simplex, known as the classic two-player zero-sum games in normal form, or simply {\it matrix games}.
Indeed, existing convergence results on the last iterate of \alg or \OMWU for matrix games are unsatisfactory --- 
they lack explicit convergence rates~\citep{popov1980modification, mertikopoulos2019optimistic}, only apply to exponentially small learning rate thus not reflecting the behavior of the algorithms in practice~\citep{daskalakis2019last},
or require additional conditions such as uniqueness of the equilibrium or a good initialization~\citep{daskalakis2019last}.

Motivated by this fact, in this work, we first improve the last-iterate convergence result of \OMWU for matrix games. Under the same unique equilibrium assumption as made by \citet{daskalakis2019last}, we show {\it linear} convergence with a concrete rate in terms of the Kullback-Leibler divergence between the last iterate and the equilibrium, using a learning rate \major{whose value is set to a \emph{universal constant}. }

We then significantly extend our results and consider \alg for general constrained and smooth convex-concave saddle-point problems, without the uniqueness assumption.
Specifically, we start with proving an average duality gap convergence of \alg at the rate of $\order(1/\sqrt{T})$ after $T$ iterations.
Then, to obtain a more favorable last-iterate convergence in terms of the distance to the set of equilibria, we propose a general sufficient condition on $\calX, \calY$, and $f$, 
called {\it \rsilong} (\rsi), under which we prove concrete last-iterate convergence rates, all with a constant learning rate and without further assumptions.

Our last-iterate convergence results of \alg greatly generalize that of~\citep[Theorem~2]{hsieh2019convergence}, which itself is a consolidated version of results from several earlier works.
The key implication of our new results is that, by showing that matrix games satisfy our \rsi condition, we provide  by far the most general last-iterate guarantee with a linear convergence for this problem using \alg. %\chenyu{``This problem'' is restricted to Matrix game + OGDA. Should we say this clearer to prevent overclaiming? }
Compared to that of \OMWU, the convergence result of \alg holds more generally even when there are multiple equilibria. % and it is achieved with any initialization. %does not suffer from the drawbacks of existing works mentioned earlier --- it 

More generally, the same linear last-iterate convergence holds for any bilinear games over polytopes since they also satisfy the \rsi condition as we show.
To complement this result, we construct an example of a bilinear game with a non-polytope feasible set where \alg provably does not ensure linear convergence, indicating that {\it the shape of the feasible set matters}.

Finally, we also provide experimental results to support our theory.
In particular, we observe that \alg generally converges faster than \OMWU for matrix games, despite the facts that both provably converge exponentially fast and that \OMWU is often considered more favorable compared to \alg when the feasible set is the simplex.

%Finally, to further showcase the significance of our results for matrix games, we greatly generalize the ideas to episodic stochastic games (a.k.a. Markov games).
%We show that using \alg at each state with appropriate feedback ensures both linear convergence to an equilibrium when playing against itself and at the same time a worst-case regret guarantee when playing against an arbitrary slowly-changing opponent.
%As far as we know, this is the first algorithm for stochastic games with these two properties simultaneously and both with concrete finite-time bounds.
%Our regret bound is also better than the recent work of~\citet{radanovic2019learning}, which considers a variant of \OMWU.

\section{Related Work}

\paragraph{Average-iterate convergence.}
While showing last-iterate convergence has been a challenging task, it is well-known that the average-iterate of many standard algorithms such as \GDA and \MWU enjoys a converging duality gap at the rate of $O(1/\sqrt{T})$~\citep{freund1999adaptive}.
A line of works show that the rate can be improved to $O(1/T)$ using the ``optimistic'' version of these algorithms such as \alg and \OMWU~\citep{rakhlin2013optimization, daskalakis2015near,syrgkanis2015fast}.
For tasks such as training GANs, however, average-iterate convergence is unsatisfactory since averaging large neural networks is usually prohibited. %and the problem is non-convex-concave.

\paragraph{Extra-Gradient (\EG) algorithms.}

The saddle-point problem fits into the more general variational inequality framework~\citep{harker1990finite}. 
A classic algorithm for variational inequalities is \EG, first introduced in~\citep{Korpelevich1976TheEM}.
%It is known that in the smooth convex-concave case, the optimal rate of averaged (ergodic) iterates convergence is $\mathcal{O}(1/T)$ \citep{nemirovski2004prox}.
\citet{tseng1995linear} is the first to show last-iterate convergence for \EG in various settings such as bilinear or strongly-convex-strongly-concave problems. 
Recent works significantly expand the understanding of \EG and its variants for unconstrained bilinear problems~\citep{liang2019interaction}, unconstrained strongly-convex-strongly-concave problems~\citep{mokhtari2019unified}, and more~\citep{zhang2019lower, lin2020near, Golowich2020Last}.

The original \EG is not applicable to a repeated game setting where only one gradient evaluation is possible in each iteration.
\modify{
Moreover, unlike \alg and \OMWU, \EG is shown to have linear regret against adversarial opponents, and thus it is not a no-regret learning algorithm~\citep{bowling2005convergence,golowich2020tight}.
}
However, there are ``single-call variants'' of EG that address these issues.
In fact, some of these versions coincide with the \alg algorithm under different names such as modified Arrow–Hurwicz method~\citep{popov1980modification} and ``extrapolation from the past''~\citep{gidel2018variational}.
\modify{
Apart from \alg, other single-call variants of \EG include Reflected Gradient ~\citep{malitsky2015projected,cui2016analysis,malitsky2020forward} and Optimistic Gradient~\citep{daskalakis2018training,mokhtari2019convergence}. 
These variants are all equivalent in the unconstrained setting but differ in the constrained setting.
To the best of our knowledge, none of the existing results for any single-call variant of \EG covers the constrained bilinear case (which is one of our key contributions).
}
\modify{
\paragraph{Error Bounds and Metric Subregularity}
To derive linear convergence for variational inequality problems, \emph{error bound} method is a commonly used technique \citep{pang1997error,luo1993error}. 
For example, it is a standard approach to studying the last-iterate convergence of EG algorithms \citep{tseng1995linear,hsieh2020explore,azizian2020tight}.
An error bound method is associated with an error function that gives every point in the feasible set a measure of sub-optimality that is lower bounded by the distance of the point to the optimal set up to some problem dependent constant.
If such a error function exists, linear convergence can be obtained. 
The choice of the error function depends on the feasible region, the objection function, and the algorithm.
Common error functions include natural residual functions \citep{iusem2017extragradient,malitsky2019golden} and gap functions \citep{larsson1994class,solodov2000some,chen2017accelerated}.
Our method to derive the last-iterate convergence for \alg can also be viewed as an error bound method.

\emph{Metric subregularity} is another important concept to derive linear convergence via some Lipschitz behavior of a set-valued operator~\citep{leventhal2009metric,liang2016convergence,alacaoglu2019convergence,latafat2019new}. 
Metric subregularity is closely related to error bound methods \citep{kruger2015error}. 
In fact, as we prove in \pref{app:ms}, one special case of our condition \rsi (that allows us to show linear convergence) is \emph{equivalent} to metric subregularity of an operator defined in terms of the normal cone of the feasible set and the gradient of the objective.
This is also the reason why we call our condition \rsilong. 
Although metric subregularity has been extensively used in the literature, 
to the best of our knowledge, our work is the first to use this condition to analyze \alg.
}

\paragraph{\alg and \OMWU.}
Recently, last-iterate convergence for \alg has been proven in various settings such as convex-concave problems~\citep{daskalakis2018training}, unconstrained bilinear problems~\citep{daskalakis2018limit,liang2019interaction}, strongly-convex-strongly-concave problems~\citep{mokhtari2019unified}, and others (e.g.~\citep{mertikopoulos2019optimistic}).

However, the behavior of \alg and \OMWU for the constrained bilinear case, or even the special case of classic matrix games, appears to be much more mysterious and less understood.
\modify{\citet{cheung2020chaos} provide an alternative view on the convergence behavior of \OMWU by studying volume contraction in the dual space. }%use volume analysis in Dynamical Systems to demonstrate the convergence of \OMWU by proving that it contracts the volume of the sets of the flows of starting points in dual spaces exponentially fast. 
\citet{daskalakis2019last} show last-iterate convergence of \OMWU for matrix games under a uniqueness assumption and without a concrete rate.
Although it is implicitly suggested in~\citep{daskalakis2019last,daskalakis2019} that a rate of $O(1/T^{1/9})$ is possible, it is still not clear how to choose the learning rate appropriately from their analysis. % and the bound has an exponential dependence on the size of the matrix.
As mentioned, our results for \OMWU significantly improve theirs, with a clean linear convergence rate using a constant learning rate under the same uniqueness assumption,
while our results for \alg further remove the uniqueness assumption. 

%% file: preliminaries.tex
\section{Notations and Preliminaries}
\label{sec:notation}

We consider the following constrained saddle-point problem:
$
     \min_{\x\in\calX} \max_{\y\in\calY} f(\x,\y),
$
where $\calX$ and $\calY$ are \minor{closed convex sets}, and $f$ is a continuous differentiable function that is convex in $\x$ for any fixed $\y$ and concave in $\y$ for any fixed $\x$.
%Denote the value of this optimization problem by $\rho_f$.
By the celebrated minimax theorem~\citep{neumann1928theorie},
we have $\min_{\x\in\calX} \max_{\y\in\calY} f(\x,\y) = \max_{\y\in\calY} \min_{\x\in\calX} f(\x,\y)$.

The set of minimax optimal strategy is denoted by $\calX^* = \argmin_{\x\in\calX}\max_{\y\in\calY} f(\x,\y)$,
and the set of maximin optimal strategy is denoted by $\calY^* = \argmax_{\y\in\calY}\min_{\x\in\calX} f(\x,\y)$.
It is well-known that $\calX^*$ and $\calY^*$ are convex, and any pair $(\x^*, \y^*) \in \calX^* \times \calY^*$ is a Nash equilibrium satisfying
$f(\x^*, \y) \leq f(\x^*, \y^*) \leq f(\x, \y^*)$ for any $(\x,\y) \in  \calX \times \calY$.

For notational convenience, we define $\calZ=\calX\times \calY$ and similarly $\calZ^*=\calX^*\times \calY^*$.
For a point $\z = (\x,\y) \in \calZ$, we further define $f(\z) = f(\x,\y)$ and $F(\z) = \left(\nabla_{\x} f(\x,\y), -\nabla_{\y} f(\x,\y)\right)$. 

Our goal is to find a point $\z \in \calZ$ that is close to the set of Nash equilibria $\calZ^*$, and
we consider three ways of measuring the closeness.
The first one is the {\it duality gap}, defined as
$
    \alpha_f(\z) = \max_{\y'\in \calY} f(\x, \y') - \min_{\x'\in \calX} f(\x', \y), 
$
which is always non-negative since $\max_{\y'\in\calY} f(\x,\y')\geq f(\x,\y) \geq \min_{\x'\in\calX} f(\x',\y)$.

The second one is the distance between $\z$ and $\calZ^*$.
Specifically, for any \minor{closed set} $\calA$, we define the projection operator $\Pi_{\calA}$ as $\Pi_{\calA}(\va)=\argmin_{\va'\in\calA} \|\va-\va'\|$ (throughout this work $\|\cdot\|$ represents $L_2$ norm).
The squared distance between $\z$ and $\calZ^*$ is then defined as
$
\dist(\z, \calZ^*) = \|\z - \Pi_{\calZ^*}(\z)\|^2.	
%\min_{\z'\in\Pi_{\calZ^*}(\z)}\distp{2}(\z, \z'),~\text{where $\distp{p}(\z,\z')=\|\x-\x'\|^2_p+\|\y-\y'\|^2_p$.}
$

The third one is only for the case when $\calX$ and $\calY$ are probability simplices, and $\z^*=(\x^*, \y^*)$ is the unique equilibrium. In this case, we use the sum of Kullback-Leibler divergence $\KL(\x^*, \x)+\KL(\y^*, \y)$ to measure the closeness between $\z=(\x, \y)$ and $\z^*$, where $\KL(\x, \x') = \sum_{i} x_i \ln \frac{x_i}{x'_i}$. With a slight abuse of notation, we use $\KL(\z,\z')$ to denote $\KL(\x,\x')+\KL(\y,\y')$. 

\paragraph{Other notations.}
We denote the $(d-1)$-dimensional probability simplex as $\Delta_d=\{\bu\in\mathbb{R}_+^d: \sum_{i=1}^d u_i=1\}$. 
For a convex function $\psi$, the corresponding \emph{Bregman divergence} is defined as $D_\psi(\bu, \bv)=\psi(\bu) - \psi(\bv) - \inner{\nabla \psi(\bv), \bu-\bv}$. If $\psi$ is $\gamma$-strongly convex in a domain, then $D_\psi(\bu, \bv)\geq \frac{\gamma}{2}\|\bu-\bv\|^2$ for any $\bu, \bv$ in that domain. For $\bu\in\mathbb{R}^d$, we define $\supp(\bu)=\left\{i:~ u_i > 0\right\}$. %The $L_p$-norm of $\bu$ is defined as $\|\bu\|_p=\left(\sum_{i} |u_i|^p\right)^{\frac{1}{p}}$. 

%Throughout the paper, we make the following two regularity assumptions.
%The first one is without loss of generality and assumes that the diameter of the feasible set $\calZ$ is normalized to $1$.
%\begin{assumption} \label{assum:smoothness}
%For any $z, z'\in\calZ$, $\|z-z'\| \leq 1$ holds.
%\end{assumption}

%The second one assumes that $f$ is $L$-smooth, which can be written as the following.
%\begin{assumption}\label{assum:diameter}
%For any $z, z'\in\calZ$, $\|F(z) - F(z')\| \leq L\|z-z'\|$ holds.
%\end{assumption}

%Finally, we state a useful lemma related to the duality gap, used in several places of our proofs.

%\begin{lemma}
%    \label{lem: relating duality gap and gradient}
%    For any $z\in\calZ$, we have $\alpha_f(z)\leq \max_{z'\in\calZ}F(z)^\top (z-z')$. 
%\end{lemma}
%\begin{proof}
%This is a direct consequence of the convexity of $f(\cdot, y)$ and the concavity of $f(x, \cdot)$:
%    \begin{align*}
%        \alpha_f(z)  &= \max_{(x', y')\in \calX\times \calY} \left(f(x, y') - f(x, y) + f(x,y) - f(x', y)\right)\\ 
%        &\leq \max_{(x', y')\in \calX\times \calY} \left(\nabla_{y} f(x,y)^\top (y'-y) + \nabla_x f(x,y)^\top (x-x')\right)    
%        =\max_{z'\in\calZ} F(z)^\top (z-z'). 
%    \end{align*}
%\end{proof}
%Note that the lemma also indicates that $\max_{z'\in\calZ} F(z)^\top (z-z')$ is always non-negative due to the non-negativity of the duality gap.  

\paragraph{Optimistic Gradient Descent Ascent (\alg).} %\label{sec:OGDA}
%\emph{Optimistic (projected) Gradient Descent Ascent} (\alg) 
Starting from an arbitrary point $(\xp_1, \yp_1)=(\x_0, \y_0)$ from $\calZ$,
\alg with step size $\eta > 0$ iteratively computes the following for $t = 1, 2,\ldots $,
\begin{align*}
    \x_{t}&= \Pi_{\calX} \big(\xp_{t} - \eta \nabla_{\x} f(\x_{t-1}, \y_{t-1}) \big),   &&\xp_{t+1} = \Pi_{\calX}  \big( \xp_t - \eta \nabla_{\x} f(\x_t, \y_t) \big),    \\
    \y_{t}&= \Pi_{\calY} \big(\yp_{t} + \eta \nabla_{\y} f(\x_{t-1}, \y_{t-1}) \big),  &&\yp_{t+1} = \Pi_{\calY}  \big( \yp_t + \eta \nabla_{\y} f(\x_t, \y_t) \big).  
\end{align*}
%which  can be compactly written as
%\begin{align}
%    \zp_{t+1} &= \Pi_{\calZ}  \big( \zp_t - \eta F(\z_t) \big),    &&\z_{t+1}= \Pi_{\calZ} \big(\zp_{t+1} - \eta F(\z_t) \big).   \label{eq: update rule 2}  
%\end{align}
Note that there are several slightly different versions of the algorithm in the literature, which differ in the timing of performing the projection. Our version is the same as those in \citep{chiang2012online, rakhlin2013optimization}. It is also referred to as ``single-call extra-gradient'' in \citep{hsieh2019convergence}, but it does not belong to the class of ``extra-gradient'' methods discussed in \citep{tseng1995linear, liang2019interaction, Golowich2020Last} for example. 

Also note that \alg only requires accessing $f$ via its gradient. In fact, only one gradient at the point $(\x_t, \y_t)$ is needed for iteration $t$.
This aspect makes it especially suitable for a repeated game setting, where in each round, one player proposes  $\x_t$ while another player proposes $\y_t$.
With only the information of the gradient from the environment ($\nabla_{\x} f(\x_t, \y_t)$ for the first player and  $\nabla_{\y} f(\x_t, \y_t)$ for the other),
both players can execute the algorithm.

\paragraph{Optimistic Multiplicative Weights Update (\OMWU).} %\label{sec:OMWU}
When the feasible sets $\calX$ and $\calY$ are probability simplices $\Delta_M$ and $\Delta_N$ for some integers $M$ and $N$,
\OMWU is another common iterative algorithm to solve the saddle-point problem.
For simplicity, we assume that it starts from the uniform distributions $(\xp_1, \yp_1) = (\x_0, \y_0) = \left(\frac{\mathbf{1}_M}{M}, \frac{\mathbf{1}_N}{N}\right)$, where $\mathbf{1}_d$ is the all-one vector of dimension $d$. 
Then \OMWU with step size $\eta >0$ iteratively computes the following for $t = 1, 2,\ldots $,
\begin{align*}
    x_{t,i}&=\frac{ {\xpp}_{t,i}\exp(-\eta (\nabla_{\x} f(\x_{t-1}, \y_{t-1}))_i)}{\sum_j{\xpp}_{t,j}\exp( - \eta (\nabla_{\x} f(\x_{t-1}, \y_{t-1}))_j)},    &&\xpp_{t+1,i} =  \frac{ {\xpp}_{t,i}\exp(-\eta (\nabla_{\x} f(\x_t, \y_t))_i)}{\sum_j{\xpp}_{t,j}\exp( - \eta (\nabla_{\x} f(\x_t, \y_t))_j)},     \\
y_{t,i}&=\frac{ {\ypp}_{t,i}\exp(\eta (\nabla_{\y} f(\x_{t-1}, \y_{t-1}))_i)}{\sum_j{\ypp}_{t,j}\exp( \eta (\nabla_{\y} f(\x_{t-1}, \y_{t-1}))_j)},     &&\ypp_{t+1,i} =\frac{ {\ypp}_{t,i}\exp(\eta (\nabla_{\y} f(\x_t, \y_t))_i)}{\sum_j{\ypp}_{t,j}\exp( \eta (\nabla_{\y} f(\x_t, \y_t))_j)}.
\end{align*}

%% file: algorithm.tex
\paragraph{\OMWU and \alg as Optimistic Mirror Descent Ascent.}
%\label{sec: mirror descent}
\OMWU and \alg can be viewed as special cases of \emph{Optimistic Mirror Descent Ascent}. %One can verify that their update rules can be equivalently written in the following form: 
%\begin{align}
%    \x_{t} &= \argmin_{\x\in\calX}\Big\{\eta \inner{\x,  \nabla_{\x} f(\x_{t-1},\y_{t-1})} + D_\psi(\x, \xp_{t})\Big\},  \label{eq: omda update 1}\\
%    \xp_{t+1} &= \argmin_{\x\in\calX}\Big\{\eta \inner{\x, \nabla_{\x} f(\x_t,\y_t)} + D_\psi(\x, \xp_t)\Big\},   \label{eq: omda update 2} \\
%    \y_{t} &= \argmax_{\y\in\calY}\Big\{\eta \inner{\y,  \nabla_{\y} f(\x_{t-1},\y_{t-1})} - D_\psi(\y, \yp_{t})\Big\}, \label{eq: omda update 3}\\
%    \yp_{t+1} &= \argmax_{\y\in\calY}\Big\{\eta \inner{\y, \nabla_{\y} f(\x_t,\y_t)} - D_\psi(\y, \yp_t)\Big\}, \label{eq: omda update 4}
%\end{align}
%where $\psi$ is the \emph{regularizer}, and $D_{\psi}(\bu,\bv)\triangleq \psi(\bu)-\psi(\bv)-\inner{\nabla \psi(\bv), \bu-\bv}$ is the Bregman divergence with respect to $\psi$. In \OMWU, $\psi(\bu)=\sum_i u_i \ln u_i$ and $D_\psi(\bu, \bv)=\sum_i u_i\ln \frac{u_i}{v_i}$; in \alg, $\psi(\bu)= \frac{1}{2}\|\bu\|^2$ and $D_\psi(\bu, \bv)=\frac{1}{2}\|\bu-\bv\|^2$. Using the notation $\z=(\x,\y)$ and by the definition of $F(\z)$, we simplify \pref{eq: omda update 1}-\pref{eq: omda update 4} as 
Specifically, let {\it regularizer} $\psi(\bu)$ denote the negative entropy $\sum_i u_i \ln u_i$ for the case of \OMWU and (half of) the $L_2$ norm square $\frac{1}{2}\|\bu\|^2$ for the case of \alg (so that $D_{\psi}(\bu,\bv)$ is $\KL(\bu, \bv)$ and $\frac{1}{2}\|\bu-\bv\|^2$ respectively).
Then using the shorthands $\z_t=(\x_t,\y_t)$ and $\zp_t=(\xp_t,\yp_t)$ and recalling the notation defined earlier: $\calZ=\calX\times \calY$ and $F(\z) = \left(\nabla_{\x} f(\x,\y), -\nabla_{\y} f(\x,\y)\right)$, 
one can rewrite \OMWU/\alg compactly as
\begin{align}
    \z_{t} &= \argmin_{\z\in\calZ} \Big\{\eta \inner{\z, F(\z_{t-1})} + D_\psi(\z,\zp_{t})\Big\},  \label{eq: omda update 5}\\
    \zp_{t+1} &= \argmin_{\z\in\calZ} \Big\{\eta \inner{\z,  F(\z_t)} + D_\psi(\z,\zp_t)\Big\}. \label{eq: omda update 6} 
\end{align}

By the standard regret analysis of Optimistic Mirror Descent, we have the following important lemma, which is readily applied to \OMWU and \alg when $\psi$ is instantiated as the corresponding regularizer. The proof is mostly standard
(see e.g., \cite[Lemma 1]{rakhlin2013optimization}). For completeness, we include it in \pref{app:regret bound sufficient decrease}. 

\begin{lemma}
	\label{lem: regret bound omwu}
	Consider update rules \pref{eq: omda update 5}  and \pref{eq: omda update 6} and
	define $\distp{p}(\z,\z') =\|\x-\x'\|^2_p+\|\y-\y'\|^2_p$.
	 Suppose that $\psi$ satisfies $D_\psi(\z,\z')\geq \frac{1}{2}\distp{p}(\z,\z')$ for some $p\geq 1$, and $F$ satisfies $\distp{q}(F(\z),F(\z'))\leq L^2\distp{p}(\z,\z')$ for $q\geq 1$ with $\frac{1}{p}+\frac{1}{q}=1$.  Also, assume that $\eta\leq \frac{1}{8L}$. Then for any $\z\in\calZ$ and any $t\geq 1$, we have
	\begin{align*}
	\eta F(\z_t)^\top (\z_t-\z) \leq D_\psi(\z, \zp_t) - D_\psi(\z, \zp_{t+1}) - D_\psi(\zp_{t+1}, \z_t) - \tfrac{15}{16}D_\psi(\z_t, \zp_t)+ \tfrac{1}{16} D_\psi(\zp_t, \z_{t-1}).
	\end{align*}
\end{lemma}

\section{Convergence Results for \OMWU}
\label{sec: omwu convergence}
In this section, we show that for a two-player zero-sum matrix game with a unique equilibrium, \OMWU with a constant learning rate converges to the equilibrium exponentially fast. The assumption and the algorithm are the same as those considered in \citep{daskalakis2019last}, but our analysis improves theirs in two ways. First, we do not require the learning rate to be exponentially smaller than some problem-dependent quantity. Second, we explicitly provide a linear convergence rate. In \pref{sec:convergence}, we further remove the uniqueness assumption and significantly generalize the results by studying \alg. 

In a matrix game we have $\calX=\Delta_M$, $\calY = \Delta_N$, and $f(\z)=\x^\top \G\y$ for some matrix $\G\in[-1,1]^{M\times N}$.
To show the last-iterate convergence of \OMWU, we first apply \pref{lem: regret bound omwu} with $D_\psi(\bu,\bv)=\KL(\bu,\bv)$, $\z=\z^*$ (the unique equilibrium of the game matrix $\G$) and $(p, q)=(1,\infty)$.
The constant $L$ can be chosen as $1$ since $\distp{\infty}(F(\z),F(\z')) = \max_{i} |(\G (\y-\y'))_i|^2+\ \max_{j} |(\G^\top (\x-\x'))_j|^2\leq \|\y-\y'\|_1^2+\|\x-\x'\|_1^2=\distp{1}(\z,\z')$. Also notice that $F(\z_t)^\top (\z_t-\z^*) = f(\x_t, \y_t) - f(\x^*, \y_t) + f(\x_t, \y^*) - f(\x_t, \y_t) = f(\x_t, \y^*)- f(\x^*, \y_t) \geq 0$ by the optimality of $\z^*$.
Therefore, we have when $\eta\le \frac{1}{8}$,
\begin{align*}
     \KL(\z^*, \zp_{t+1})\leq \KL(\z^*, \zp_t) - \KL(\zp_{t+1}, \z_{t}) - \tfrac{15}{16}\KL(\z_t, \zp_t) + \tfrac{1}{16}\KL(\zp_t, \z_{t-1}). 
\end{align*}
Defining $\Theta_t=\KL(\z^*, \zp_t) + \tfrac{1}{16}\KL(\zp_t, \z_{t-1})$ and $\zeta_t=\KL(\zp_{t+1}, \z_{t}) +\KL(\z_t, \zp_t)$, we rewrite the above as
\begin{align}
    \Theta_{t+1} \leq \Theta_t - \tfrac{15}{16}\zeta_t.    \label{eq: simple recursion}
\end{align}
From \pref{eq: simple recursion} it is clear that the quantity $\Theta_t$ is always non-increasing in $t$ due to the non-negativity of $\zeta_t$. Furthermore, the more the algorithm moves between round $t$ and round $t+1$ (that is, the larger $\zeta_t$ is), the more $\Theta_t$ decreases. 

To establish the rate of convergence, a natural idea is to relate $\zeta_t$ back to $\Theta_t$ or $\Theta_{t+1}$. For example, if we can show $\zeta_t \geq c\Theta_{t+1}$ for some constant $c>0$, then \pref{eq: simple recursion} implies $\Theta_{t+1}\leq \Theta_t - \frac{15c}{16}\Theta_{t+1}$, which further gives $\Theta_{t+1}\leq \left(1+\tfrac{15c}{16}\right)^{-1}\Theta_t$. This immediately implies a linear convergence rate for $\Theta_t$ as well as $\KL(\z^*, \zp_t)$ since $\KL(\z^*, \zp_t)\leq \Theta_t$. 

Moreover, notice that to find such $c$, it suffices to find a $c'>0$ such that $\zeta_t \geq c'\KL(\z^*,\zp_{t+1})$. This is because it will then give $\zeta_t \geq \tfrac{1}{16}\KL(\zp_{t+1}, \z_t) + \tfrac{15}{16}\zeta_t \geq \tfrac{1}{16}\KL(\zp_{t+1}, \z_t) + \tfrac{15c'}{16}\KL(\z^*,\zp_{t+1}) \geq \min\{1, \tfrac{15c'}{16}\}\Theta_{t+1}$, and thus $c\triangleq \min\{1, \tfrac{15c'}{16}\}$ satisfies the condition. 

From the discussion above, we see that to establish the linear convergence of $\KL(\z^*, \zp_{t})$, we only need to show that there exists some $c'>0$ such that $\KL(\zp_{t+1}, \z_{t}) +\KL(\z_t, \zp_t) \geq c'\KL(\z^*,\zp_{t+1})$. The high-level interpretation of this inequality is that when $\zp_{t+1}$ is far from the equilibrium $\z^*$ (i.e., $\KL(\z^*,\zp_{t+1})$ is large), the algorithm should have a large move between round $t$ and $t+1$ making $\KL(\zp_{t+1}, \z_{t}) +\KL(\z_t, \zp_t)$ large. 

In our analysis, we use a two-stage argument to find such a $c'$. In the first stage, we only show that 
$\KL(\zp_{t+1}, \z_{t}) +\KL(\z_t, \zp_t) \geq c''\KL(\z^*,\zp_{t+1})^2$ for some $c''>0$, and use it to argue a slower convergence rate $\KL(\z^*,\zp_{t}) = \order\left(\frac{1}{t}\right)$. Then in the second stage, we show that after $\zp_t$ and $\z_t$ become close enough to $\z^*$, we have $\KL(\zp_{t+1}, \z_{t}) +\KL(\z_t, \zp_t) \geq c'\KL(\z^*,\zp_{t+1})$ for some $c'>0$. 

\modify{
This kind of two-stage argument might be reminiscent of that used by \citet{daskalakis2019last}; however, the techniques we use are very different.
Specifically, \citet{daskalakis2019last} utilize tools of ``spectral analysis'' similar to \citep{liang2019interaction} and show that the OMWU update can be viewed as a ``contraction mapping'' with respect to a matrix whose eigenvalue is smaller than 1.
Our analysis, on the other hand, leverages analysis of online mirror descent, starting from the ``one-step regret bound'' (\pref{lem: regret bound omwu}) and making use of the two negative terms that are typically dropped in the analysis.
Importantly, our analysis does not need an exponentially small learning rate required by \citep{daskalakis2019last}.
Thus, unlike their results, our learning rate is kept as a \major{universal} constant in all stages.
The arguments above are formalized below: 
}
\begin{lemma}
    \label{lem: KL sufficient decrease}
    Consider a matrix game $f(\x,\y)=\x^\top \G \y$ with $\calX = \Delta_M$, $\calY = \Delta_N$, and $\G\in [-1, 1]^{M\times N}$.
    Assume that there exists a unique Nash equilibrium $\z^*$ and $\eta\leq \frac{1}{8}$.
    Then, there exists a constant $C_1>0$ that depends on $\G$ such that for any $t\geq 1$, \OMWU ensures
    \begin{align*}
        \KL(\zp_{t+1}, \z_{t}) +\KL(\z_t, \zp_t) \geq \eta^2 C_1\KL(\z^*,\zp_{t+1})^2.  
    \end{align*}
    Also, there is a constant $\xi>0$ that depends on $\G$ (defined in \pref{def: xi}) such that as long as $\max\{\|\z^*-\zp_t\|_1, \|\z^*-\z_t\|_1\} \leq \frac{\eta\xi}{10}$, then
    \begin{align*}
        \KL(\zp_{t+1}, \z_{t}) +\KL(\z_t, \zp_t) \geq \eta^2 C_2\KL(\z^*,\zp_{t+1})  
    \end{align*}
    for another constant $C_2>0$ that depends on $\G$. 
\end{lemma}

With \pref{lem: KL sufficient decrease} and the earlier discussion, the last-iterate convergence rate of \OMWU is established: 
\begin{theorem}
\label{thm: point convergence omwu}
For a matrix game $f(\x,\y)=\x^\top \G \y$ with a unique Nash equilibrium $\z^*$,
\OMWU with a learning rate $\eta\leq \frac{1}{8}$ guarantees
$
         \KL(\z^*,\z_t) \leq C_3(1+C_4)^{-t},
$
    where $C_3,C_4>0$ are some constants depending on the game matrix $\G$. %Choosing $\zp_0=(\frac{\mathbf{1}_M}{M},\frac{\mathbf{1}_N}{N})$, $\KL(z^*,\zp_0)$ can be bounded by $\ln(MN)$ , which makes the constants only depend on $G$.
\end{theorem}

Proofs for this section are deferred to \pref{app:OMWU_proofs}, where all
problem-dependent constants are specified as well.\footnote{%
\modify{One might find that the constant $C_3$ is exponential in some problem-dependent quantity $T_0$. 
However, this is simply a loose bound in exchange for more concise presentation ---
our proof in fact shows that when $t < T_0$, the convergence is of a slower $1/t$ rate,
and when $t \geq T_0$, the convergence is linear without this large constant.}
}
%However, this is simply due to a loose relaxation of the first stage 
%In fact, since we use a two-stage argument to show convergence, in which the first stage with slower convergence (i.e., $\order(\frac{1}{t} )$ rate) only lasts for a constant number of steps, we push this constant into $C_3$ for simplicity.
\modify{
To the best of our knowledge, \pref{thm: point convergence omwu} gives the first last-iterate convergence result for \OMWU with a concrete linear rate. } %we use a two-stage argument to prove \pref{thm: point convergence omwu}. In the first stage, \OMWU has a slower convergence rate $\order\left(\frac{1}{t}\right)$, while in the second stage, \OMWU has a linear convergence rate. However, as the first stage only lasts for a constant number of steps, we push the constant into the coefficient $C_3$ for simplicity. }
We note that the uniqueness assumption is critical for our analysis,
and whether this is indeed necessary for \OMWU is left as an important future direction.

\section{Convergence Results for \alg}
\label{sec:convergence}
In this section, we provide last-iterate convergence results for \alg, which are much more general than those in \pref{sec: omwu convergence}. %As a warmup, in \pref{sec:dual-conv} we first provide a convergence guarantee in terms of the sum of duality gap, which is new in the constrained setting to the best of our knowledge.  
%In \pref{sec:pointwise}, 
We propose a general condition subsuming many well-studied cases, under which \alg enjoys a concrete last-iterate convergence guarantee in terms of the $L_2$ distance between $\z_t$ and $\calZ^*$. The results in this part can be specialized to the setting of bilinear games over simplex, but the unique equilibrium assumption made in \pref{sec: omwu convergence} and in 
\citep{daskalakis2019last} is no longer needed. 
%In this section, we provide our main convergence results for the iterate $\zp_t$ of \alg.
%Specifically, in \pref{sec:dual-conv} we first provide a convergence guarantee in terms of the duality gap.
%Then in \pref{sec:pointwise}, we propose a general condition subsuming many well-studied cases, under which \alg is shown to ensure a last-iterate convergence guarantee in terms of the distance between $\zp_t$ and $\calZ^*$.

Throughout the section we make the assumption that $f$ is $L$-smooth: 
%The first one is without loss of generality and assumes that the diameter of the feasible set $\calZ$ is normalized to $1$.\footnote{%
%If the diameter is some general value $D$,
%then we can define a new game with $\z'=\z/D$ and $f'(\z') = f(D\z')$, where the diameter is $1$ and the Lipschitz constant $L$ in \pref{assum:smoothness} becomes $LD$.
%}
%\begin{assumption} \label{assum:diameter}
%For any $\z, \z'\in\calZ$, $\|\z-\z'\| \leq 1$ holds.
%\end{assumption}

\begin{assumption}\label{assum:smoothness}
For any $\z, \z'\in\calZ$, $\|F(\z) - F(\z')\| \leq L\|\z-\z'\|$ holds.\footnote{%
This is equivalent to the condition $\distp{q}(F(\z),F(\z'))\leq L^2\distp{p}(\z,\z')$ in \pref{lem: regret bound omwu} with $p=2$, hence the same notation $L$.
}
\end{assumption}

To introduce our general condition, we first provide some intuition by applying \pref{lem: regret bound omwu} again.
Letting $\psi(\bu)=\frac{1}{2}\|\bu\|^2$ in \pref{lem: regret bound omwu}, we get that for \alg, for any $\z\in\calZ$ and any $t\geq 1$, 
\begin{align*}
    2\eta F(\z_t)^\top (\z_t-\z) \leq \|\zp_t-\z\|^2 - \|\zp_{t+1}-\z\|^2 - \|\zp_{t+1}-\z_t\|^2 - \tfrac{15}{16}\|\z_t-\zp_t\|^2+ \tfrac{1}{16} \|\zp_t - \z_{t-1}\|^2. 
\end{align*}
Now we instantiate the inequality above with $\z=\Pi_{\calZ^*}(\zp_t)\in\calZ^*$. Since $\z=\Pi_{\calZ^*}(\zp_t)$ is an equilibrium, we have $F(\z_t)^\top (\z_t-\z) \geq f(\x_t, \y_t) - f(\x, \y_t) + f(\x_t, \y) - f(\x_t, \y_t)  =  f(\x_t, \y) -  f(\x, \y_t) \ge 0$ by the convexity/concavity of $f$ and the optimality of $\z$, and thus
\begin{align*}
    %\|\zp_{t+1}-\Pi_{\calZ^*}(\zp_{t+1})\|^2 
    \|\zp_{t+1}-\Pi_{\calZ^*}(\zp_t)\|^2  
    \leq \|\zp_t-\Pi_{\calZ^*}(\zp_t)\|^2 - \|\zp_{t+1}-\z_t\|^2 - \tfrac{15}{16}\|\z_t-\zp_t\|^2+ \tfrac{1}{16} \|\zp_t - \z_{t-1}\|^2.
\end{align*}
Further noting that the left-hand side is lower bounded by $\dist(\zp_{t+1}, \calZ^*)$ by definition, we arrive at
\[
\dist(\zp_{t+1}, \calZ^*) 
\leq \dist(\zp_{t}, \calZ^*) - \|\zp_{t+1}-\z_t\|^2 - \tfrac{15}{16}\|\z_t-\zp_t\|^2+ \tfrac{1}{16} \|\zp_t - \z_{t-1}\|^2.
\]
Similarly, we define $\Theta_t = \|\zp_t - \Pi_{\calZ^*}(\zp_t)\|^2 + \tfrac{1}{16}\|\zp_t-\z_{t-1}\|^2$, $\zeta_t=\|\zp_{t+1}-\z_t\|^2 + \|\z_t-\zp_t\|^2$, and rewrite the above as 
\begin{align}
     \Theta_{t+1}\leq \Theta_t - \tfrac{15}{16}\zeta_t.    \label{eq: ogda recurr}
\end{align}

%Before presenting the main theorems, we first show two important lemmas.
%The first one shows the connection between $\|\zp_{t+1}-z\|^2$ and $\|\zp_t-z\|^2$ for any $z\in\calZ$, and the proof follows standard analysis of \alg (see %for example, \cite[Lemma 5]{chiang2012online} and
%e.g., \cite[Lemma 1]{rakhlin2013optimization}). For completeness, we provide a proof in \pref{app:regret bound sufficient decrease}. 

%\begin{lemma}
%	\label{lem: regret bound}
%	For any $t\geq 1$ and $z\in\calZ$, \alg with $\eta\leq \frac{1}{8L}$ ensures
%	\begin{align*}
%	2\eta F(z_t)^\top (z_t-z) \leq \|\zp_t-z\|^2 - \|\zp_{t+1}-z\|^2 - \|\zp_{t+1}-z_t\|^2 - \tfrac{15}{16}\|z_t-\zp_t\|^2+ \tfrac{1}{16} \|\zp_t - z_{t-1}\|^2. 
%	\end{align*}
  %  And \OMWU on the simplexes with $\eta\le\frac{1}{8}$ ensures 
%    \begin{align*}
%    2 \eta F\left(z_{t}\right)^{\top}\left(z_{t}-z\right) \leq \KL(z,\widehat{z}_{t})-\KL(z,\widehat{z}_{t+1})-\KL(\widehat{z}_{t+1},z_{t})-\tfrac{15}{16}\KL(z_{t},\widehat{z}_{t})+\tfrac{1}{16}\left\|\widehat{z}_{t}-z_{t-1}\right\|^{2}_1.
%    \end{align*}
%\end{lemma}

%Further rearranging then shows that $\widehat{z}_{t+1}$ is closer to $z$ than $\widehat{z}_t$ as long as $\|\zp_{t+1}-z_t\|^2 + \frac{15}{16}\|z_t-\zp_t\|^2 - \frac{1}{16} \|\zp_t - z_{t-1}\|^2$ is positive, which drives the convergence of $\zp_t$ towards $\calZ^*$. 
As in \pref{sec: omwu convergence}, our goal now is to lower bound $\zeta_t$ by some quantity related to $\dist(\zp_{t+1}, \calZ^*)$, and then use \pref{eq: ogda recurr} to obtain a convergence rate for $\Theta_t$. In order to incorporate more general objective functions into the discussion, in the following \pref{lem: sufficient decrease}, we provide an intermediate lower bound for $\zeta_t$, which will be further related to $\dist(\zp_{t+1}, \calZ^*)$ later. 

%To this end, in the following lemma, we provide a lower bound for the main term $\|\zp_{t+1}-z_t\|^2 + \|z_t-\zp_t\|^2$.

\begin{lemma}
    \label{lem: sufficient decrease}
    For any $t \geq 0$ and $\z'\in\calZ$ with $\z'\neq \zp_{t+1}$, \alg with $\eta\leq \frac{1}{8L}$ ensures
    \begin{align}
       &\|\zp_{t+1}-\z_t\|^2 + \|\z_t-\zp_t\|^2 \geq \frac{32}{81}\eta^2 \frac{ \left[ F(\zp_{t+1})^\top (\zp_{t+1}-\z')\right]_+^2}{\|\zp_{t+1}-\z'\|^2}, \label{eq:driving} 
     \end{align}  
     where $[a]_+\triangleq \max\{a, 0\}$, and similarly, for $\z'\neq \z_{t+1}$, 
      \begin{align}
       &\|\zp_{t+1}-\z_{t+1}\|^2 + \|\z_t-\zp_{t+1}\|^2 \geq \frac{32}{81}\eta^2 \frac{ \left[ F(\z_{t+1})^\top (\z_{t+1}-\z')\right]_+^2}{\|\z_{t+1}-\z'\|^2} \label{eq:driving z}.
    \end{align}
    
\end{lemma}

We note that a direct consequence of \pref{lem: sufficient decrease} is an ``average duality gap'' guarantee for \alg when $\calZ$ is bounded:
\begin{align}
\frac{1}{T}\sum_{t=1}^T \alpha(\z_t) = \frac{1}{T}\sum_{t=1}^T \max_{\x'\in\calX, \y'\in\calY}\left(f(\x_t, \y') - f(\x', \y_t) \right)=\order\left(\frac{D}{\eta\sqrt{T}}\right)  \label{eq: duality gap bound 1}
\end{align} 
where $D\triangleq \sup_{\z, \z'\in\calZ}\|\z-\z'\|$ is the diameter of $\calZ$ (the duality gap may be undefined when $\calZ$ is unbounded).
We are not aware of any previous work that gives this result for the constrained case. See \pref{app: sum of duality gap} for the proof of \pref{eq: duality gap bound 1} and comparisons with previous works. 

%See \pref{app:regret bound sufficient decrease} for the proof.
%Note that the right-hand side of \pref{eq:driving} can be further related to the duality gap of $\zp_{t+1}$ through \pref{lem: relating duality gap and gradient}. Thus, the high-level idea of proving last-iterate convergence emerges: when $\zp_{t+1}$ has a large duality gap, the term $\|\zp_{t+1}-z_t\|^2 + \|z_t-\zp_t\|^2$ has to be large also, which in turn leads to a large decrease in $\|\zp_{t+1}-z\|^2$ compared to $\|\zp_{t}-z\|^2$ by \pref{lem: regret bound}.

%\subsection{Pointwise Convergence}
%\label{sec:pointwise}

However, to obtain last-iterate convergence results, we need to make sure that the right-hand side of \pref{eq:driving} is large enough.
%Now we continue to argue the last-iterate convergence of $\dist(\z_{t}, \calZ^*)$. Recall that a large right-hand side of \pref{eq:driving} ensures a sufficient decrease in $\dist(\zp_t, \calZ^*)$ in view of \pref{eq: ogda recurr}. 
%To obtain a last-iterate convergence result, however,
%we still need to connect $\dist(\zp_t, \calZ^*)$ back to the duality gap.
Motivated by this fact, we propose the following general condition on $f$ and $\calZ$ to achieve so.
\modify{
\begin{definition}[\rsilong (\rsi)]\label{def:rsi}
The $\rsi$ condition is defined as: 
for any $\z\in\calZ\backslash \calZ^*$ with $\z^*=\Pi_{\calZ^*}(\z)$, 
	\begin{align*}
	   \sup_{\z'\in\calZ} \frac{F(\z)^\top (\z-\z')}{\|\z-\z'\|} \geq C \|\z-\z^*\|^{\beta+1} \tag*{(\rsi)}
	\end{align*}
	holds for some parameter $\beta \geq 0$ and $C>0$. 
\end{definition}
}
\modify{We call this condition \rsilong because the case with $\beta=0$ is equivalent to one type of metric subregularity in variational inequality problems,
as we prove in \pref{app:ms}.
The condition is also closely related to other error bound conditions that have been identified for variational inequality problems (e.g., \citet{tseng1995linear, gilpin2008first, malitsky2019golden}).
Although these works have shown that under similar conditions their algorithms exhibit linear convergence, to the best of our knowledge, there is no previous work that analyzes \alg or other no-regret learning algorithms using such conditions.  %its resemblance to the Restricted Secant Inequality \citep{zhang2013gradient, karimi2016linear}.
%The latter leads to linear convergence in convex optimization.  
%To the best of our knowledge, our proposed condition is new for saddle-point problems.
%The closest condition in the literature is from a recent work by~\citet{hsieh2019convergence} for variational inequalities (their Assumption~3(s)).
%In our context, their condition is $(F(\z)-F(\z'))^\top (\z-\z') \geq C \|\z-\z'\|^{2}$ for {\it all} $\z, \z' \in\calZ$.
%\modify{This in fact implies \rsi with $\beta=0$}.
%To see this, simply set $\z' = \z^*$ and notice that $F(\z^*)^\top (\z-\z^*)\geq 0$ holds by the first-order optimality condition of $\z^*$.
%Therefore, our condition \rsi is only less stringent and covers more problems.
%In fact, under their stronger condition the equilibrium is necessarily unique,
%but we generally allow multiple equilibria.

}

%Note that \typetwo is a more general function class than the one that satisfies Assumption 3.(s) in \citep{hsieh2019convergence} in the sense that with $\beta=0$, \typetwo recovers a weaker version of their assumption by only requiring $x'=x^*$. 
%To see why these two assumptions are reasonable, we show that both bilinear games (with polytope feasible set or arbitrary bounded convex set) and stongly-convex-concave games satisfy one of these two assumptions.

%\major{
%\paragraph{Remark. }
%In some cases when $\calZ$ is unbounded, a global $C$ might be unavailable for all points $\z$ in the feasible set $\calZ$. In this case, a weaker condition in fact suffices which only requires that  \pref{def:rsi} holds for some $\z$ in a bounded region. We elaborate this weakened condition in \pref{app: unbounded}. }

\rsi covers many standard settings studied in the literature.
The first and perhaps the most important example is bilinear games with a polytope feasible set, which in particular includes the classic two-player matrix games  considered in \pref{sec: omwu convergence}. 

\begin{theorem}
	\label{thm: bilinear-polytope}
	A bilinear game $f(\x,\y)=\x^\top \G \y$ with $\calX \subseteq \mathbb{R}^M$ and $\calY \subseteq \mathbb{R}^N$ being polytopes and $\G \in \mathbb{R}^{M\times N}$ satisfies \typeone with $\beta=0$. %The two player matrix game is a special case of it with $\calX=\Delta_M\triangleq \{x\in\mathbb{R}^M: x(i)\geq 0, \sum_{i=1}^M x(i)=1\}$ and $\calY=\Delta_N$. 
\end{theorem} 

We emphasize again that different from \pref{lem: KL sufficient decrease}, \pref{thm: bilinear-polytope} does not require a unique equilibrium.
Note that we have not provided the concrete form of the parameter $C$ in the theorem (which depends on $\calX$, $\calY$, and $\G$), but it can be found in the proof (see \pref{app:bilinear polytope}).\footnote{\major{After the first version of this paper, we found that \citep[Lemma 3]{gilpin2008first} gives a simpler proof for our \pref{thm: bilinear-polytope}. Although their lemma only focuses on the case where the feasible sets are probability simplices, it can be directly extended to the case of polytopes. }}
The next example shows that strongly-convex-strongly-concave problems are also special cases of our condition.
 
\begin{theorem}
	\label{thm: strongly convex}
	If $f$ is strongly convex in $\x$ and strongly concave in $\y$, then \typetwo holds with $\beta=0$. 
\end{theorem}

Next, we provide a toy example where \typetwo holds with $\beta > 0$.
\begin{theorem}
	\label{thm: beta ge 0}
	Let $\mathcal{X}=\mathcal{Y}\triangleq\{(a,b): 0\le a,b\le 1,~a+b=1\}$, $n > 2$ be an integer, and
$f(\x,\y)=x_1^{2n}-x_1y_1-y_1^{2n}$. Then \typetwo holds with $\beta=2n-2$.  
\end{theorem}

%\begin{theorem}  
%     \label{thm: kkt inquality}
%     The ``non-degenerated KKT condition'' considered by \citet{Lei2020Last} implies \typeone with $\beta=0$ (see appendix for the detailed statement of their condition). 
%     \chenyu{TODO: finish the proof}
%\end{theorem}

%For a smaller $\beta$, the connection between the right-hand side of \pref{eq:driving} and $\|\zp_{t+1}-z^*\|^2$ is stronger, and can give a faster convergence. For all $\beta\geq 0$, the convergence rate is given by the following theorem. 
%Then according to the definition of \typeone and \typetwo, by direct calculation, we can further lower bound the term $\|\zp_{t+1}-z_t\|^2 + \|z_t-\zp_t\|^2$ by $\Omega(\eta^2\|\zp_{t+1}-z'\|^{\beta+1})$ for any $z'\in \calZ$ if $f$ satisfies either \typeone or \typetwo, which leads the the following convergence result by choosing $z'=\Pi_{\calZ^*}(\zp_t)$ and solving the recursive inequality.

With this general condition, we are now able to complete the loop.
For any value of $\beta$, we show the following last-iterate convergence guarantee for \alg. 

\modify{
\begin{theorem}
\label{thm: point convergence}
    For any $\eta\leq \frac{1}{8L}$, if \rsi holds with $\beta= 0$, then \alg guarantees linear last-iterate convergence:
    \begin{align}\label{eq:conv-bilinear}
         \dist(\z_t, \calZ^*) \leq 64\dist(\zp_1, \calZ^*)(1+C_5)^{-t};
    \end{align}
    on the other hand, if the condition holds with $\beta > 0$, then we have a slower convergence:
    \begin{align}\label{eq:conv-cvx}
         \dist(\z_t, \calZ^*) \leq 32\left[\left(1+4\left(\frac{4}{\beta}\right)^{\frac{1}{\beta}}\right)\dist(\zp_1, \calZ^*) + 2\left(\frac{2}{C_5\beta}\right)^{\frac{1}{\beta}}\right]t^{-\frac{1}{\beta}}, 
    \end{align}
    where $C_5\triangleq \min\left\{\frac{16\eta^2 C^2}{81}, \frac{1}{2}\right\}$.
\end{theorem}
}

We defer the proof to \pref{app:point convergence} and make several remarks.
First, note that based on a convergence result on $\dist(\z_t, \calZ^*)$, one can immediately obtain a convergence guarantee for the duality gap $\alpha_f(\z_t)$ as long as $f$ is also Lipschitz.
This is because $\alpha_f(\z_t) \leq \max_{\x', \y'} f(\x_t,\y') - f(\x^*, \y') + f(\x', \y^*) - f(\x', \y_t)\leq \order(\|\x_t-\x^*\| + \|\y_t-\y^*\|) = \order\Big(\sqrt{\dist(\z_t, \calZ^*)}\Big)$, where $(\x^*, \y^*) = \Pi_{\calZ^*}(\z_t)$.
While this leads to stronger guarantees compared to \pref{eq: duality gap bound 1}, we emphasize that the latter holds even without the \rsi condition.

Second, our results significantly generalize~\citep[Theorem~2]{hsieh2019convergence} which itself is a consolidated version of several earlier works and also shows a linear convergence rate of \alg under a condition stronger than our \typetwo with $\beta=0$ as discussed earlier.
More specifically, our results show that linear convergence holds for a much broader set of problems. %, especially under \typeone which was not discovered before.
Furthermore, we also show slower sublinear convergence rates for any value of $\beta>0$, which is also new as far as we know.
In particular, we empirically verify that \alg indeed does not converge exponentially fast for the toy example defined in~\pref{thm: beta ge 0} (see \pref{app:experiment}).

Last but not least, the most significant implication of \pref{thm: point convergence} is that it provides by far the most general linear convergence result for \alg for the classic two-player matrix games, or more generally bilinear games with polytope constraints, according to \pref{thm: bilinear-polytope} and \pref{eq:conv-bilinear}.
Compared to recent works of~\citep{daskalakis2018limit, daskalakis2019last} for matrix games (on \alg or \OMWU), our result is considerably stronger:
1) we do not require a unique equilibrium while they do;
2) linear convergence holds for any initial points $\zp_1$, while their result only holds if the initial points are in a small neighborhood of the unique equilibrium (otherwise the convergence is sublinear initially);
3) our only requirement on the step size is $\eta\le \frac{1}{8L}$,\footnote{\modify{In fact, any $\eta < \frac{1}{2L}$ is enough to achieve linear convergence rate for \alg, as one can verify by going over our proof. 
We use $\eta\le \frac{1}{8L}$ simply for consistency with the results for \OMWU
(where $\eta$ cannot be set any larger due to technical reasons).}} while they require an exponentially small $\eta$, which does not reflect the behavior of the algorithms in practice.
Even compared with our result in \pref{sec: omwu convergence}, we see that for \alg, the unique equilibrium assumption is not required, and we do not have an initial phase of sublinear convergence as in \pref{lem: KL sufficient decrease}. 
In \pref{app:experiment}, we empirically show that \alg often outperforms \OMWU when both are tuned with a constant learning rate.

One may wonder what happens if a bilinear game has a non-polytope constraint.
It turns out that in this case, \typeone may only hold with $\beta > 0$, due to the following example showing that linear convergence provably does not hold for \alg when the feasible set has a curved boundary.
\begin{theorem}\label{thm: bilinear-general}
	There exists a bilinear game with a non-polytope feasible set such that \modify{\rsi holds with $\beta=3$}, and $\dist(\z_t, \calZ^*)= \Omega(1/t^2)$ holds for \alg. 
\end{theorem}

%\modify{In the proof of \pref{thm: bilinear-general}, we further show that the example we construct satisfies \rsi with $\beta=3$, demonstrating that our \rsi considtion is able to capture bilinear games with curved feasible sets. }
This example indicates that the shape of the feasible set plays an important role in last-iterate convergence, which may be an interesting future direction to investigate,
This is also verified empirically in our experiments (see \pref{app:experiment}).

%% file: experiments-main.tex
\section[Experiments]{Experiments for Matrix Games} %\footnote{All source codes can be found at \url{https://github.com/bahh723/OGDA-last-iterate}.} %hide for submission 
\label{sec:experiment}

In this section, we provide empirical results on the performance of \alg and \OMWU for matrix games on probability simplex.\footnote{%
Note that in this case the projection step of \alg can be implemented efficiently in $O(M\ln M + N\ln N)$ time~\citep{wang2013projection}.
} 
We include more empirical results in other settings in \pref{app:experiment}.
We set the size of the game matrix to be $32\times 32$, then generate a random matrix with each entry $G_{ij}$ drawn uniformly at random from $[-1,1]$, and finally rescale its operator norm to $1$.
With probability $1$, the game has a unique Nash Equilibrium \citep{daskalakis2019last}.

\begin{figure}
	\centering
	\includegraphics[width=.7\textwidth,trim={2cm 5cm 3cm 7cm}]{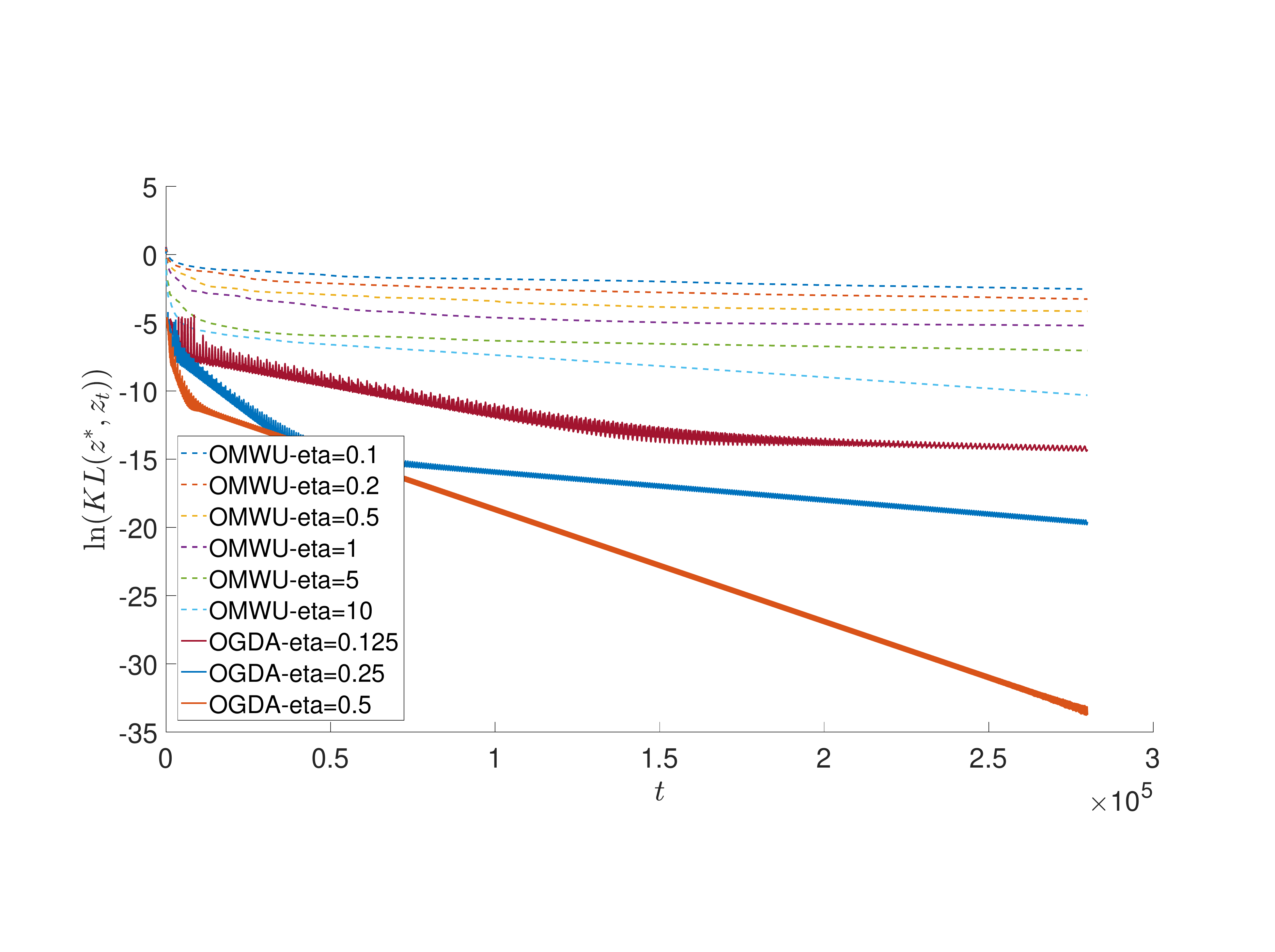}
	\caption{Experiments of \alg and \OMWU with different learning rates for a matrix game $f(\x,\y)=\x^\top \G\y$. %where we generate $\G\in \mathbb{R}^{32\times 32}$ with each entry $G_{ij}$ drawn uniformly at random from $[-1,1]$ and then rescale $\G$'s operator norm to $1$. 
	``\alg/\OMWU-eta=$\eta$'' represents the curve of \alg/\OMWU with learning rate $\eta$. The configuration order in the legend is consistent with the order of the curves. 
	For \OMWU, $\eta\geq11$ makes the algorithm diverge. 
	The plot confirms the linear convergence of \OMWU and \alg, although \alg is generally observed to converge faster than \OMWU. 
	%The left graph shows the point-wise convergence results of \OMWU and \alg. The huge oscillation of \alg with $\eta=0.5$ is due to the numeric precision issue. The right graph shows the duality-gap convergence results of \OMWU and \alg. 
	}
	\label{fig:MGsimplex}
\end{figure}

We compare the performances of \alg and \OMWU. For both algorithms, we choose a series of different learning rates and compare their performances, as shown in \pref{fig:MGsimplex}. 
The $x$-axis represents time step $t$, and 
the $y$-axis represents $\ln(\KL(\z^*, \z_t))$ % the KL-divergence between $\z^*$ and $\z_t$ after taking the natural logarithm, 
(we observe similar results using $\dist(\z^*, \z_t)$ or the duality gap as the measure;
see \pref{app:a1}).
%while the $y$-axis of the right plot represents the logarithmic duality gap at each round $t$, which is 
%$
%\ln (\alpha_f(\z_t))=\ln \left(\max_j( \G^\top \x_t)_j-\min_i(\G\y_t)_i\right).
%$
Note that here we approximate $\z^*$ by running \alg for much more iterations and taking the very last iterate.
We also verify that the iterates of \OMWU converge to the same point as \alg.
%So the distance measure we use is reasonable.

From \pref{fig:MGsimplex}, we see that all curves eventually become a straight line, supporting our linear convergence results.
Generally, the slope of the straight line is larger for a larger learning rate $\eta$.
However, the algorithm diverges when $\eta$ exceeds some value (such as $11$ for the case of \OMWU). 
Comparing \OMWU and \alg, we see that \alg converges faster, which is also consistent with our theory if one compares the bounds in \pref{thm: point convergence omwu} and \pref{thm: point convergence} (with the value of the constants revealed in the proofs).
We find this observation interesting, since \OMWU is usually considered more favorable for problems defined over the simplex, especially in terms of regret minimization.
Our experiments suggest that, however, in terms of last-iterate convergence, \alg might perform even better than \OMWU.

%From \pref{fig:MGsimplex}, we have the following observations. 
%First, considering convergence to Nash Equilibrium, the curves of \alg with learning rate $\eta\in \{\frac{1}{8}, \frac{1}{4}, \frac{1}{2}\}$ are straight lines when $t> 2\times 10^5$ (the oscillation of $\eta=0.5$ is due to the numeric precision issue), supporting the linear convergence of the distance shown in \pref{thm: point convergence}. 
%The curves of \OMWU are also straight lines when $t$ is large. We tune a series of learning rates, and the slope is quite small when $\eta$ is small. This also matches our theoretical results showing that \OMWU has a worse slope constant. When $\eta$ grows, the slope also increases, though these values are not commonly used in practice, such as $\eta=5$ or $\eta = 10$.

%Second, for the duality gap, the graph is roughly the same as the one for point-wise convergence. 
%One minor difference for the plot of duality gap is that the curves for both algorithms are slightly oscillating, although they still converge linearly overall.

%\begin{wrapfigure}{r}{0.5\textwidth}

%\end{wrapfigure}

%% file: appendix.tex
%!TEX root=iclr2021_conference.tex
\include{appendixA}
\input{appendixB}
\input{appendixC}
\input{appendixD}

\input{appendixE}

\input{appendix_metric}

%\input{appendix_unbounded}
\input{appendix_bilinear}
\input{appendixG}
\input{appendixH}
\input{appendixI}

%% file: appendixA.tex
\section[Experiments]{More Experiment Results} %\footnote{All source codes can be found at \url{https://github.com/bahh723/OGDA-last-iterate}.} %hide for submission 
\label{app:experiment}

\subsection{More empirical results for Matrix Games}\label{app:a1}
Here, we provide more plots for the same matrix game experiment described in \pref{sec:experiment}.
Specifically, the left plot in \pref{fig:MGsimplexL2} shows the convergence with respect to $\ln\|\z_t - \z^*\|$, while the right plot shows the convergence with respect to the logarithm of the duality gap $\ln (\alpha_f(\z_t))=\ln \left(\max_j( \G^\top \x_t)_j-\min_i(\G\y_t)_i\right)$.
One can see that the plots are very similar to those in \pref{fig:MGsimplex}.

\begin{figure}[h]
\centering
\includegraphics[width=0.49\textwidth]{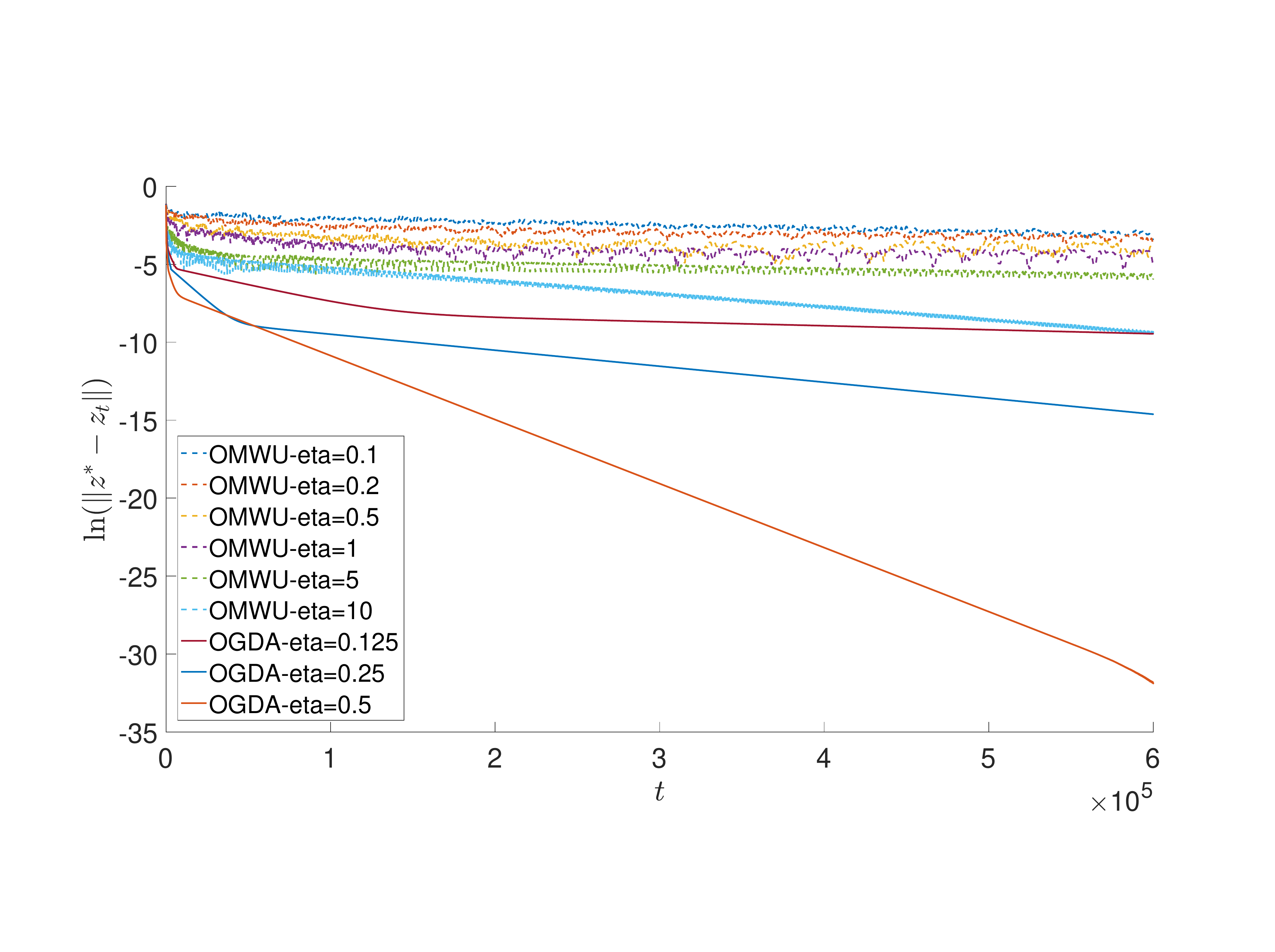}
\includegraphics[width=0.49\textwidth]{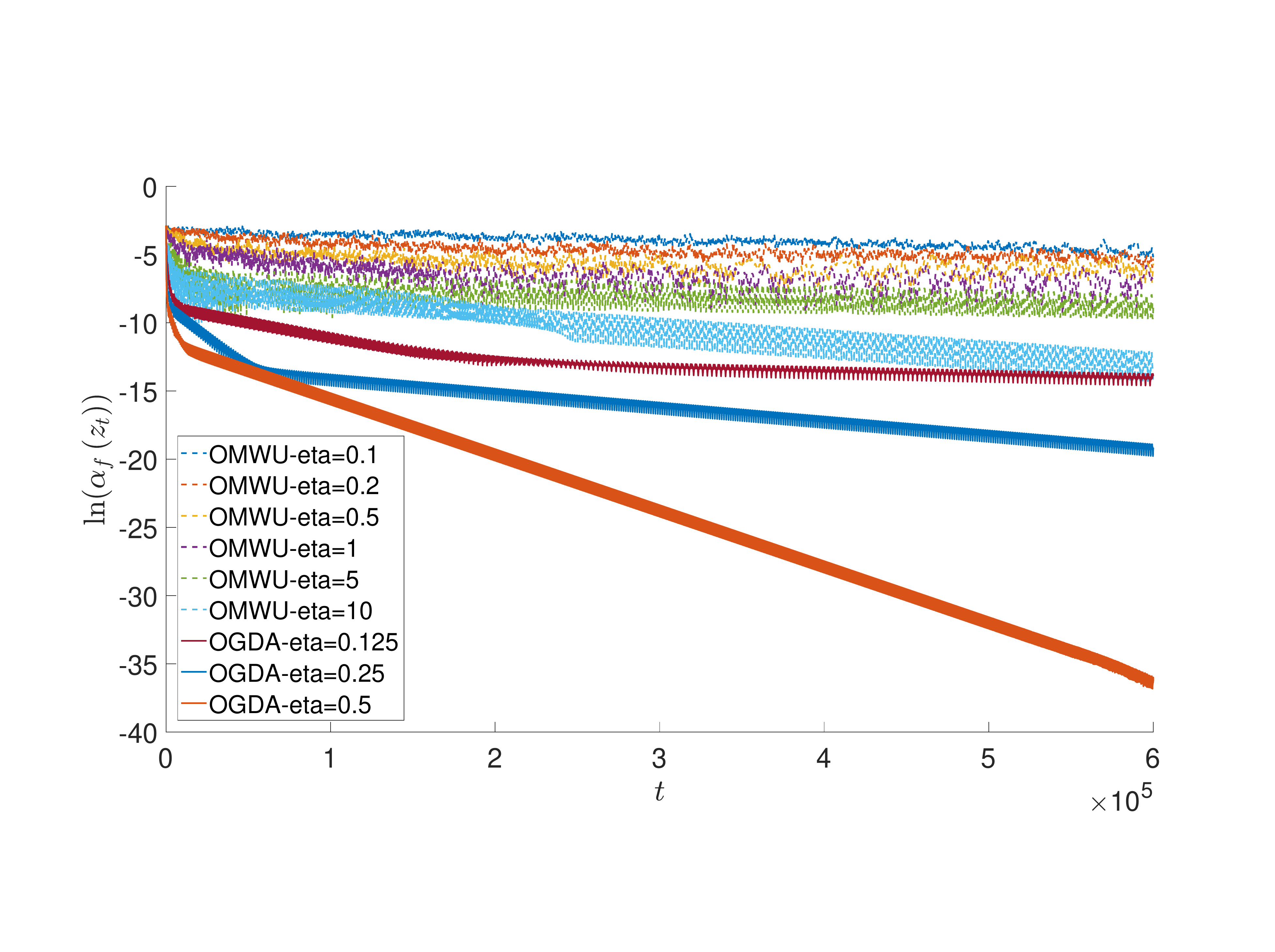}
\caption{Experiments of \alg and \OMWU with different learning rates on a matrix game $f(\x,\y)=\x^\top \G\y$, where we generate $\G\in \mathbb{R}^{32\times 32}$ with each entry $G_{ij}$ drawn uniformly at random from $[-1,1]$ and then rescale $\G$'s operator norm to $1$. ``\alg/\OMWU-eta=$\eta$'' represents the curve of \alg/\OMWU with learning rate $\eta$. 
The configuration order in the legend is consistent with the order of the curves. 
	For \OMWU, $\eta\geq11$ makes the algorithm diverge. 
	The plot confirms the linear convergence of \OMWU and \alg, although \alg is generally observed to converge faster than \OMWU. 
}
\label{fig:MGsimplexL2}
\end{figure}

\subsection{Matrix Game on Curved Regions}
Next, we conduct experiments on a bilinear game similar to the one constructed in the proof of \pref{thm: bilinear-general}.
Specifically, the bilinear game is defined by \begin{align*}
f(\x,\y)={x_2}{y_1}-{x_1}{y_2},\quad\mathcal{X}=\mathcal{Y}\triangleq\{(a,b),0\le a \le \tfrac{1}{2}, 0\le b\le \tfrac{1}{2^n},~a^n\le b\}.
\end{align*}
For any positive integer $n$, the equilibrium point of this game is $(0,0)$ for both $\x$ and $\y$.
Note that in \pref{thm: bilinear-general}, we prove that \alg only converges at a rate no better than $\Omega(1/t^2)$ in this game when $n=2$.

\pref{fig:MGCurved} shows the empirical results for various values of $n$.
In this figure, we plot $\|\z_t-\z^*\|$ versus time step $t$ in log-log scale.
Note that in a log-log plot, a straight line with slope $s$ implies a convergence rate of order $\mathcal{O}(t^s)$, that is, a sublinear convergence rate.
It is clear from \pref{fig:MGCurved} that \alg indeed converges sublinearly for all $n$,
supporting our \pref{thm: bilinear-general}.

\begin{figure}
    \centering
	\includegraphics[width=0.7\textwidth]{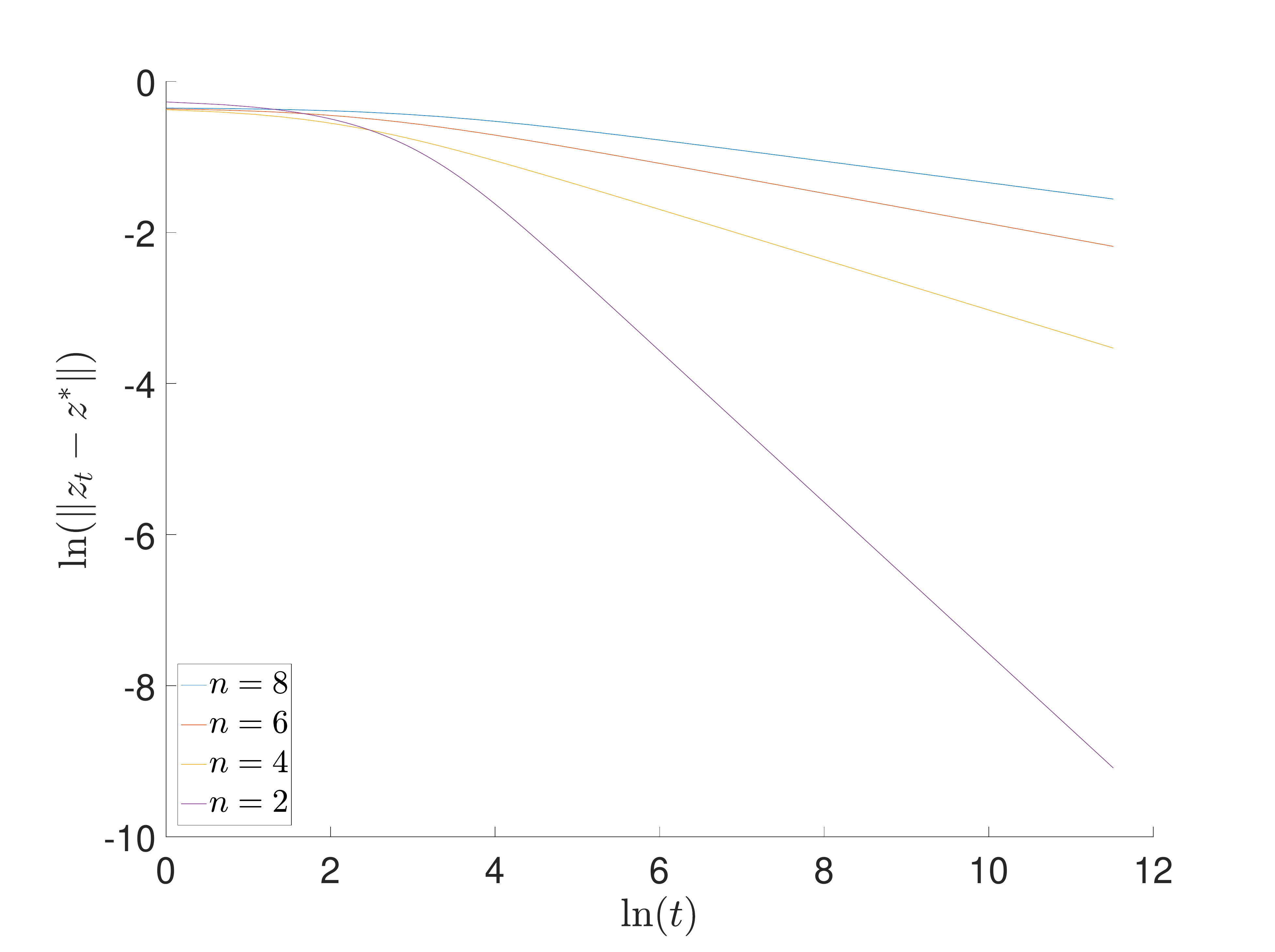}\caption{Experiments of \alg on matrix games with curved regions where $f(\x,\y)={x_2}{y_1}-{x_1}{y_2},\quad\mathcal{X}=\mathcal{Y}\triangleq\{(a,b),0\le a\le \frac{1}{2}, 0\le b\le \frac{1}{2^n},~a^n\le b\}$, and $n=2,4,6,8$. This figure is a log-log plot of $\|\z_t-\z^*\|$ versus $t$, and it indicates sublinear convergence rates of \alg in all these games.} 
	\label{fig:MGCurved}
\end{figure}

\begin{figure}
    \centering
	\includegraphics[width=0.7\textwidth]{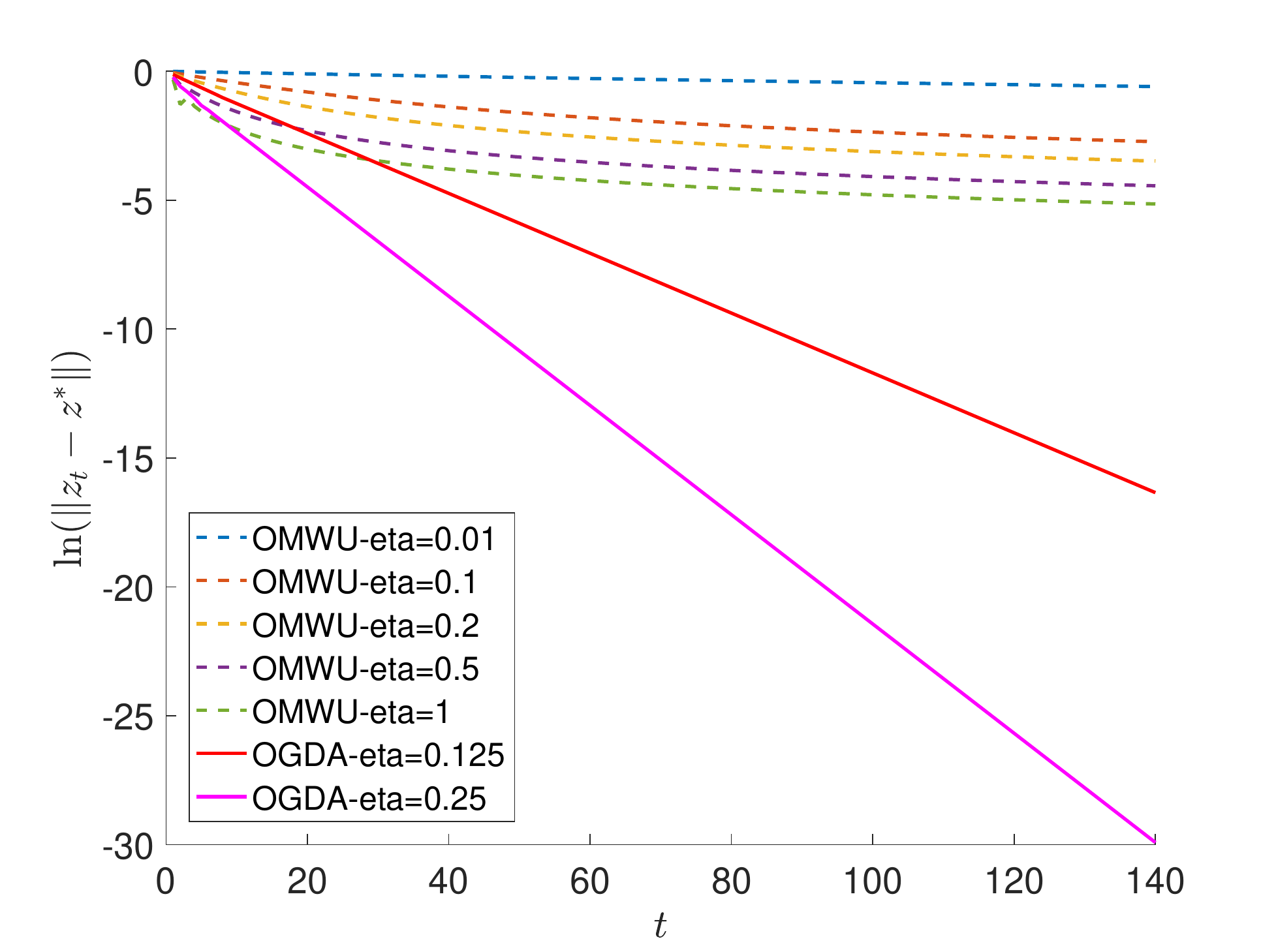}\caption{Experiments on a strongly-convex-strongly-concave game where $f(\x,\y)=x_1^2-y_1^2+2x_1y_1$ and $\mathcal{X}=\mathcal{Y}\triangleq\{(a,b),0\le a,b\le 1,~a+b=1\}$. The figure is showing $\ln\|\z_t-\z^*\|$ versus the time step $t$. The result shows that  \alg enjoys linear convergence and outperforms \OMWU in this case. }
	\label{fig:strongly}
\end{figure}

\subsection{Strongly-convex-strongly-concave Games}
In this section, we use the same experiment setup for strongly-convex-strongly-concave games in \citep{Lei2020Last}, where
\begin{align*}
f(\x,\y)=x_1^2-y_1^2+2x_1y_1,\quad\text{and}\quad\mathcal{X}=\mathcal{Y}\triangleq\{(a,b),0\le a,b\le 1,~a+b=1\}.
\end{align*}
The equilibrium point is $(0,1)$ for both $\x$ and $\y$. In \pref{fig:strongly}, we present the log plot of $\|\z_t-\z^*\|$ versus time step $t$ and compare \alg with  \OMWU using different learning rates as in \pref{app:a1}. %We also tune the learning rate respectively to make the comparison more reasonable.
The straight line of \alg implies that \alg algorithm converges exponentially fast, supporting \pref{thm: strongly convex} and \pref{thm: point convergence}.
Also note that here, \alg outperforms \OMWU, which is different from the empirical results shown in \citep{Lei2020Last}.
We hypothesize that this is because they use a different version of \alg.

\subsection{An Example with $\beta > 0$ for \typetwo}
We also consider the toy example in \pref{thm: beta ge 0}, where 
$f(\x,\y)=x_1^{2n}-x_1y_1-y_1^{2n}$ for some integer $ n \ge 2$
and $\mathcal{X}=\mathcal{Y}\triangleq\{(a,b),0\le a,b\le 1,~a+b=1\}$.
The equilibrium point is $(0,1)$ for both $\x$ and $\y$. We prove in \pref{thm: beta ge 0} that \typetwo does not hold for $\beta = 0$ but does hold for $\beta=2n-2$. 

The point-wise convergence result is shown in~\pref{fig:Nonstrongly}, which is again a log-log plot of $\|\z_t-\z^*\|$ versus time step $t$. One can observe that the convergence rate of \alg is sublinear, supporting our theory again. %Moreover, we can see that the slope of the curve for each $n$ is approximately $-\frac{1}{2n-2}=-\frac{1}{\beta}$, which means that $\|\hat{z}-z^\star\|\approxeq \Theta\left(t^{-\frac{1}{\beta}}\right)$. This result matches the upper bound shown in \pref{thm: point convergence}, showing that our theoretical result in~\pref{thm: point convergence} is tight in this case. 
\begin{figure}
    \centering
    \includegraphics[width=0.7\textwidth]{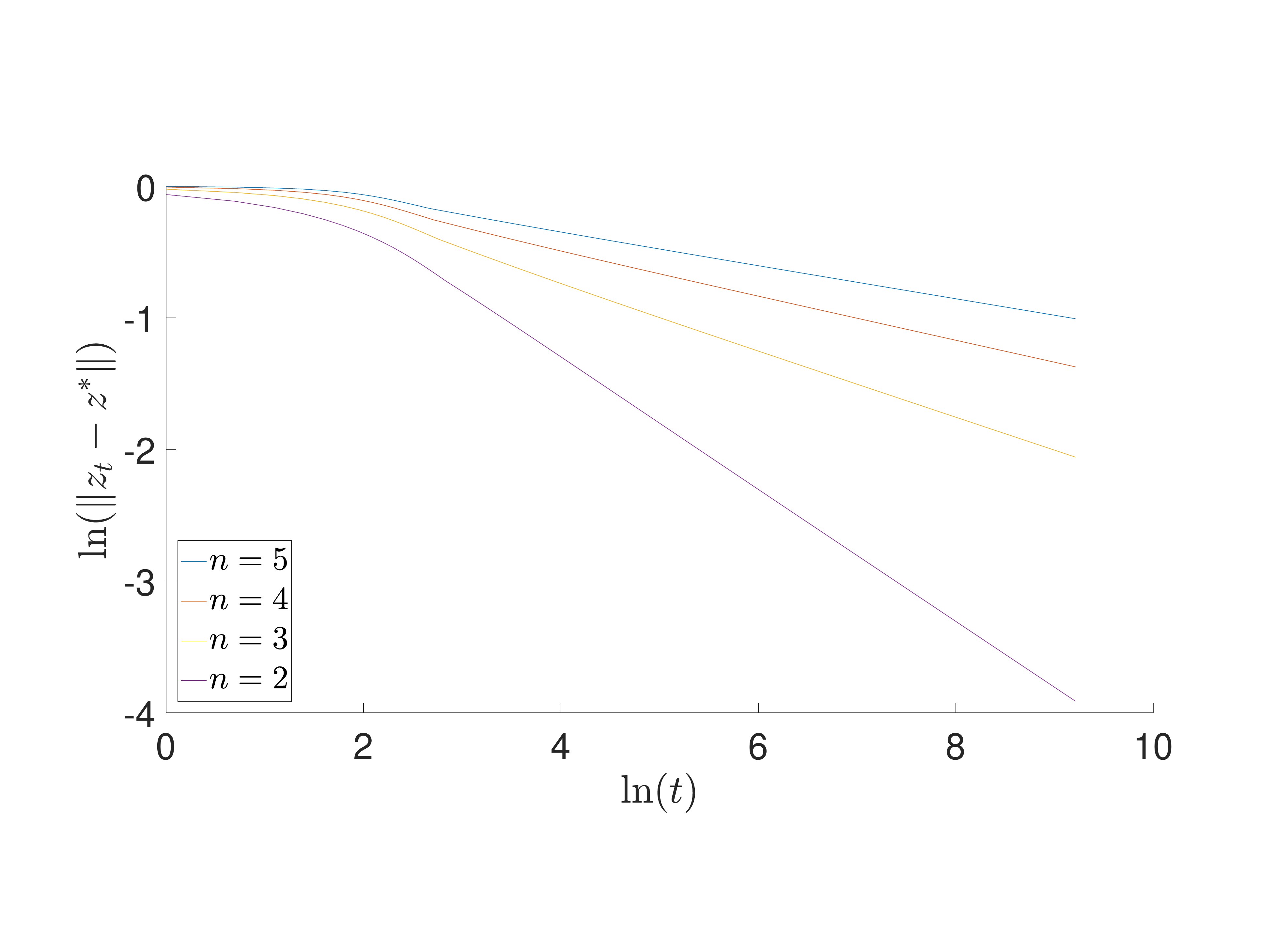}\caption{Experiments of \alg on a set of games satisfying \typetwo with $\beta > 0$, where $f(\x,\y)=x_1^{2n}-x_1y_1-y_1^{2n}$ for some integer $ n \ge 2$
and $\mathcal{X}=\mathcal{Y}\triangleq\{(a,b),0\le a,b\le 1,~a+b=1\}$. The result shows that \alg converges to the Nash equilibrium with sublinear rates in these instances. %Moreover, the slope of the curve is approximately $-\frac{1}{2n-2}$ for each $n$, matching the result shown in~\pref{thm: point convergence}.  
        }
        \label{fig:Nonstrongly}
\end{figure}

\modify{
\subsection{Matrix Games with Multiple Nash Equilibria}

Finally, we provide empirical results for \alg and \OMWU in matrix games with multiple Nash equilibria, even though theoretically we only prove linear convergence results for \OMWU assuming that the Nash equilibrium is unique.
We consider the following game matrix
$$
G = \begin{bmatrix}
	0 & -1 & 1 & 0 & 0 \\
	1 & 0 & -1 & 0 & 0 \\
	-1 & 1 & 0 & 0 & 0 \\
	-1 & 1 & 0 & 2 & -1\\
	-1 & 1 & 0 & -1 & 2
\end{bmatrix}.
$$

The value of $G$ is $0$. To verify this, consider $\x_0=\y_0=\begin{bmatrix} \frac{1}{3} & \frac{1}{3} & \frac{1}{3} & 0 & 0 \end{bmatrix}$. Then we have for $\max_{\y\in \Delta_5}\x_0^\top G\y = \min_{\x\in \Delta_5}\x^\top G\y_0 = 0$. Direct calculation gives the following set of Nash equilibria.
\begin{align*}
	\mathcal{X}^* &= \left\{\x_0 \right\},\\
	\mathcal{Y}^* &= \left\{\y\in \Delta_5: y_1=y_2=y_3;\;\; \frac{1}{2}y_5\leq y_4\leq 2y_5\right\}.
\end{align*}

\pref{fig:MGContinuumNE} shows the point-wise convergence result. $\Pi_{\mathcal{Z}^*}(z_t)$ is the projection of $z_t$ on the set of Nash qquilibria. One can observe from the plots that both \alg and \OMWU achieve linear convergence rate in this example. We thus conjecture that the uniqueness assumption for \pref{thm: point convergence omwu} can be further relaxed. 

\begin{figure}[h]
\centering
\includegraphics[width=0.6\textwidth]{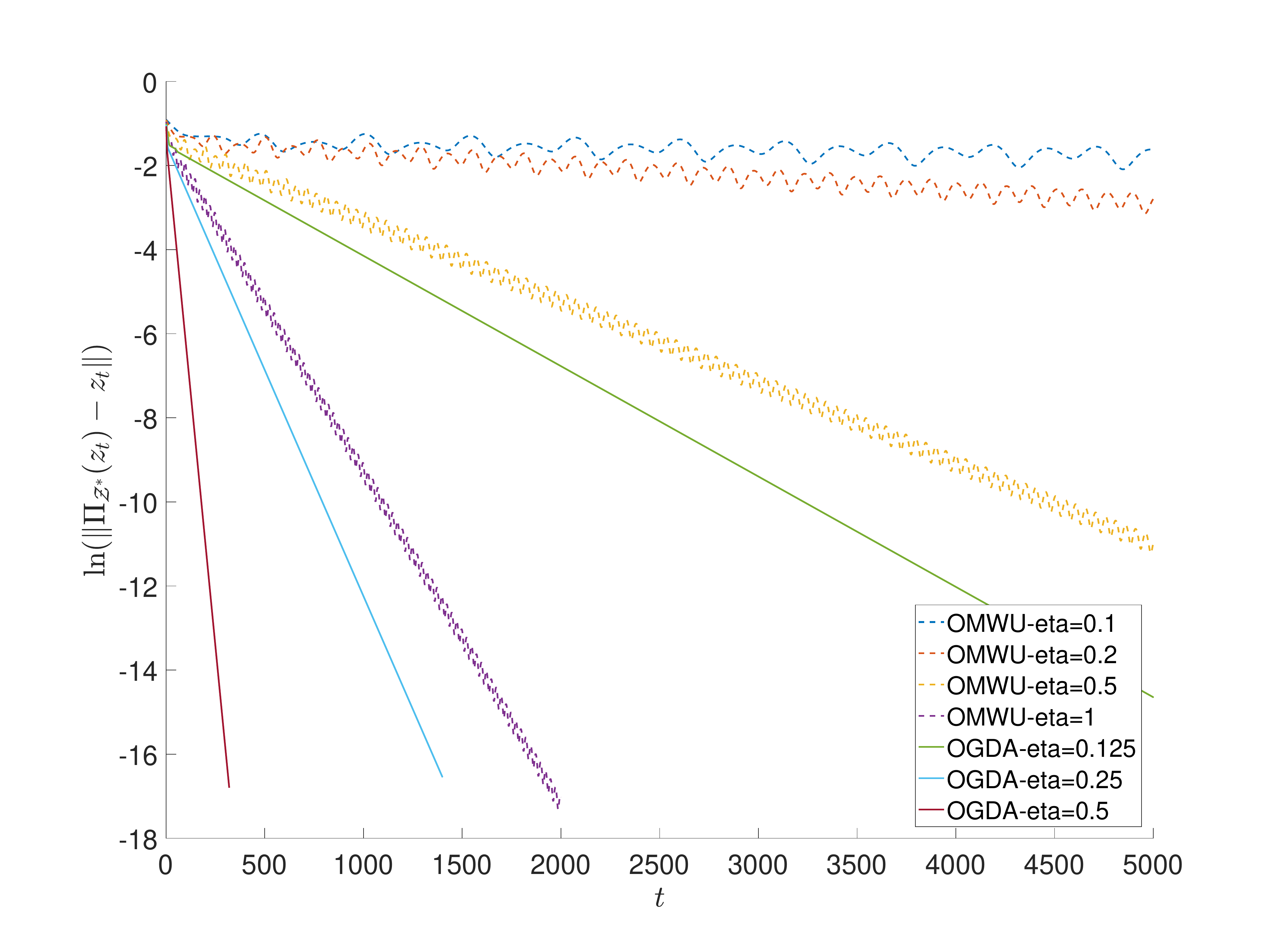}
\includegraphics[width=0.6\textwidth]{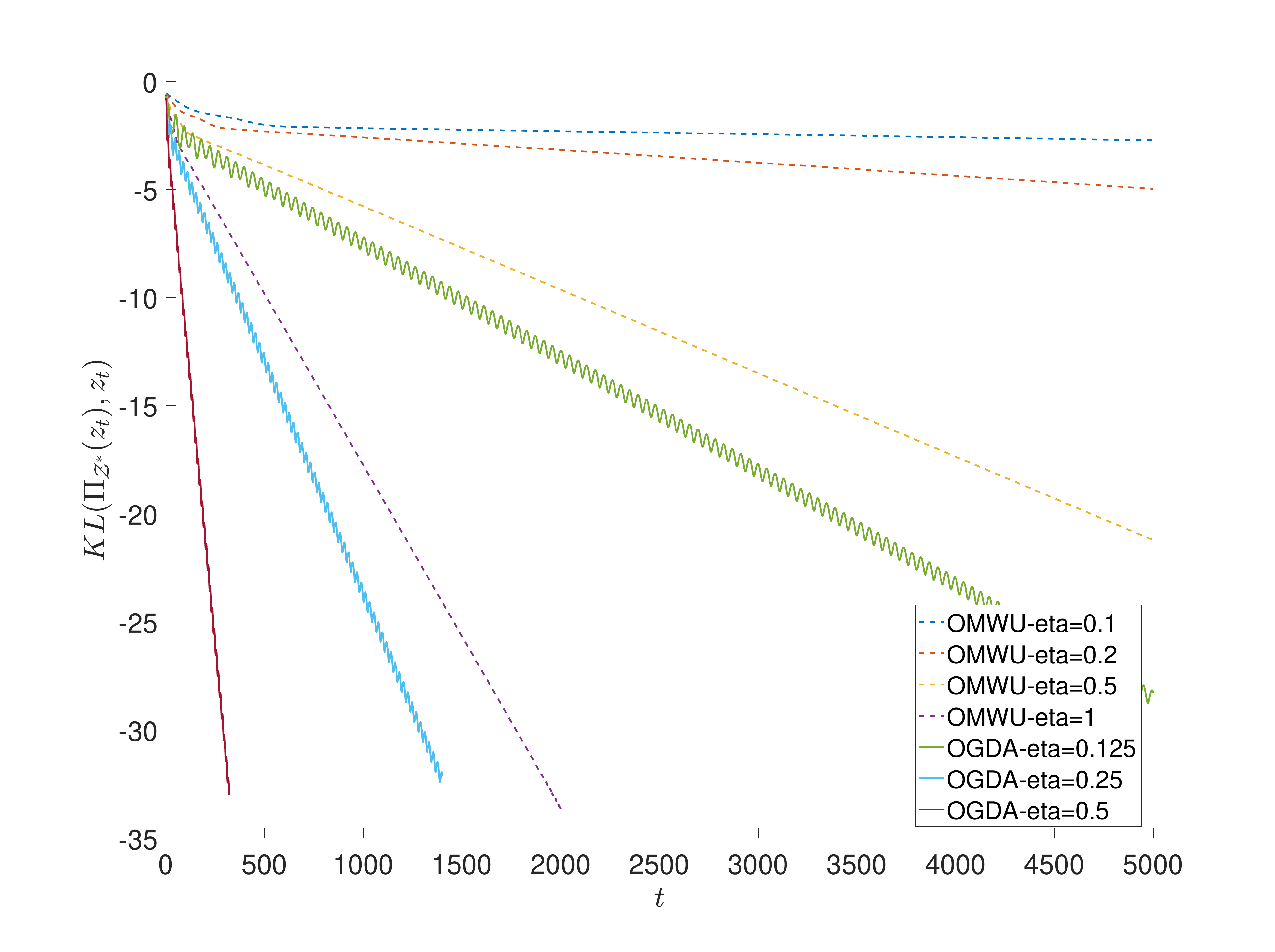}
\includegraphics[width=0.6\textwidth]{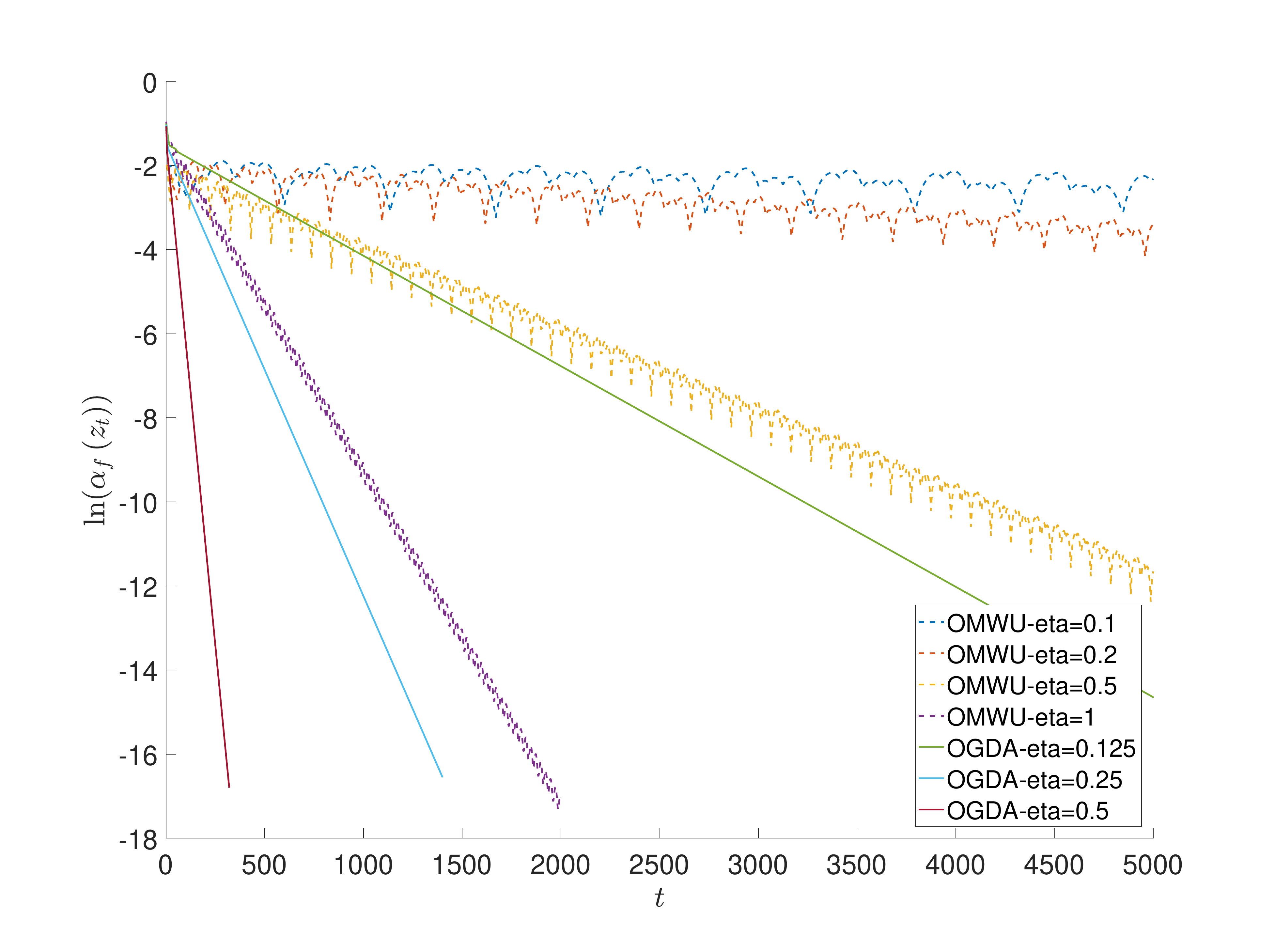}
\caption{\modify{Experiments of \alg and \OMWU with different learning rates on a matrix game with multiple Nash equilibria. ``\alg/\OMWU-eta=$\eta$'' represents the curve of \alg/\OMWU with learning rate $\eta$. We observe from these plots that both \alg and \OMWU enjoy a linear convergence rate, even though we are only able to show the linear convergence of \OMWU under the uniqueness assumption.}
}
\label{fig:MGContinuumNE}
\end{figure}
}

%% file: appendixB.tex
\section{Lemmas for Optimistic Mirror Descent}
\label{app:regret bound sufficient decrease}

We prove \pref{lem: regret bound omwu} in this section.
To do so, we use the following two lemmas.

\begin{lemma}
	\label{lem: OGD useful lemma}
	Let $\calA$ be a convex set and $\bu' = \argmin_{\bu'\in\calA}\left\{\inner{\bu', \g} + D_\psi(\bu', \bu) \right\}$. Then for any $\bu^*\in\calA$, 
	\begin{align}
	\inner{\bu'-\bu^*, \g} \leq D_\psi(\bu^*, \bu) - D_\psi(\bu^*, \bu') - D_\psi(\bu', \bu).    \label{eq: optimality OMD}
	\end{align}
\end{lemma}

\begin{proof}
     Since $D_\psi(\bu',\bu)=\psi(\bu')-\psi(\bu) - \inner{\nabla \psi(\bu), \bu'-\bu}$, by the first-order optimality condition of $\bu'$, we have 
     \begin{align*}
          \left(\g+\nabla \psi(\bu') - \nabla\psi(\bu)\right)^\top (\bu^*-\bu')\geq 0. 
     \end{align*}
     On the other hand, notice that the right-hand side of \pref{eq: optimality OMD} is 
     \begin{align*}
          &\psi(\bu^*) - \psi(\bu) - \inner{\nabla \psi(\bu), \bu^*-\bu} \\
          &\qquad  - \psi(\bu^*) + \psi(\bu') + \inner{\nabla \psi(\bu'), \bu^*-\bu'} \\
          &\qquad - \psi(\bu') + \psi(\bu) + \inner{\nabla \psi(\bu), \bu'-\bu}\\
          &= \inner{\nabla \psi(\bu')-\nabla \psi(\bu), \bu^*-\bu'}. 
     \end{align*}
     Therefore, \pref{eq: optimality OMD} is equivalent to $\inner{\g + \nabla \psi(\bu')-\nabla \psi(\bu), \bu^*-\bu'}\geq 0$, which we have already shown above. 
\end{proof}

%\begin{proof}
%	Since $u'=\argmin_{u'\in\calA}\|u'-u+g\|^2$, by the first-order optimality condition for $u'$, we have 
%	\begin{align*}
%	(u'-u+g)^\top (u^*-u')\geq 0. 
%	\end{align*}
%	Note that the right-hand side minus the left-hand side of \pref{eq: optimality OMD} is exactly equal to $2(u'-u+g)^\top (u^*-u')$. Thus \pref{eq: optimality OMD} holds. 
%\end{proof}

%\begin{lemma}
%	\label{lem: OGD useful lemma}
%	Let $\calA$ be a convex set, and let $u' = \Pi_\calA(u-g)$. Then for any $u^*\in\calA$, 
%	\begin{align}
%	2\inner{u'-u^*, g} \leq \|u^*-u\|^2 - \|u^*-u'\|^2 - \|u'-u\|^2.    \label{eq: optimality OMD}
%	\end{align}
%\end{lemma}
%\begin{proof}
%	Since $u'=\argmin_{u'\in\calA}\|u'-u+g\|^2$, by the first-order optimality condition for $u'$, we have 
%	\begin{align*}
%	(u'-u+g)^\top (u^*-u')\geq 0. 
%	\end{align*}
%	Note that the right-hand side minus the left-hand side of \pref{eq: optimality OMD} is exactly equal to $2(u'-u+g)^\top (u^*-u')$. Thus \pref{eq: optimality OMD} holds. 
%\end{proof}

\begin{lemma}
	\label{lem: proj optimal}
	Suppose that $\psi$ satisfies $D_\psi(\x,\x')\geq \frac{1}{2}\|\x-\x'\|_p^2$ for some $p\geq 1$, and let $\bu, \bu_1, \bu_2\in \calA$ (a convex set) be related by the following: 
	\begin{align*}
	\bu_1 &= \argmin_{\bu'\in\calA} \left\{\inner{\bu', \g_1} + D_\psi(\bu', \bu)\right\}, \\
	\bu_2 &= \argmin_{\bu'\in\calA} \left\{\inner{\bu', \g_2} + D_\psi(\bu', \bu)\right\}.
	\end{align*}
	Then we have
	\begin{align*}
	\|\bu_1-\bu_2\|_p\leq \|\g_1-\g_2\|_q,  
	\end{align*}
	where $q\geq 1$ and $\frac{1}{p}+\frac{1}{q}=1$. 
\end{lemma}

%\begin{lemma}
%	\label{lem: proj optimal}
%	Let $u, u_1, u_2\in \calA$ (a convex set) be related by the following: 
%	\begin{align*}
%	u_1 = \Pi_{\calA} (u - g_1), \\
%	u_2 = \Pi_{\calA} (u - g_2).
%	\end{align*}
%	Then we have
%	\begin{align*}
%	\|u_1-u_2\|\leq \|g_1-g_2\|. 
%	\end{align*}
%\end{lemma}

\begin{proof}
	By the first-order optimality conditions of $\bu_1$ and $\bu_2$, we have
	\begin{align*}
	\inner{\nabla \psi(\bu_1)- \nabla \psi(\bu) + \g_1, \bu_2-\bu_1} \geq 0, \\
	\inner{\nabla \psi(\bu_2) -\nabla \psi(\bu) + \g_2, \bu_1-\bu_2} \geq 0.
	\end{align*}
	Summing them up and rearranging the terms, we get 
	\begin{align}\label{eq: tmptmp2}
     \inner{\bu_2-\bu_1, \g_1-\g_2} \geq \inner{\nabla\psi(\bu_1) - \nabla\psi(\bu_2), \bu_1-\bu_2}.  
	\end{align}
	By the condition on $\psi$, we have $\inner{\nabla\psi(\bu_1), \bu_1- \bu_2}\geq \psi(\bu_1)-\psi(\bu_2) + \frac{1}{2}\|\bu_1-\bu_2\|_p^2$ and $\inner{\nabla\psi(\bu_2), \bu_2-\bu_1}\geq \psi(\bu_2)-\psi(\bu_1) +\frac{1}{2}\|\bu_1-\bu_2\|_p^2$. Summing them up we get $\inner{\nabla\psi(\bu_1) - \nabla\psi(\bu_2), \bu_1-\bu_2} \geq \|\bu_1-\bu_2\|_p^2$. Combining this with \pref{eq: tmptmp2} we get 
	\begin{align*}
	     \inner{\bu_2-\bu_1, \g_1-\g_2} \geq \|\bu_1-\bu_2\|_p^2. 
	\end{align*}
	Since $\inner{\bu_2-\bu_1, \g_1-\g_2}\leq \|\bu_1-\bu_2\|_p\|\g_1-\g_2\|_q$ by H\"{o}lder's inequality, we further get $\|\bu_1-\bu_2\|_p \leq \|\g_1-\g_2\|_q$. 
\end{proof}

%\begin{proof}
%	By the first-order optimality conditions of $u_1$ and $u_2$, we have
%	\begin{align*}
%	(u_1-u + g_1)\cdot (u_2-u_1) \geq 0, \\
%	(u_2-u + g_2) \cdot (u_1-u_2) \geq 0.
%	\end{align*}
%	Summing them up and rearranging, we get 
%	\begin{align*}
%	-\|u_1-u_2\|^2 + \inner{u_2-u_1, g_1-g_2} \geq 0.  
%	\end{align*}
%	Rearranging, we get $\|u_1-u_2\|^2 \leq \inner{u_2-u_1, g_1-g_2} \leq \|u_1-u_2\|\|g_1-g_2\|$. Therefore, $\|u_1-u_2\|\leq \|g_1-g_2\|$. 
%\end{proof}

\begin{proof}[Proof of \pref{lem: regret bound omwu}]
	Considering \pref{eq: omda update 6}, and using \pref{lem: OGD useful lemma} with $\bu=\zp_t$, $\bu'=\zp_{t+1}$, $\bu^*=\z$, and $\g =\eta F(\z_t) $, we get 
	\begin{align*}
	\eta F(\z_t)^\top (\zp_{t+1}-\z)  \leq D_\psi(\z, \zp_t) - D_\psi(\z, \zp_{t+1}) - D_\psi(\zp_{t+1}, \zp_t). 
	\end{align*}
	Considering \pref{eq: omda update 5}, and using \pref{lem: OGD useful lemma} with $\bu=\zp_{t}$, $\bu'=\z_{t}$, $\bu^*=\zp_{t+1}$, and $\g = \eta F(\z_{t-1})$, we get 
	\begin{align*}
	\eta F(\z_{t-1})^\top (\z_t - \zp_{t+1})\leq D_\psi(\zp_{t+1}, \zp_{t}) - D_\psi(\zp_{t+1}, \z_t) - D_\psi(\z_t, \zp_t). 
	\end{align*}
	Summing up the two inequalities above, and adding $\eta \left(F(\z_t)-F(\z_{t-1})\right)^\top (\z_t-\zp_{t+1})$ to both sides, we get  
	\begin{align}
	&\eta F(\z_t)^\top (\z_t-\z)  \nonumber  \\
	&\leq D_\psi(\z, \zp_t) - D_\psi(\z, \zp_{t+1}) - D_\psi(\zp_{t+1}, \z_t) - D_\psi(\z_t, \zp_t) + \eta \left(F(\z_t)-F(\z_{t-1})\right)^\top (\z_t-\zp_{t+1}).  \label{eq: to be continue 1}
	\end{align} 
	Using \pref{lem: proj optimal} with $\bu=\xp_t$, $\bu_1=\x_t$, $\bu_2=\xp_{t+1}$, $\g_1 = \eta \nabla_{\x}f(\z_{t-1})$ and $\g_2 = \eta \nabla_{\x}f(\z_{t})$, we get $\|\x_t-\xp_{t+1}\|_p\leq \eta \|\nabla_{\x}f(\z_{t-1}) -\nabla_{\x}f(\z_{t})\|_q$. Similarly, we have $\|\y_t-\yp_{t+1}\|_p\leq \eta \|\nabla_{\y}f(\z_{t}) -\nabla_{\y}f(\z_{t-1})\|_q$. Therefore, by H{\"o}lder's inequality, we have
	\begin{align*}
	&\eta \left(F(\z_t)- F(\z_{t-1})\right)^\top (\z_t-\zp_{t+1}) \\
	& \leq \eta \|\x_t-\xp_{t+1}\|_p\|\nabla_{\x}f(\z_{t-1}) -\nabla_{\x}f(\z_{t})\|_q + \eta \|\y_t-\yp_{t+1}\|_p\|\nabla_{\y}f(\z_{t-1}) -\nabla_{\y}f(\z_{t})\|_q \\
	& \leq \eta^2 \|\nabla_{\x}f(\z_{t-1}) -\nabla_{\x}f(\z_{t})\|_q^2 + \eta^2 \|\nabla_{\y}f(\z_{t-1}) -\nabla_{\y}f(\z_{t})\|_q^2 \\
	&= \eta^2 \distp{q}(F(\z_t),F(\z_{t-1})) \\
	&\leq \eta^2 L^2\distp{p}(\z_t,\z_{t-1})  \tag{by assumption}\\
	&\leq \frac{1}{64}\distp{p}(\z_t,\z_{t-1}) . \tag{by our choice of $\eta$}
	\end{align*} 
	Continuing from \pref{eq: to be continue 1}, we then have
	\begin{align*}
	&\eta F(\z_t)^\top (\z_t-\z) \\
	&\leq D_\psi(\z, \zp_t) - D_\psi(\z, \zp_{t+1}) - D_\psi(\zp_{t+1}, \z_t) - D_\psi(\z_t, \zp_t) + \frac{1}{64}\distp{p}(\z_t,\z_{t-1}) \\
	&\leq D_\psi(\z, \zp_t) - D_\psi(\z, \zp_{t+1}) - D_\psi(\zp_{t+1}, \z_t) - D_\psi(\z_t, \zp_t) + \frac{1}{32}\distp{p}(\z_t,\zp_t)+ \frac{1}{32} \distp{p}(\zp_t,\z_{t-1})   \tag{$\|\bu+\bv\|_p^2 \leq \left(\|\bu\|_p + \|\bv\|_p\right)^2 \leq 2\|\bu\|_p^2+ 2\|\bv\|_p^2$} \\
	&\leq D_\psi(\z, \zp_t) - D_\psi(\z, \zp_{t+1}) - D_\psi(\zp_{t+1}, \z_t) - D_\psi(\z_t, \zp_t) + \frac{1}{16}D_\psi(\z_t, \zp_t) + \frac{1}{16} D_\psi(\zp_t, \z_{t-1}) \tag{by the assumption on $\psi$}\\
	&= D_\psi(\z, \zp_t) - D_\psi(\z, \zp_{t+1}) - D_\psi(\zp_{t+1}, \z_t) - \frac{15}{16}D_\psi(\z_t, \zp_t)+ \frac{1}{16} D_\psi(\zp_t, \z_{t-1}). 
	\end{align*}
	This concludes the proof.
\end{proof}

%% file: appendixC.tex
\section{An Auxiliary Lemma on Recursive formulas}

Here, we provide an auxiliary lemma that gives an explicit bound based on a particular recursive formula. This will be useful later for deriving the convergence rate.

\begin{lemma}
     \label{lem: recursion lemma p > 1}
     Consider a non-negative sequence $\{B_t\}_{t=1,2,\cdots}$ that satisfies for some $p>0$ and $q>0$, 
     \begin{itemize}
          \item $B_{t+1}\leq B_t - q B_{t+1}^{p+1}$,  \ \ $\forall t\geq 1$
          \item $q(1+p)B_{1}^p\leq 1$.  
     \end{itemize}
     Then $B_t\leq ct^{-\frac{1}{p}}$, where $c=\max\left\{B_1,\left(\frac{2}{qp}\right)^{\frac{1}{p}}\right\}$.
\end{lemma}
\begin{proof}
     We first prove that $B_{t+1}\leq B_t - \frac{q}{2}B_t^{p+1}$. Notice that since $B_t$ are all non-negative, by the first condition, we have $B_{t+1}\leq B_t\leq \cdots \leq B_1$. Using the fundamental theorem of calculus, we have
	\begin{align*}
	B_t^{p+1} - B_{t+1}^{p+1} = \int_{B_{t+1}}^{B_t} \left(\frac{\der }{\der x} x^{p+1}\right) \der x 
	= (p+1)\int_{B_{t+1}}^{B_t} x^p \der x  
	\leq (p+1)(B_{t} - B_{t+1})B_t^{p}
	\end{align*}
	and thus
	\begin{align*}
	B_{t+1}
	\leq B_t - qB_{t+1}^{p+1}  
	\leq B_t - qB_{t}^{p+1} + q(p+1)\left(B_t - B_{t+1}\right)B_t^{p}.  
	\end{align*}
	By rearranging, we get 
	\begin{align*}
	B_{t+1} 
	&\leq\left( 1 - \frac{qB_t^{p}}{1+q(1+p)B_t^{p}}\right)B_t \leq \left( 1 - \frac{qB_t^{p}}{2}\right)B_t = B_t - \frac{q}{2}B_t^{p+1},
	\end{align*}
	where the last inequality is because $q(1+p)B_t^p \leq q(1+p)B_1^p\leq 1$. 
	
	Below we use induction to prove $B_t  \leq c t^{-\frac{1}{p}}$, where $c=\max\left\{B_1,\left(\frac{2}{qp}\right)^{\frac{1}{p}}\right\}$.  This clearly holds for $t=1$. Suppose that it holds for $1, \ldots, t$.
	Note that the function $f(B_t) = \left(1 - \frac{q}{2}B_{t}^{p}\right) B_t$ is increasing in $B_t$ as $f'(B_t) = 1 - \frac{q(p+1)}{2}B_t^p \geq 1 - \frac{q(p+1)}{2}B_1^p \geq 0$. 
	Therefore, we apply the induction hypothesis and get
	\begin{align*}
	B_{t+1}
	&\leq \left(1 - \frac{q}{2}B_{t}^{p}\right) B_t 
	\leq \left(1 - \frac{q}{2} c^{p} t^{-1} \right) c t^{-\frac{1}{p}} \\
	&=c t^{-\frac{1}{p}} - \frac{q}{2} c^{p+1} t^{-1-\frac{1}{p}} 
     \leq c t^{-\frac{1}{p}} -\frac{c}{p} t^{-1-\frac{1}{p}} \tag{$\frac{c}{p}\leq \frac{q}{2}c^{p+1}$ by the definition of $c$}\\
	&\leq c(t+1)^{-\frac{1}{p}}, 
	\end{align*}
	where the last inequality is by the fundamental theorem of calculus: 
	\begin{align*}
	t^{-\frac{1}{p}} - (1+t)^{-\frac{1}{p}} 
	&= \int_{1+t}^t \left(\frac{\der}{\der x} x^{-\frac{1}{p}}\right) \der x 
	= \int_{1+t}^t \left(-\frac{1}{p}\right) x^{-1-\frac{1}{p}}\der x \\
	&= \int_{t}^{t+1} \frac{1}{p} x^{-1-\frac{1}{p}} \der x \leq \frac{1}{p} t^{-1-\frac{1}{p}}.
	\end{align*}
	This completes the induction.
\end{proof}

%% file: appendixD.tex
\section{Proofs of \pref{lem: KL sufficient decrease} and \pref{thm: point convergence omwu}}\label{app:OMWU_proofs}
In this section, we consider $f(\x,\y)=\x^\top \G \y$ with $\calX =\Delta_M$ and $\calY =\Delta_N$ being simplex and $\G\in [-1, 1]^{M\times N}$. We assume that $\G$ has a unique Nash equilibrium $\z^*=(\x^*,\y^*)$. 
The value of the game  is denoted as $\rho = \min_{\x\in\calX} \max_{\y\in\calY} \x^\top \G \y = \max_{\y\in\calY} \min_{\x\in\calX} \x^\top \G \y = \x^{*\top}\G\y^*$.

Before proving \pref{lem: KL sufficient decrease} and \pref{thm: point convergence omwu}, in \pref{sec: constant section}, we define some constants for later analysis; in \pref{sec: aux lemm}, we state more auxiliary lemmas, which are useful when proving \pref{lem: KL sufficient decrease} and \pref{thm: point convergence omwu} in \pref{sec: OMWU main part}. 

\subsection{Some Problem-dependent Constants}\label{sec: constant section}
First, we define a constant $\xi$ that is determined by $\G$. 
\begin{definition}\label{def: xi}
    \begin{align*}
    \xi\triangleq \min\left\{ \min_{i\notin\supp(\x^*)} (\G\y^*)_i - \rho, \quad \rho - \max_{i\notin\supp(\y^*)} (\G^\top \x^*)_i \right\} \in (0,1]. 
    \end{align*}
\end{definition}
The fact $\xi \leq 1$ can be shown by: 
\[
\xi \leq \frac{\min_{i\notin\supp(\x^*)} (\G\y^*)_i - \rho + \rho - \max_{i\notin\supp(\y^*)} (\G^\top \x^*)_i}{2} \leq \frac{\|\G\y^*\|_\infty + \|\G^\top \x^*\|_\infty}{2} \leq 1,
\]
while the fact $\xi>0$ is a direct consequence of Lemma C.3 of \citet{mertikopoulos2018cycles}, stated below. 
\begin{lemma}[Lemma C.3 of \citet{mertikopoulos2018cycles}]\label{lem: gap lemma}
     Let $\G\in\mathbb{R}^{M\times N}$ be a game matrix for a two-player zero-sum game with value $\rho$. Then there exists a Nash equilibrium $(\x^*, \y^*)$ such that
     \begin{align*}
          (\G\y^*)_i &= \rho  &&\forall i\in \supp(\x^*), \\
          (\G\y^*)_i &> \rho  &&\forall i\notin \supp(\x^*), \\
          (\G^\top \x^*)_i &= \rho  &&\forall i\in\supp(\y^*), \\
          (\G^\top \x^*)_i &< \rho &&\forall i\notin\supp(\y^*).
     \end{align*}
\end{lemma}

Below, we define
$\mathcal{V}^{*}(\mathcal{Z})=\mathcal{V}^{*}(\mathcal{X})\times \mathcal{V}^{*}(\mathcal{Y})$, where 
\[
\mathcal{V}^{*}(\mathcal{X})\triangleq\{\x: \x\in\Delta_M,~\supp(\x)\subseteq \supp(\x^*)\}
\] 
and 
\[
\mathcal{V}^{*}(\mathcal{Y})\triangleq\{\y: \y\in\Delta_N,~ \supp(\y)\subseteq \supp(\y^*)\}.
\]
\begin{definition}\label{def: cx cy}
    \begin{align*}
        c_x \triangleq \min_{{\x\in\Delta_M \backslash \{\x^*\}}}  \max_{\y\in \mathcal{V}^{*}(\mathcal{Y})} \frac{(\x-\x^*)^\top \G \y}{\|\x-\x^*\|_1}, \qquad c_y \triangleq \min_{{\y\in\Delta_N\backslash\{\y^*\}}}  \max_{\x\in \mathcal{V}^{*}(\mathcal{X})}  \frac{\x^\top \G (\y^*- \y)}{\|\y^*- \y\|_1}.
    \end{align*} 
\end{definition}

Note that in the definition of $c_x$ and $c_y$, the outer minimization is over an open set, which may make the definition problematic as the optimal value may not be attained. However, the following lemma shows that $c_x$ and $c_y$ are well-defined. 

\begin{lemma}\label{lem: cx-well-defined}
    $c_x$ and $c_y$ are well-defined, and $0<c_x,c_y\le 1$.
\end{lemma}

\begin{proof}
    We first show $c_x$ and $c_y$ are well-defined. To simplify the notations, we define $x_{\min}^*\triangleq\min_{i\in \supp(\x^*)}x_i^*$ and $\calX'\triangleq\{\x: \x\in \Delta_M,~\|\x-\x^*\|_1\ge x_{\min}^*\}$, and define $y_{\min}^*$ and $\calY'$ similarly. We will show that
\begin{align*}
    c_x = \min_{\x\in \calX'}  \max_{\y\in\mathcal{V}^{*}(\mathcal{Y})} \frac{(\x-\x^*)^\top \G \y}{\|\x-\x^*\|_1},\quad c_y = \min_{\y\in\calY'}  \max_{\x\in\mathcal{V}^{*}(\mathcal{X})}  \frac{\x^{\top} \G (\y^*- \y)}{\|\y^*- \y\|_1}, 
\end{align*}
which are well-defined as the outer minimization is now over a closed set. Consider $c_x$, it suffices to show that for any $\x\in \Delta_M$ such that $\x\ne \x^*$ and $\|\x-\x^*\|_1 < x_{\min}^*$, there exists $\x'\in \Delta_M$ such that $\|\x'-\x^*\|_1 = x_{\min}^*$ and
\begin{align}\label{eq: mapping}
     \frac{(\x-\x^*)^\top \G \y}{\|\x-\x^*\|_1} = \frac{(\x'-\x^*)^\top \G \y}{\|\x'-\x^*\|_1}, \;\forall \y.
\end{align}

In fact, we can simply choose $\x' = \x^*+(\x-\x^*)\cdot \frac{x_{\min}^*}{\|\x-\x^*\|_1}$. We first argue that $\x'$ is still in $\Delta_M$. For each $j\in [K]$, if $x_j-x^*_j\ge 0$, we surely have $x'_j \ge x_j^* +0\ge 0$; otherwise, $x_j^* > x_j\geq 0$ and thus $j\in \supp(\x^*)$ and $x^*_j\ge x_{\min}^*$, which implies $x'_j \geq x^*_j - |x_j-x^*_j|\cdot \frac{x_{\min}^*}{\|\x-\x^*\|_1} \ge x^*_j - x_{\min}^* \ge 0$. In addition, $\sum_{j} x'_j = \sum_j x^*_j =1$. Combining these facts, we have $\x'\in \Delta_M$. 

Moreover, according to the definition of $\x'$, $\|\x'-\x^*\|_1 = x_{\min}^*$ holds. Also, since $\x^*-\x$ and $\x^*-\x'$ are parallel vectors, \pref{eq: mapping} is satisfied. The  arguments above show that the $c_x$ in \pref{def: cx cy} is a well-defined real number. The case of $c_y$ is similar.

Now we show $0<c_x,c_y\le 1$. The fact that $c_x, c_y\leq 1$ is a direct consequence of $\G$ being in $[-1, 1]^{M\times N}$. Below, we %denote $\rho=\x^{*\top}\G\y^*$, and 
use contradiction to prove that $c_y>0$. First, if $c_y < 0$, then there exists $\y\neq \y^*$ such that $ \x^{*\top}\G\y^*< \x^{*\top}\G \y$. This contradicts with the fact that $(\x^*, \y^*)$ is the equilibrium. 
    
%Next we show $c_y\ne 0$ by contradiction, which completes the proof. 
On the other hand, if $c_y=0$, then there is some $\y\neq \y^*$ such that 
\begin{equation}\label{eq:some_condition}
\max_{\x\in \mathcal{V}^{*}(\mathcal{X})}\x^{\top}\G(\y^*-\y) = 0.
\end{equation}
 Consider the point $\y' = \y^* + \frac{\xi}{2}(\y-\y^*)$ (recall the definition of $\xi$ in \pref{def: xi} and that $0< \xi \leq 1$),
    which lies on the line segment between $\y^*$ and $\y$. Then, for any ${\x\in\calX}$,
    \begin{align*}
        \x^{\top}\G\y' 
        &= \sum_{i\notin\supp(\x^*)} x_i(\G\y')_i +  \sum_{i\in\supp(\x^*)} x_i(\G\y')_i \\
        &\geq \sum_{i\notin\supp(\x^*)} \big(x_i(\G\y^*)_i- x_i\|\y'-\y^*\|_1\big)  +\sum_{i\in\supp(\x^*)}\left(\frac{\xi}{2}\cdot x_i(\G(\y-\y^*))_i + x_i(\G\y^*)_i\right) \tag{using $G_{ij}\in [-1,-1]$ for the first part and $\y'=\y^*+\frac{\xi}{2}(\y-\y^*)$ for the second} \\
        &\geq \sum_{i\notin\supp(\x^*)} \big(x_i(\G\y^*)_i- x_i\|\y'-\y^*\|_1\big)  +\sum_{i\in\supp(\x^*)} x_i\rho   \tag{using \pref{eq:some_condition} and $(\G\y^*)_i=\rho$ for all $i\in\supp(\x^*)$} \\
       &\geq \sum_{i\notin\supp(\x^*)} \big(x_i\left((\G\y^*)_i- \xi\right) \big)  +\sum_{i\in\supp(\x^*)} x_i\rho   \tag{using $\y'-\y^*=\frac{\xi}{2}(\y-\y^*)$ and $\|\y-\y^*\|_1 \leq 2$} \\
        &\geq \sum_{i\notin\supp(\x^*)} x_i\rho +  \sum_{i\in\supp(\x^*)} x_i\rho  \tag{by the definition of $\xi$} \\
        &= \rho.
    \end{align*}
    This shows that $\min_{\x\in\calX} \x^{\top}\G\y'\ge \rho$, that is, $\y' \neq \y^*$ is also a maximin point, contradicting that $\z^*$ is unique. Therefore, $c_y>0$ has to hold, and so does $c_x>0$ by the same argument.
\end{proof} 

Finally, we define the following constant that depends on $\G$: 
\begin{definition}\label{def: epsilon definition}
\begin{align*}
     \epsilon\triangleq\min_{j\in\supp({\z^*})}  \exp\left(-\frac{\ln(MN)}{z_j^*}\right). 
\end{align*}
\end{definition}

\subsection{Auxiliary Lemmas}\label{sec: aux lemm}

All lemmas stated in this section is for the case $f(\x,\y)=\x^\top \G \y$ with $\calZ =\Delta_M \times \Delta_N$ and a unique Nash equilibrium $\z^*=(\x^*,\y^*)$. 

\begin{lemma}
    \label{lem:them5prime}
    %Consider $f(\z)=\x^\top \G\y$ with $\calZ=\Delta_M \times \Delta_N$. 
    For any $\z\in \calZ$, we have
    \[
    \max_{\z'\in \mathcal{V}^{*}(\mathcal{Z})}F(\z)^\top(\z-\z')\ge C\|\z^*-\z\|_1
    \] 
    for $C = \min\{c_x, c_y\} \in (0, 1]$.  
\end{lemma}

\begin{proof}
    Recall that $\rho={\x^*}^\top \G\y^*$ is the game value and note that
    \begin{align*}
    \max_{\z'\in \mathcal{V}^{*}(\mathcal{Z})}F(\z)^\top(\z-\z')
    &= \max_{\z'\in\mathcal{V}^*(\calZ)}(\x-\x')^\top \G \y + \x^\top \G (\y' - \y) =   \max_{\z'\in\mathcal{V}^*(\calZ)} -\x^{\prime\top} \G \y + \x^\top \G \y' \\
    &=\max_{\x'\in \mathcal{V}^{*}(\mathcal{X})}\left(\rho -\x'^\top \G\y\right)+\max_{\y'\in \mathcal{V}^{*}(\mathcal{Y})} \left(\x^\top \G\y'-\rho\right)\\
    &=\max_{\x'\in \mathcal{V}^{*}(\mathcal{X})}\x'^\top \G(\y^*-\y)+\max_{\y'\in \mathcal{V}^{*}(\mathcal{Y})}(\x-\x^*)^\top \G\y'    \tag{\pref{lem: gap lemma}}\\
    &\ge c_y\|\y^*-\y\|_1+c_x\|\x^*-\x\|_1\tag{by \pref{def: cx cy}}\\
    &\ge \min\{c_x,c_y\}\|\z^*-\z\|_1,
    \end{align*}
    %By \pref{lem: cx-well-defined}, we have $0<\min\{c_x,c_y\}\le 1$, 
    which completes the proof.
\end{proof}
%Now we are ready to prove \pref{thm: point convergence omwu}.

\begin{lemma}\label{lem:reverse pinsker}
For any $\z\in\calZ$, we have
\[
\KL(\z^*,\z)\le\sum_{i\in\supp(\z^*)}\frac{(z^*_i-z_i)^2}{z_i}+\sum_{i\notin\supp(\z^*)}z_i\le \frac{1}{\min_{i\in \supp(\z^*)}z_i}\|\z^*-\z\|_1.
\]

\end{lemma}
\begin{proof}
Using the definition of the Kullback-Leibler divergence, we have
\begin{align*}
\KL(\x^*,\x)&=\sum_{i}x^*_i\ln\left(\frac{x_i^*}{x_i}\right)\le \ln \left(\sum_i \frac{{x^*_i}^2}{x_i}\right)=\ln\left(1+\sum_i\frac{({x^*_i}-x_i)^2}{x_i}\right)\le \sum_i\frac{({x^*_i}-x_i)^2}{x_i},
\end{align*}
where the first inequality is by the concavity of the $\ln(\cdot)$ function, and the second inequality is because $\ln(1+u)\le u$. Considering $i\in\supp(\x^*)$ and $i\notin\supp(\x^*)$ separately in the last summation, we have
\begin{align*}
\sum_i\frac{({x^*_i}-x_i)^2}{x_i}=\sum_{i\in\supp(\x^*)}\frac{({x^*_i}-x_i)^2}{x_i}+\sum_{i\notin\supp(\x^*)}\frac{(x_i)^2}{x_i}=\sum_{i\in\supp(\x^*)}\frac{({x^*_i}-x_i)^2}{x_i}+\sum_{i\notin\supp(\x^*)}{x_i}.
\end{align*}

The case for $\KL(\y^*,\y)$ is similar. Combining both cases finishes the proof of the first inequality (recall that $\KL(\z^*,\z)$ is defined as $\KL(\x^*,\x) + \KL(\y^*,\y)$). The second inequality is straightforward:
\begin{align*}
\sum_{i\in\supp(\z^*)}\frac{(z^*_i-z_i)^2}{z_i}+\sum_{i\notin\supp(\z^*)}z_i
&\le \frac{1}{\min_{i\in \supp(\z^*)}z_i}\left(\sum_{i\in\supp(\z^*)} |z^*_i-z_i| + \sum_{i\notin\supp(\z^*)}|z_i|\right) \\
&= \frac{1}{\min_{i\in \supp(\z^*)}z_i}\|\z^*-\z\|_1.
\end{align*}
\end{proof}

\begin{lemma}\label{lem:stability}
For $\eta\le\frac{1}{8}$, \OMWU guarantees $\frac{3}{4}\zpp_{t,i}\le z_{t,i}\le \frac{4}{3}\zpp_{t,i}$ and $\frac{3}{4}\zpp_{t,i}\le \zpp_{t+1,i}\le \frac{4}{3}\zpp_{t,i}$. 
\end{lemma}
\begin{proof}
This is shown directly by the update of $\xp_t$:
\[
\frac{\xpp_{t,i}\exp{(-\eta)}}{\exp{(\eta)}}\le \xpp_{t+1,i}=\frac{\xpp_{t,i}\exp{(-\eta\cdot (\G\y_{t})_i)}}{\sum_{j}\xpp_{t,j}\exp{(-\eta\cdot (\G\y_{t})_j)}}\le\frac{\xpp_{t,i}\exp{(\eta)}}{\exp(-\eta)}.
\]
So by the condition on $\eta$, we have $\frac{3}{4}\xpp_{t,i}\le\exp(-2\eta)\cdot \xpp_{t,i}\le\xpp_{t+1,i}\le \exp(2\eta)\cdot\xpp_{t,i}\le \frac{4}{3}\xpp_{t,i}$. The cases for $\x_t$, $\yp_t$ and $\y_t$ are similar.
\end{proof}

\begin{lemma}\label{lem:local norm}
For any two probability vectors $\bu, \bv$, if for every entry $i$, $\frac{1}{2}u_i\le v_i\le \frac{3}{2}u_i$, then $\frac{1}{3}\sum_{i}\frac{(v_i-u_i)^2}{u_i} \leq \KL(\bu,\bv) \leq \sum_{i}\frac{(v_i-u_i)^2}{u_i}\leq \frac{1}{4}$. 
\end{lemma}
\begin{proof}
Using the definition of the Kullback-Leibler divergence, we have
\begin{align*}
\KL(\bu,\bv)&=-\sum_{i}u_i\ln \frac{v_i}{u_i}\ge -\sum_{i}u_i \left(\frac{v_i-u_i}{u_i}-\frac{1}{3}\frac{(v_i-u_i)^2}{u_i^2}\right)=\frac{1}{3}\sum_i\frac{(v_i-u_i)^2}{u_i}, \\
\KL(\bu,\bv)&=-\sum_{i}u_i\ln \frac{v_i}{u_i}\le -\sum_{i}u_i \left(\frac{v_i-u_i}{u_i}-\frac{(v_i-u_i)^2}{u_i^2}\right)=\sum_i\frac{(v_i-u_i)^2}{u_i}\leq \frac{1}{4},
\end{align*}
where the first inequality is because $\ln(1+a)\le a-\frac{1}{3}a^2$ for $-\frac{1}{2}\le a\le \frac{1}{2}$, and the second inequality is because $\ln(1+a)\geq a-a^2$ for $-\frac{1}{2}\leq a\leq \frac{1}{2}$. The third inequality is by using the condition $|u_i-v_i|\leq \frac{1}{2}u_i$. 
\end{proof}

%\begin{lemma}\label{lem: theta 1}
%     $\Theta_1\leq \frac{1}{64} + \ln(MN)$. 
%\end{lemma}
%\begin{proof}
%     $\Theta_1 = \KL(\z^*, \zp_1) + \frac{1}{16}\KL(\zp_1, \z_{0})\leq \sum_{i}z_i^*\ln\frac{z_i^*}{\zpp_{1,i}} + \frac{1}{16}\times \frac{1}{4} \leq \ln(MN) + \frac{1}{64}$, where the first inequality is by \pref{lem:stability} and \pref{lem:local norm}. 
%\end{proof}

\begin{lemma}
     \label{lem: smallest entry}
     For all $i\in\supp(\z^*)$ and $t$, \OMWU guarantees $\zpp_{t,i}\geq \epsilon$ ($\epsilon$ is defined in \pref{def: epsilon definition}). 
\end{lemma}
\begin{proof}
%    Applying \pref{lem: regret bound} with the Bregman divergence being the Kullback-Leibler divergence for OMWU, we have
%\begin{align*}
%2 \eta F\left(z_{t}\right)^{\top}\left(z_{t}-z\right) \leq \KL(z,\widehat{z}_{t})-\KL(z,\widehat{z}_{t+1})-\KL(\widehat{z}_{t+1},z_{t})-\tfrac{15}{16}\KL(z_{t},\widehat{z}_{t})+\tfrac{1}{16}\left\|\widehat{z}_{t}-z_{t-1}\right\|^{2}_1.
%\end{align*}
%Therefore by using Pinsker's inequality,
%\begin{align}
%    \KL(z,\widehat{z}_{t+1})\leq \KL(z,\widehat{z}_{t})-\frac{1}{2}\left(\KL(\widehat{z}_{t+1},z_{t})+\KL(z_{t},\widehat{z}_{t})\right)-\frac{1}{2}\left\|\widehat{z}_{t+1}-z_{t}\right\|^{2}_1+\frac{1}{16}\left\|\widehat{z}_{t}-z_{t-1}\right\|^{2}_1.\label{eq:reg omwu}
%\end{align}
%We simplify the notations with $\Omega_{t} \triangleq$ $\KL(z^*,\widehat{z}_{t})$, and $\theta_{t}=\frac{1}{16}\left\|\widehat{z}_{t}-z_{t-1}\right\|^{2}_1\le\frac{1}{16}$ by \pref{lem:stability}, where $z^*$ is the unique Nash equilibrium. Based on \pref{eq:reg omwu}, we have
Using \pref{eq: simple recursion}, we have 
\begin{align}
    \KL(\z^*, \zp_t)\leq \Theta_t \leq \cdots \leq \Theta_1 = \tfrac{1}{16}\KL(\zp_1, \z_0)+\KL(\z^*, \zp_1)= \KL(\z^*, \zp_1), \label{eq:Theta_decreasing}
\end{align}
where the last equality is because $\zp_1=\z_0=(\frac{\mathbf{1}_M}{M}, \frac{\mathbf{1}_N}{N})$.

%\[
%\Omega_{t}\le\Omega_{t} +8\theta_{t}\le\Omega_{t-1}+\theta_{t-1} \leq \Omega_{t-1}+8\theta_{t-1}\le\cdots\le \Omega_{0}+8\theta_0\le 1+\KL(z^*,\zp_0). 
% \]
Then, for any $i\in\supp(\z^*)$, we have 
\begin{align*}
z_i^*\ln \frac{1}{\zpp_{t,i}}
&\le \sum_{j}z_j^* \ln \frac{1}{\zpp_{t,j}} = \KL(\z^*, \zp_t) - \sum_j z_j^* \ln z_j^*\leq \KL(\z^*,\zp_1) - \sum_{j}z_j^* \ln {z_j^*} \\
&=  \sum_j z_j^* \ln \frac{1}{\zpp_{1,j}} =  \ln(MN). 
\end{align*}
Therefore, we conclude for all $t$ and $i\in\supp(\z^*)$, $\zpp_{t,i}$ satisfies
\[
\zpp_{t,i}\ge \exp\left(-\frac{\ln(MN)}{z_i^*}\right)\ge \min_{j\in\supp({\z^*})}  \exp\left(-\frac{\ln(MN)}{z_j^*}\right)= \epsilon. 
\]
\end{proof}

%\begin{lemma}\label{lem: temp lemma}
%    For some constant $\delta>0$ depending on the game matrix $G$ and $\KL(z^*,\zp_0)$, we have
%     \begin{align*}
%          \|z^*-\zp_{t+1}\|_1^2 \leq \frac{64(1+\KL(z^*,\zp_0))}{\delta t}, \qquad \|z^*-z_t\|_1^2 \leq %\frac{128(1+\KL(z^*,\zp_0))}{\delta t}. 
%     \end{align*}
%\end{lemma}

\subsection{Proofs of \pref{lem: KL sufficient decrease} and \pref{thm: point convergence omwu}}
\label{sec: OMWU main part}
\begin{proof}[Proof of \pref{lem: KL sufficient decrease}]

Below we consider any $\z' \in\calZ$ such that $\supp(\z')\subseteq\supp(\z^*)$, that is, $\z'\in\mathcal{V}^{*}(\mathcal{Z})$. Considering \pref{eq: omda update 5}, and using the first-order optimality condition of $\zp_{t+1}$, we have 
\begin{align*}
(\nabla\psi(\zp_{t+1}) - \nabla\psi(\zp_t) + \eta F(\z_t))^\top (\z'-\zp_{t+1}) \geq 0,
\end{align*}
where $\psi(\z)=\sum_i z_i\ln z_i$. Rearranging the terms and we get 
\begin{align}
 \eta F\left(\z_{t}\right)^{\top}\left(\zp_{t+1}-\z'\right) \le \left(\nabla\psi(\zp_{t+1})-\nabla\psi(\zp_{t})\right)^{\top}\left(\z'-\zp_{t+1}\right)=\sum_i \left(z'_i-\zpp_{t+1,i}\right)\ln \frac{\zpp_{t+1,i}}{\zpp_{t,i}}.   \label{eq: tmp eq}
\end{align}
The left hand side of \pref{eq: tmp eq} is lower bounded as
\begin{align*}
 \eta F\left(\z_{t}\right)^{\top}\left(\zp_{t+1}-\z'\right)
 &=\eta F\left(\zp_{t+1}\right)^{\top}\left(\zp_{t+1}-\z'\right)+\eta\left(F\left(\z_{t}\right)-F\left(\zp_{t+1}\right)\right)^{\top}\left(\zp_{t+1}-\z'\right) \\ 
 &\geq \eta F\left(\zp_{t+1}\right)^{\top}\left(\zp_{t+1}-\z'\right) - \eta\|F\left(\z_{t}\right)-F\left(\zp_{t+1}\right)\|_\infty\|\zp_{t+1}-\z'\|_1 \\ 
 & \geq \eta F\left(\zp_{t+1}\right)^{\top}\left(\zp_{t+1}-\z'\right)-4\eta \left\|\z_{t}-\zp_{t+1}\right\|_1 
 \tag{$\|F\left(\z_{t}\right)-F\left(\zp_{t+1}\right)\|_\infty \leq \left\|\z_{t}-\zp_{t+1}\right\|_1 \leq 4$} \\ 
 & \geq \eta F\left(\zp_{t+1}\right)^{\top}\left(\zp_{t+1}-\z'\right)-\frac{1}{2}\left\|\z_{t}-\zp_{t+1}\right\|_1; \tag{$\eta \leq 1/8$}
\end{align*}
on the other hand, the right hand side of \pref{eq: tmp eq} is upper bounded by 
\begin{align*}
\sum_i \left(z'_i-\zpp_{t+1,i}\right)\ln \frac{\zpp_{t+1,i}}{\zpp_{t,i}}&=\sum_{i\in\supp(\z^*)} z'_i\ln \frac{\zpp_{t+1,i}}{\zpp_{t,i}}-\KL(\zp_{t+1},\zp_{t}) \tag{$\supp(\z')\subseteq \supp(\z^*)$}\\
&\le\sum_{i\in\supp(\z^*)} \left|\ln \frac{\zpp_{t+1,i}}{\zpp_{t,i}}\right|\\
&=\sum_{i\in\supp(\z^*)} \max\left\{\ln\left(1+ \frac{\zpp_{t+1,i}-\zpp_{t,i}}{\zpp_{t,i}}\right), 
\ln\left(1+ \frac{\zpp_{t,i}-\zpp_{t+1,i}}{\zpp_{t+1,i}}\right)
\right\} \\
&\le\sum_{i\in\supp(\z^*)} \ln\left(1+\frac{\left| {\zpp_{t+1,i}}-{\zpp_{t,i}}\right|}{\min\{\zpp_{t+1,i}, \zpp_{t,i}\}}\right)   \\
&\le \frac{4}{3}\sum_{i\in \supp(\z^*)}   \frac{\left| {\zpp_{t+1,i}}-{\zpp_{t,i}}\right|}{\zpp_{t,i}} \tag{$\ln(1+a) \leq a$ and \pref{lem:stability}}.
\end{align*}
Combining the bounds on the two sides of \pref{eq: tmp eq}, we get 
\begin{align*}
    \eta F\left(\zp_{t+1}\right)^{\top}\left(\zp_{t+1}-\z'\right) 
    &\leq \frac{4}{3}\sum_{i\in \supp(\z^*)}   \frac{\left| {\zpp_{t+1,i}}-{\zpp_{t,i}}\right|}{\zpp_{t,i}} + \frac{1}{2}\|\z_t-\zp_{t+1}\|_1.
\end{align*}
Since $\z'$ can be chosen as any point in $\calV^*(\calZ)$, we further lower bound the left-hand side above using \pref{lem:them5prime} and get 
\begin{align}
    \eta C\|\z^*-\zp_{t+1}\|_1 
    &\leq \frac{4}{3}\sum_{i\in\supp(\z^*)} \frac{|\zpp_{t+1,i} - \zpp_{t,i}|}{\zpp_{t,i}} + \frac{1}{2}\|\z_t-\zp_{t+1}\|_1 
    \nonumber \\
    &\leq \frac{4}{3\epsilon} \|\zp_{t+1}-\zp_t\|_1+ \frac{1}{2}\|\z_t-\zp_{t+1}\|_1, 
     \tag{\pref{lem: smallest entry}}\\
&\leq \frac{4}{3\epsilon}\left(\|\zp_{t+1}-\zp_t\|_1+ \|\z_t-\zp_{t+1}\|_1\right) \label{eq: 1 norm upper bound 2} 
\end{align}
where the last inequality uses $\epsilon \leq 1$. 
With the help of \pref{eq: 1 norm upper bound 2}, below we prove the desired inequalities. 

\paragraph{Case 1. General case. }
\begin{align*}
     &\KL(\zp_{t+1},\z_t)+\KL(\z_{t},\zp_t)\\
     &\geq \frac{1}{2}\|\xp_{t+1} - \x_t\|_1^2 + \frac{1}{2}\|\yp_{t+1} - \y_t\|_1^2+ \frac{1}{2}\|\x_t-\xp_{t}\|_1^2+\frac{1}{2}\|\y_t-\yp_{t}\|_1^2  \tag{Pinsker's inequality}\\
     &\geq \frac{1}{4}\|\zp_{t+1} - \z_t\|_1^2 + \frac{1}{4}\|\z_t-\zp_{t}\|_1^2\tag{$a^2+b^2\geq \frac{1}{2}(a+b)^2$}\\
     &\geq \frac{1}{16}\|\zp_{t+1} - \z_t\|_1^2 + \frac{1}{8}\left(\|\zp_{t+1} - \z_t\|_1^2+ \|\z_t-\zp_{t}\|_1^2 \right) \\
     &\geq \frac{1}{16}\|\zp_{t+1}-\z_t\|_1^2+\frac{1}{16}\|\zp_{t+1}-\zp_t\|_1^2  \tag{$a^2+b^2\geq \frac{1}{2}(a+b)^2$ and triangle inequality} \\
     &\geq \frac{1}{32}\left(\|\zp_{t+1}-\z_t\|_1+\|\zp_{t+1}-\zp_t\|_1\right)^2 \tag{$a^2+b^2\geq \frac{1}{2}(a+b)^2$} \\
     &\ \geq \frac{1}{32}\left(\frac{3\epsilon \eta C}{4}\right)^2 \|\z^*-\zp_{t+1}\|_1^2    \tag{\pref{eq: 1 norm upper bound 2}}\\
     &\geq \frac{\epsilon^2 \eta^2 C^2 }{64} \times \epsilon^2 \KL(\z^*, \zp_{t+1})^2 = \frac{\epsilon^4 \eta^2 C^2 }{64} \KL(\z^*, \zp_{t+1})^2    \tag{\pref{lem:reverse pinsker} and \pref{lem: smallest entry}}.
\end{align*}
This proves the first part of the lemma with $C_1 = \epsilon^4 C^2/64$.

\paragraph{Case 2. The case when $\max\{\|\z^*-\zp_t\|_1, \|\z^*-\z_t\|_1\} \leq \frac{\eta\xi}{10}$. }
\begin{align}
         &\KL(\zp_{t+1}, \z_t)+\KL(\z_{t},\zp_t) \nonumber \\
         &\ge \frac{1}{3}\sum_{i} \left(\frac{(\zpp_{t+1,i}-z_{t,i})^2}{\zpp_{t+1,i}}+\frac{(z_{t,i}-\zpp_{t,i})^2}{z_{t,i}}\right)   \tag{\pref{lem:stability} and \pref{lem:local norm}}\\
         &\ge \frac{1}{4}\sum_{i\notin \supp(\z^*)} \left(\frac{(\zpp_{t+1,i}-z_{t,i})^2}{\zpp_{t,i}}+\frac{(z_{t,i}-\zpp_{t,i})^2}{\zpp_{t,i}}\right)  \tag{\pref{lem:stability}} \\
         &\geq \frac{1}{8}\sum_{i\notin\supp(\z^*)} \frac{(\zpp_{t+1,i}-\zpp_{t,i})^2}{\zpp_{t,i}}.  \label{eq: second kind bound}
\end{align}
Below we continue to bound $\sum_{i\notin\supp(\z^*)} \frac{(\zpp_{t+1,i}-\zpp_{t,i})^2}{\zpp_{t,i}}$. 

%Define  $\xi=\min\left\{ \min_{i\notin\supp(x^*)} (Gy^*)_i - \rho, \rho - \max_{i\notin\supp(y^*)} (G^\top x^*)_i \right\}$ to be a game-dependent constant. Recall by Lemma C.3 of \citet{mertikopoulos2018cycles}, we always have $\xi>0$. 
%For $t\ge{12800(M+N)^2(1+\KL(z^*,\zp_0))}/({\delta \eta^2\xi^2})\triangleq T_0$, by \pref{lem: temp lemma}, we have $\|y_t-y^*\|_1\leq \frac{\eta\xi}{10M}$. Thus,
By the assumption, we have $\|\y_t-\y^*\|_1 \leq \frac{\eta\xi}{10}$, which by \pref{lem: gap lemma} and \pref{def: xi} implies 
    \begin{align*}
        &\forall i\in\supp(\x^*), \qquad  (\G\y_t)_i \leq (\G\y^*)_i + \frac{\eta\xi}{10} = \rho + \frac{\eta\xi}{10} \leq \rho + \frac{\xi}{10}, \\
        &\forall i\notin\supp(\x^*), \qquad (\G\y_t)_i \geq (\G\y^*)_i - \frac{\eta\xi}{10} \geq \rho + \xi - \frac{\eta\xi}{10} \geq \rho + \frac{9\xi}{10}. 
    \end{align*}
    We also have $\|\xp_t-\x^*\|_1\leq \frac{\eta\xi}{10}$, so $\sum_{j\notin\supp(\x^*)}\xpp_{t,j} \leq \frac{\eta\xi}{10}$. 
    Then, for $i\notin\supp(\x^*)$, we have 
    \begin{align*}
         \xpp_{t+1,i} 
         &= \frac{\xpp_{t,i}\exp(-\eta (\G\y_t)_i)}{\sum_{j}\xpp_{t,j}\exp(-\eta (\G\y_t)_j)} \\
         &\leq \frac{\xpp_{t,i}\exp(-\eta (\G\y_t)_i)}{\sum_{j\in\supp(\x^*) }\xpp_{t,j}\exp(-\eta (\G\y_t)_j)} \\
         &\leq \frac{\xpp_{t,i}\exp(-\eta (\rho + \frac{9\xi}{10}))}{\sum_{j\in\supp(\x^*) }\xpp_{t,j}\exp(-\eta (\rho + \frac{\xi}{10}))} \\
         %&\leq \frac{\xpp_{t,i}\exp\left(-\frac{9}{10}\eta\xi\right)}{\left(1-\frac{\eta\xi}{10}\right)\exp\left(-\frac{1}{10}\eta \xi\right)}  \\
         &= \frac{\xpp_{t,i}\exp\left(-\frac{8}{10}\eta\xi\right)}{\left(1-\sum_{j\notin\supp(\x^*)}\xpp_{t,j}\right)} \\
         &\leq \frac{\xpp_{t,i}\exp\left(-\frac{8}{10}\eta\xi\right)}{\left(1-\frac{\eta\xi}{10}\right)}
         \leq \xpp_{t,i}\left(1-\frac{1}{2}\eta \xi\right),
    \end{align*}
    where the last inequality is because $\frac{\exp(-0.8u)}{1-0.1u}\leq 1-0.5u$ for $u\in[0,1]$. 
    Rearranging gives
    \begin{align*}
         \frac{|\xpp_{t+1,i}-\xpp_{t,i}|^2}{\xpp_{t,i}}\geq \frac{\eta^2\xi^2}{4}\xpp_{t,i}\geq \frac{\eta^2\xi^2}{8}\xpp_{t+1,i},
    \end{align*}
    where the last step uses \pref{lem:stability}.
    The case for $\yp_t$ is similar, so we have 
    \begin{align*}
    \frac{|\zpp_{t+1,i}-\zpp_{t,i}|^2}{\zpp_{t,i}}\geq \frac{\eta^2\xi^2}{8}\zpp_{t+1,i}.
    \end{align*}
    Combining this with \pref{eq: second kind bound}, we get 
    \begin{equation}
         \KL(\zp_{t+1},\z_t)+\KL(\z_{t},\zp_t) \geq \frac{\eta^2 \xi^2}{64}\sum_{i\notin\supp(\z^*)}\zpp_{t+1,i}. \label{eq:etaxix} 
    \end{equation}
    Now we combine two lower bounds of $\KL(\zp_{t+1}, \z_t)+\KL(\z_{t},\zp_t)$. Using an intermediate step in \textbf{Case 1}, and \pref{eq:etaxix}, we get 
    \begin{align*}
         \KL(\zp_{t+1}, \z_t)+\KL(\z_{t},\zp_t)
         &= \frac{1}{2}\left(\KL(\zp_{t+1}, \z_t)+\KL(\z_{t},\zp_t)\right) + \frac{1}{2}\left(\KL(\zp_{t+1}, \z_t)+\KL(\z_{t},\zp_t)\right) \\
         &\geq \frac{\epsilon^2 \eta^2 C^2 }{128} \|\z^*-\zp_{t+1}\|_1^2 + \frac{\eta^2 \xi^2}{128} \sum_{i\notin\supp(\z^*)}\zpp_{t+1,i} \\
         &= \frac{\epsilon^3 \eta^2C^2\xi^2}{128} \left(\frac{1}{\xi^2\epsilon}\|\zp_{t+1}-\z^*\|_1^2 + \frac{1}{\epsilon^3 C^2}\sum_{i\notin\supp(\z^*)}\zpp_{t+1,i}\right)\\
         &\geq \frac{\epsilon^3 \eta^2C^2\xi^2}{128} \left(\frac{1}{\epsilon}\|\zp_{t+1}-\z^*\|_1^2 + \sum_{i\notin\supp(\z^*)}\zpp_{t+1,i}\right) 
         \tag{$\xi \leq 1$, $C \leq 1$, and $\epsilon \leq 1$} \\
         &\ge \frac{\epsilon^3 \eta^2C^2\xi^2}{128}\KL(\z^{*},\zp_{t+1}).   \tag{\pref{lem:reverse pinsker} and \pref{lem: smallest entry}}
    \end{align*}
This proves the second part of the lemma with $C_2= \epsilon^3 C^2\xi^2/128$.
\end{proof}

Now we are ready to prove \pref{thm: point convergence omwu}.
\begin{proof}[Proof of \pref{thm: point convergence omwu}]
As argued in \pref{sec: omwu convergence}, with $\Theta_t=\KL(\z^*, \zp_t) + \tfrac{1}{16}\KL(\zp_t, \z_{t-1})$ and $\zeta_t=\KL(\zp_{t+1}, \z_{t}) +\KL(\z_t, \zp_t)$, we have (see \pref{eq: simple recursion})
\begin{align*}
    \Theta_{t+1} \leq \Theta_t - \tfrac{15}{16}\zeta_t.  
\end{align*}
We the proceed as, 
\begin{align*}
     \zeta_t 
     &\geq \frac{1}{2}\KL(\zp_{t+1}, \z_{t}) + \frac{1}{2}\zeta_t \\
     &\geq \frac{1}{2}\KL(\zp_{t+1}, \z_{t}) + \frac{\eta^2 C_1}{2}\KL(\z^*, \zp_{t+1})^2 \tag{\pref{lem: KL sufficient decrease}}\\
     &\geq 2\KL(\zp_{t+1}, \z_t)^2+ \frac{\eta^2 C_1}{2}\KL(\z^*, \zp_{t+1})^2  \tag{by \pref{lem:stability} and \pref{lem:local norm}} \\
     &\geq \frac{\eta^2 C_1}{2}\left(\KL(\zp_{t+1}, \z_t)^2+ \KL(\z^*, \zp_{t+1})^2\right) \tag{$C_1 = \epsilon^4 C^2/64 \leq 1/64$ as shown in the proof of \pref{lem: KL sufficient decrease}}\\ 
     &\geq \frac{\eta^2 C_1}{4}\left( \KL(\zp_{t+1}, \z_t) + \KL(\z^*, \zp_{t+1}) \right)^2 \\
     &\geq \frac{\eta^2 C_1}{4}\Theta_{t+1}^2. 
\end{align*}
Therefore, $\Theta_{t+1} \leq \Theta_t - \tfrac{15\eta^2 C_1}{64}\Theta_{t+1}^2 \leq \Theta_t - \tfrac{15\eta^2 C_1}{64 + \ln MN}\Theta_{t+1}^2$. 
Also, recall $\zp_1 = \z_0 = (\frac{\mathbf{1}_M}{M}, \frac{\mathbf{1}_N}{N})$ and thus
$
\Theta_1 = \KL(\z^*, \zp_1) \leq \ln(MN).
$
Therefore, the conditions of \pref{lem: recursion lemma p > 1} are satisfied with $p=1$ and $q=\frac{15\eta^2 C_1}{64+ \ln (MN)}$, %we have $q(p+1)\Theta_1^p\leq \frac{15\eta^2 C_1}{64 +\ln MN}\times 2\times \ln(MN) \leq 1$, and 
and we conclude that
     \begin{align*}
         \Theta_t \leq \frac{C'}{t},
     \end{align*}
     where $C'=\max\left\{\ln (MN),\frac{128 + 2\ln (MN)}{15\eta^2 C_1}\right\} = \frac{128 + 2\ln (MN)}{15\eta^2 C_1}$. 
     
Next we prove the main result. Set $T_0=\frac{12800C'}{\eta^2\xi^2}$. For $t\ge T_0$, we have using Pinsker's inequality,
\begin{align*}
	\|\z^*-\zp_t\|_1^2&\le 2\|\x^*-\xp_t\|_1^2 + 2\|\y^*-\yp_t\|_1^2 \leq  4\KL(\z^*, \zp_t)\le \frac{4C'}{T_0}\le \frac{\eta^2\xi^2}{100},\\
	\|\z^*-\z_t\|_1^2& \le 2\|\z^*-\zp_{t+1}\|_1^2+2\|\zp_{t+1}-\z_t\|_1^2\\ 
	& \le 4\|\x^*-\xp_{t+1}\|_1^2+4\|\xp_{t+1}-\x_t\|_1^2
	+ 4\|\y^*-\yp_{t+1}\|_1^2+4\|\yp_{t+1}-\y_t\|_1^2 \\
	&\le 8\KL(\z^*,\zp_{t+1})+8\KL(\zp_{t+1},\z_t) \\ 
	&\leq 128\Theta_{t+1}\le \frac{128C'}{T_0}\le \frac{\eta^2\xi^2}{100}.
\end{align*}

Therefore, when $t\ge T_0$,  the condition of the second part of \pref{lem: KL sufficient decrease} is satisfied, and we have
\begin{align*}
	\zeta_t 
	&\geq \frac{1}{2}\KL(\zp_{t+1}, \z_{t}) + \frac{1}{2}\zeta_t \\
	&\geq \frac{1}{2}\KL(\zp_{t+1}, \z_{t}) + \frac{\eta^2 C_2}{2}\KL(\z^*, \zp_{t+1}) \tag{by \pref{lem: KL sufficient decrease}} \\
	&\geq \frac{\eta^2 C_2}{2}\Theta_{t+1} \tag{$C_2= \epsilon^3 C^2\xi^2/128 \leq 1/128$ as shown in the proof of \pref{lem: KL sufficient decrease}}. 
\end{align*}

Therefore, when $t\ge T_0$, $\Theta_{t+1}\le \Theta_t-\frac{15\eta^2C_2}{32}\Theta_{t+1}$, which further leads to $$\Theta_t \le \Theta_{T_0}\cdot \left(1+\frac{15\eta^2 C_2}{32}\right)^{T_0-t}\le \Theta_{1}\cdot \left(1+\frac{15\eta^2 C_2}{32}\right)^{T_0-t} \le \ln(MN)\left(1+\frac{15\eta^2 C_2}{32}\right)^{T_0-t}.$$
where the second inequality uses \pref{eq:Theta_decreasing}.
The inequality trivially holds for $t< T_0$ as well, so it holds for all $t$. 

We finish the proof by relating $\KL(\z^*,\z_t)$ and $\Theta_{t+1}$. Note that by \pref{lem:reverse pinsker}, \pref{lem:stability}, and \pref{lem: smallest entry}, we have
\begin{align*}
\KL(\z^*,\z_t)^2&\le \frac{\|\z^*-\z_t\|^2}{\min_{i\in\supp(z^*)}z_{t,i}^2}\le\frac{16\|\z^*-\z_t\|^2}{9\epsilon^2}\le 4\left(\frac{\|\z^*-\zp_{t+1}\|^2+\|\zp_{t+1}-\z_t\|^2}{\epsilon^2}\right).
\end{align*}
We continue to bound the last term as
\begin{align*}
&4\left(\frac{\|\z^*-\zp_{t+1}\|^2+\|\zp_{t+1}-\z_t\|^2}{\epsilon^2}\right) \\
&= 4\left(\frac{\|\x^*-\xp_{t+1}\|^2+\|\y^*-\yp_{t+1}\|^2+\|\xp_{t+1}-\x_t\|^2+\|\yp_{t+1}-\y_t\|^2}{\epsilon^2}\right) \\% \tag{$(a+b)^2\le 2(a^2+b^2)$}\\
&= 4\left(\frac{\|\x^*-\xp_{t+1}\|_1^2+\|\y^*-\yp_{t+1}\|_1^2+\|\xp_{t+1}-\x_t\|_1^2+\|\yp_{t+1}-\y_t\|_1^2}{\epsilon^2}\right) \tag{$\|\x\|_2\le \|\x\|_1$} \\
&\le  \frac{128}{\epsilon^2}\left(\frac{\KL(\z^*,\zp_{t+1})}{16}+\frac{\KL(\zp_{t+1},\z_t)}{16}\right) \tag{Pinsker's inequality}\\
&\le \frac{128}{\epsilon^2}\Theta_{t+1}.
\end{align*}
Combining everything, we get
\[
\KL(\z^*,\z_t) \leq \frac{\sqrt{128}}{\epsilon}\sqrt{\Theta_{t+1}}
\leq  \frac{\sqrt{128\ln(MN)}}{\epsilon}\left(1+\frac{15\eta^2 C_2}{32}\right)^{\frac{T_0-t-1}{2}},
\]
which completes the proof.
\end{proof}

%% file: appendixE.tex
\section{Proofs of \pref{lem: sufficient decrease} and the Sum-of-duality-gap Bound}
\label{app: sum of duality gap}
\begin{proof}[Proof of \pref{lem: sufficient decrease}]
	Below we consider any $\z' \neq \zp_{t+1} \in\calZ$. 
	Considering \pref{eq: omda update 5} with $D_\psi(\bu, \bv)=\frac{1}{2}\|\bu-\bv\|^2$, and using the first-order optimality condition of $\zp_{t+1}$, we have 
	\begin{align*}
	(\zp_{t+1} - \zp_t + \eta F(\z_t))^\top (\z'-\zp_{t+1}) \geq 0,\\
	(\z_{t+1} - \zp_{t+1} + \eta F(\z_t))^\top (\z'-\z_{t+1}) \geq 0. 
	\end{align*}
	Rearranging the terms and we get 
	\begin{align*}
	(\zp_{t+1} - \zp_t)^\top (\z'-\zp_{t+1}) 
	&\geq \eta F(\z_t)^\top (\zp_{t+1}-\z') \\
	& = \eta F(\zp_{t+1})^\top (\zp_{t+1}-\z') + \eta\left(F(\z_t)- F(\zp_{t+1})\right)^\top (\zp_{t+1}-\z')\\
	&\geq \eta F(\zp_{t+1})^\top (\zp_{t+1}-\z') - \eta L \|\z_t-\zp_{t+1}\| \|\zp_{t+1}-\z'\| \\
	&\geq \eta F(\zp_{t+1})^\top (\zp_{t+1}-\z') - \frac{1}{8} \|\z_t-\zp_{t+1}\| \|\zp_{t+1}-\z'\|,
	\end{align*}
	and
	\begin{align*}
	(\z_{t+1} - \zp_{t+1})^\top (\z'-\z_{t+1}) 
	&\geq \eta F(\z_t)^\top (\z_{t+1}-\z') \\
	& = \eta F(\z_{t+1})^\top (\z_{t+1}-\z') + \eta\left(F(\z_t)- F(\z_{t+1})\right)^\top (\z_{t+1}-\z')\\
	&\geq \eta F(\z_{t+1})^\top (\z_{t+1}-\z') - \eta L \|\z_t-\z_{t+1}\| \|\z_{t+1}-\z'\| \\
	&\geq \eta F(\z_{t+1})^\top (\z_{t+1}-\z') - \frac{1}{8} \|\z_t-\z_{t+1}\| \|\z_{t+1}-\z'\|.
	\end{align*}
	Here, for both block, the third step uses \Holder's inequality and the smoothness condition \pref{assum:smoothness}, and the last step uses the condition $\eta \leq 1/(8L)$.
	Upper bounding the left-hand side of the two inequalities by $\|\zp_{t+1} - \zp_t\| \|\zp_{t+1}-\z'\|$ and $\|\z_{t+1} - \zp_{t+1}\| \|\z_{t+1}-\z'\|$ respectively and then rearranging, we get 
	\begin{align*}
	&\|\zp_{t+1}-\z'\|\left( \|\zp_{t+1} - \zp_t\| + \frac{1}{8}\|\z_t-\zp_{t+1}\|  \right) \geq \eta F(\zp_{t+1})^\top (\zp_{t+1}-\z'),\\
	&\|\z_{t+1}-\z'\|\left( \|\z_{t+1} - \zp_{t+1}\| + \frac{1}{8}\|\z_t-\z_{t+1}\|  \right) \geq \eta F(\z_{t+1})^\top (\z_{t+1}-\z'). %\label{eq: some step 1}
	\end{align*}
	Therefore, we have 
	\begin{align*}
		&\left( \|\zp_{t+1} - \zp_t\| + \frac{1}{8}\|\z_t-\zp_{t+1}\| \right)^2 \geq \frac{\eta^2 [F(\zp_{t+1})^\top (\zp_{t+1}-\z')]_+^2}{\|\zp_{t+1}-\z'\|^2},\\
		&\left( \|\z_{t+1} - \zp_{t+1}\| + \frac{1}{8}\|\z_{t}-\z_{t+1}\| \right)^2 \geq \frac{\eta^2 [F(\z_{t+1})^\top (\z_{t+1}-\z')]_+^2}{\|\z_{t+1}-\z'\|^2}.
	\end{align*}
	
	Finally, by the triangle inequality and the fact $(a+b)^2 \leq 2a^2 + 2b^2$, we have
	\begin{align*}
	\left( \|\zp_{t+1} - \zp_t\| + \frac{1}{8}\|\z_t-\zp_{t+1}\| \right)^2
	&\leq 	\left( \|\z_t - \zp_t\| + \frac{9}{8}\|\z_t-\zp_{t+1}\| \right)^2 \\
	&\leq 	\left( \frac{9}{8}\|\z_t - \zp_t\| + \frac{9}{8}\|\z_t-\zp_{t+1}\| \right)^2 \\
	&\leq \frac{81}{32}	\left( \|\z_t - \zp_t\|^2 +  \|\z_t-\zp_{t+1}\|^2 \right),\\
	\left( \|\zp_{t+1} - \z_{t+1}\| + \frac{1}{8}\|\z_t-\z_{t+1}\| \right)^2
	&\leq 	\left( \frac{9}{8}\|\z_{t+1} - \zp_{t+1}\| + \|\z_t-\zp_{t+1}\| \right)^2 \\
	&\leq 	\left( \frac{9}{8}\|\z_{t+1} - \zp_{t+1}\| + \frac{9}{8}\|\z_t-\zp_{t+1}\| \right)^2 \\
	&\leq \frac{81}{32}	\left( \|\z_{t+1} - \zp_{t+1}\|^2 +  \|\z_t-\zp_{t+1}\|^2 \right),
	\end{align*}
	which finishes the proof.
%	
%	\begin{align*}
%	\|\zp_{t+1}-z'\|^2\left( \|\zp_{t+1}-z_t\|^2 + \|z_t-\zp_t\|^2 \right)    \tag{using $a^2+b^2 \geq \frac{1}{2}(a+b)^2$}
%	&\geq \frac{1}{2} \|\zp_{t+1}-z'\|^2 \left(\|\zp_{t+1}-z_t\| + \|z_t-\zp_t\|\right)^2 \\
%	&= \frac{32}{81} \|\zp_{t+1}-z'\|^2 \left( \frac{9}{8}\|\zp_{t+1}-z_t\| + \frac{9}{8}\|z_t-\zp_t\| \right)^2 \\
%	&\geq \frac{32}{81} \|\zp_{t+1}-z'\|^2 \left( \frac{1}{8}\|\zp_{t+1}-z_t\| +\|\zp_{t+1}-\zp_t\| \right)^2. 
%	\end{align*}
%	Combining this with \pref{eq: some step 1} finishes the proof.  
\end{proof}

Next, we use \pref{eq: ogda recurr} and \pref{eq:driving z} to derive a result on the convergence of ``average duality gap'' across time. First, we use the following lemma to relate the right-hand side of \pref{eq:driving z} to the duality gap of $\z_{t}$. 
\begin{lemma}
    \label{lem: relating duality gap and gradient}
    \minor{Let $\calZ$ be closed and bounded}. Then for any $\z\in\calZ$, we have $\alpha_f(\z) \leq \max_{\z'\in\calZ}F(\z)^\top (\z-\z')$. 
\end{lemma}

\begin{proof}
This is a direct consequence of the convexity of $f(\cdot, \y)$ and the concavity of $f(\x, \cdot)$:
    \begin{align*}
        \alpha_f(\z)  &= \max_{(\x', \y')\in \calX\times \calY} \left(f(\x, \y') - f(\x, \y) + f(\x,\y) - f(\x', \y)\right)\\ 
        &\leq \max_{(\x', \y')\in \calX\times \calY} \left(\nabla_{y} f(\x,\y)^\top (\y'-\y) + \nabla_x f(\x,\y)^\top (\x-\x')\right)    
        =\max_{\z'\in\calZ} F(\z)^\top (\z-\z'). 
    \end{align*}
\end{proof}

With \pref{lem: relating duality gap and gradient}, the following theorem can be proven straightforwardly.
\begin{theorem}
     \label{thm: sum of duality gap}
     Let $\calZ$ be closed and bounded. Then \alg with $\eta\leq \frac{1}{8L}$ ensures
$
         \frac{1}{T}\sum_{t=1}^T \alpha_f(\z_t) = O\left(\frac{D}{\eta\sqrt{T}} \right)
$
 for any $T$, where $D\triangleq \sup_{\z, \z'\in \calZ}\|\z-\z'\|$. 
\end{theorem}
\begin{proof}
	We first bound the sum of squared duality gap as (recall  $\zeta_t=\|\zp_{t+1}-\z_t\|^2 + \|\z_t-\zp_t\|^2$):
    \begin{align*}
        \sum_{t=1}^T \alpha_f(\z_t)^2 
        &\leq \sum_{t=1}^T \left({\max_{\z'\in\calZ}}F(\z_t)^\top (\z_t-\z')\right)^2 \tag{\pref{lem: relating duality gap and gradient}}\\
        &\leq \frac{81}{32\eta^2}\sum_{t=1}^T (\zeta_{t-1}+\zeta_t)\|\z_t-\z'\|^2 \tag{\pref{lem: sufficient decrease}}\\
        &\leq \order\left( \frac{D^2}{\eta^2}\sum_{t=2}^T (\Theta_{t-1} - \Theta_{t}+\Theta_t - \Theta_{t+1})\right) \tag{\pref{eq: ogda recurr}} \\
        &= %\order\left( \frac{1}{\eta^2}\left(1+\Theta_1\right)\right) =
        \order\left(\frac{D^2}{\eta^2}\right) \tag{telescoping}. 
    \end{align*}
    Finally, by Cauchy-Schwarz inequality, we get $\frac{1}{T}\sum_{t=1}^T \alpha_f(\z_t)\leq \frac{1}{T}\sqrt{T\sum_{t=1}^T \alpha_f(\z_t)^2}=\order\left(\frac{D}{\eta \sqrt{T}}\right)$. 
\end{proof}

%Note that the result in \pref{thm: sum of duality gap} does not rely on the \rsi condition, but holds for any smooth convex-concave $f$. 
%See \pref{app:duality gap} for the proof. 
This theorem indicates that $\alpha_f(\z_t)$ is converging to zero.
A rate of \modify{$\alpha_f(\z_t) = \order(\frac{D}{\eta\sqrt{t}})$} would be compatible with the theorem, but is not directly implied by it.
%We are only able to prove a concrete convergence rate on $\alpha_f(\zp_t)$ under further conditions; see \pref{sec:pointwise}.
%
In a recent work, \citet{Golowich2020Last} consider the unconstrained setting and show that the extra-gradient  algorithm obtains the rate \modify{$\alpha_f(\z_t) = \order(\frac{D}{\eta\sqrt{t}})$}, under an extra assumption that the Hessian of $f$ is also Lipschitz \modify{(since \citet{Golowich2020Last} study the unconstrained setting, their duality gap $\alpha_f$ is defined only with respect to the best responses that lie within a ball of radius $D$ centered around the equilibrium)}.
Note that the extra-gradient algorithm requires more cooperation between the two players  compared to \alg and is less suitable for a repeated game setting.
%Whether \alg obtains the same rate, even in the same unconstrained setting with extra assumptions, is still left open.

%% file: appendix_metric.tex
%!TEX root=iclr2021_conference.tex

\modify{
\section{The Equivalence between \rsi and Metric Subregularity}\label{app:ms}

In this section, we formally that show our \rsi condition with $\beta = 0$ is equivalent to metric subregularity. Before introducing the main theorem, we introduce several definitions. We let $\calZ^*\subseteq \calZ\subseteq \mathbb{R}^K$ ($\calZ^*$ and $\calZ$ follow the same definitions as in our main text). First, we define the element-to-set distance function $d$: 
\begin{definition}
    The element-to-set distance function $d$: $\mathbb{R}^K \times 2^{\mathbb{R}^K}\rightarrow \mathbb{R}$ is defined as $d(\z,\calS)=\inf_{\z'\in \calS}\|\z-\z'\|$. 
\end{definition}

The definition of metric subregularity involves a set-valued operator $\calT: \calZ \rightarrow 2^{\mathbb{R}^K}$, which maps an element of $\calZ$ to a set in $\mathbb{R}^K$. %In many problems, we want to find a \emph{zero} of a set-valued operator $\calT$, i.e., finding a $\z$ such that $\zero\in \calT(\z)$ (we will see why this is a case in the following theorem). \emph{Metric subregular} is a useful property that enables that such proplem can be fastly solved.  

\begin{definition}
A set-valued operator $\calT$ is called \emph{metric subregular} at $(\bar{\z}, \bv)$ for $\bv\in\calT(\bar{\z})$ if there exists $\kappa>0$ and a neighborhood $\Omega$ of $\bar{\z}$ such that
\begin{equation*}
    d(\bv, \calT(\z))\ge\kappa d(\z,\calT^{-1}(\bv))
\end{equation*}
for all $\z\in\Omega$, where $\x \in \calT^{-1}(\bv) \Leftrightarrow \bv\in \calT(\x)$.  If $\Omega=\calZ$, we call $\calT$ \emph{globally metric subregular}. 

\end{definition}

The following definition of \emph{normal cone} is also required in the analysis:  
\begin{definition}
The normal cone of $\calZ$ at point $\z$ is 
    $\calN(\z)=\{\g \;|\; \g^\top(\z'-\z)\le 0,~\forall \z'\in \calZ\}$ (we omit its dependence on $\calZ$ for simplicity). Equivalently, $\calN(\z)$ is the polar cone of the convex set $\calZ-\z$ (a property that we will use in the proof).
\end{definition}

Now we are ready to show that our \rsi condition with $\beta=0$ is equivalent to metric subregularity of the operator $\calN+F$, defined via: $(\calN+F)(\z) = \{\g+F(\z) \;|\; \g \in \calN(\z) \}$.

 %The metric subregularity condition is used when the goal is to find a zero of a set-valued operator $\calT$ over a space $\calZ$. In other words, $\calT$ is an operator that maps an element $\z\in\calZ$ to a subset $\calT(\z)\subseteq \calZ$, and the goal is to find a $\z\in\calZ$ such that $\zero\in\calT(\z)$. 

% The metric subregularity condition 
%     \begin{align*}
%          \dist(\calT(\z), \zero) \geq C \dist(\z, \calT^{-1}(\zero))
%     \end{align*}
%     where $\calT^{-1}(\zero)\triangleq \left\{\z'\in\calZ: \calT(\z')=\zero\right\}$. 
     
%     Recall that in our case, 
%     \begin{align*}
%          \z'\in\calZ^* \Leftrightarrow 
%     \end{align*} 
     
    \begin{theorem}
Let $z^*\in\calZ^*$. Then the following two statements are equivalent: 
\begin{itemize}
    \item $(\calN+F)$ is globally metric subregular at $(\z^*, \zero)$ with $\kappa>0$; 
    \item For all $z\in\calZ\backslash \calZ^*$, $\max_{\z'\in\calZ} F(\z)^\top\frac{(\z-\z')}{\|\z-\z'\|}\ge\kappa d(\z,\calZ^*)$.
\end{itemize}
\end{theorem}
          \begin{proof}
Let $\calT=\calN+F$. Notice that 
\begin{align*}
    \z\in\calZ^* ~\Leftrightarrow~ F(\-z)^\top(\z'-\z)\geq 0 ~\Leftrightarrow~ -F(\z)\in \calN(\z) ~\Leftrightarrow~ \zero\in (\calN+F)(\z). 
\end{align*}
Therefore, $\zero\in\calT(\z^*)$ indeed holds, and we have $\calT^{-1}(\zero)=\calZ^*$. 
%In the following we will choose $y=0$, $\calT=\calN+F$
%, where
%
%is the normal cone of $z$. Therefore, \pref{eq:ms} can be written as
This means that the first statement in the theorem is equivalent to
\begin{align*}
              d(\zero, \calT(\z))&\ge \kappa d(\z,\calT^{-1}(\zero))
              ~\Leftrightarrow~ d(\zero, \calN(\z)+F(\z))\ge \kappa d(\z,\calZ^*). 
          \end{align*}
          This inequality holds trivially when $\z\in \calZ^*$.
          Thus, to complete the proof, it suffices to prove that  $d(\zero,\calN(\z)+F(\z))=\max_{\z'\in\calZ} F(\z)^\top\frac{(\z-\z')}{\|\z-\z'\|}$ for $\z\in \calZ\backslash\calZ^*$.
          To do so, note that
          \begin{align*}
              &d(\zero, \calN(\z)+F(\z))\\
              &=d(-F(\z),\calN(\z))\\
              &=\|-F(\z)-\Pi_{\calN(\z)}(-F(\z))\|\\
              &=\|\Pi_{\calN^\circ(\z)}(-F(\z))\| 
          \end{align*}
          where $\calN^\circ(\z)=\{\g \;|\; \g^\top \n\le 0,~\forall \n \in \calN(z)\}$ is the polar cone of $\calN(\z)$ and the last step is by Moreau's theorem.
		  Now consider the projection of $-F(\z)$ onto the polar cone $\calN^{\circ}(\z)$:
			\begin{align*}
				\Pi_{\calN^\circ(\z)}(-F(\z)) &= \argmin_{\y\in \calN^\circ(\z)}\|-F(\z)-\y\|^2 \\
												  &= \argmin_{\y\in \calN^\circ(\z)}\left\{2F(\z)^\top \y+\|\y\|^2\right\} \\
												  &= \argmin_{\y\in \calN^\circ(\z)}\left\{2F(\z)^\top \frac{\y}{\|\y\|}\cdot \|\y\|+\|\y\|^2\right\} \\
												  &= \argmin_{\lambda\geq 0, ~\bar{\z}\in\calN^{\circ}(\z),~ \|\bar{\z}\|=1}\left\{2\lambda F(\z)^\top \bar{\z}+\lambda^2\right\}, 
												  %&= \argmin_{\y=\lambda\bar{\z}:~ \lambda\geq 0, ~\bar{\z}\in\calN^{\circ}(\z),~ \|\bar{\z}\|=1} \left\{\left(\lambda + F(\z)^\top \bar{\z}\right) - \left(F(\z)^\top \bar{\z}\right)^2\right\}, 
			\end{align*}
		    where the last equality is because $\calN^{\circ}(\z)$ is a cone. Next, we find the $\bar{\z}^*$ and $\lambda^*$ that realize the last $\argmin$ operator: notice that the objective is increasing in $F(\z)^\top \bar{\z}$, so $\bar{\z}^*=\argmin_{\bar{\z}\in\calN^\circ(\z):~\|\bar{\z}\|=1} \left\{F(\z)^\top \bar{\z}\right\}$, and thus $\lambda^*=-F(\z)^\top\bar{\z}^*$ when $F(\z)^\top\bar{\z}^*\leq 0$ and $\lambda^*=0$ otherwise.
			Therefore, 
			\begin{align*}
			    \|\Pi_{\calN^\circ(\z)}(-F(\z))\| = \lambda^*  = \max\left\{0, \max_{\bar{\z}\in\calN^\circ(\z),\|\bar{\z}\|=1} -F(\z)^\top\bar{\z} \right\}. 
			\end{align*}
			
			Note that $\calN(z)$ is the polar cone of the conic hull of $\calZ-\z$. Therefore, $\calN^{\circ}(z)=(\textbf{ConicHull}(\calZ-\z))^{\circ\circ} = \textbf{ConicHull}(\calZ-\z)$ and
			\begin{align*}
			    \max\left\{0, \max_{\bar{\z}\in\calN^\circ(\z),\|\bar{\z}\|=1} -F(\z)^\top\bar{\z} \right\} &= \max\left\{0,\max_{\z'\in \cal Z} F(\z)^\top \frac{(\z-\z')}{\|\z'-\z\|}\right\}.
			\end{align*}
			
			Finally, note that when $\z\in \calZ\backslash\calZ^*$, we have $\max_{\z'\in \calZ} F(\z)^\top(\z-\z') > 0$. Combining all the facts above, we have shown $d( \zero, \calN(\z)+F(\z)) = \max_{\z'\in \calZ}F(\z)^\top\frac{(\z-\z')}{\|\z-\z'\|}$.
          \end{proof}
}

%% file: appendix_bilinear.tex
%!TEX root=iclr2021_conference.tex

\section{Proof of \pref{thm: bilinear-polytope}}
\label{app:bilinear polytope}

\begin{proof}[Proof of \pref{thm: bilinear-polytope}]
     Let $\rho = \min_{\x\in\calX}\max_{\y\in\calY}\x^\top \G\y = \max_{\y\in\calY}\min_{\x\in\calX}\x^\top \G\y$ be the game value. In this proof, we prove that there exists some $c>0$ such that 
     \begin{align}
          \max_{\y'\in\calY} \x^\top \G\y' - \rho \geq c\|\x-\Pi_{\calX^*}(\x)\|    \label{eq: polytope condition}
     \end{align}
     for all $\x\in\calX$. Similarly we prove  \begin{align*}
          \max_{\x'\in\calX} \rho - \x^{\prime\top} \G\y \geq c\|\y-\Pi_{\calY^*}(\y)\|
     \end{align*}
     for all $\y\in\calY$. 
     Assume that the diameter of the polytope is $D<\infty$. 
    \modify{Then combining the two proves 
	\begin{align*}
	\max_{\z'}\frac{F(\z)^\top (\z-\z')}{\|\z-\z'\|} 
	&\geq \frac{1}{D}\max_{\z'}F(\z)^\top (\z-\z') 
	= \frac{1}{D}\left(\max_{\y'}\x^\top \G\y' - \min_{\x'} \x^{\prime \top }\G\y\right) \\
	&\geq \frac{c}{D}\left(\|\y-\Pi_{\calY^*}(\y)\| + \|\x-\Pi_{\calX^*}(\x)\|\right) 
	\geq \frac{c}{D}\|\z-\Pi_{\calZ^*}(\z)\|,
	\end{align*}
	 meaning that \typeone holds with $\beta=0$. 
     }
     We break the proof into following several claims. 
     \paragraph{Claim 1. } If $\calX, \calY$ are polytopes, then $\calX^*$ and $\calY^*$ are also polytopes. 
	\begin{proof}[Proof of Claim 1]
		Note that $\calX^* = \left\{\x\in \calX: \max_{\y\in\calY} \x^\top \G\y\leq \rho \right\}$. Since $\calY$ is a polytope, the maximum is attained at vertices of $\calY$. Therefore, $\calX^*$ can be equivalently written as $\left\{\x\in \calX: \max_{\y\in \calV(\calY)} \x^\top \G\y\leq \rho \right\}$, where $\calV(\calY)$ is the set of vertices of $\calY$. Since the constraints of $\calX^*$ are all linear constraints, $\calX^*$ is a polytope. 
	\end{proof} 
	
	With Claim 1, we without loss of generality write $\calX^*$ as 
	\begin{align*}
	     \calX^* = \left\{\x\in\mathbb{R}^M: \quad \ai^\top \x \leq b_i, \text{\ \ for\ } i=1,\ldots, L, \quad \quad \ci^\top \x \leq d_i, \text{\ \ for\ } i=1,\ldots, K \right\}, 
	\end{align*}
	where the $\ai^\top \x \leq b_i$ constraints come from $\x\in\calX$ and the $\ci^\top \x \leq d_i$ constraints come from $\max_{\y\in \calV(\calY)} \x^\top \G\y\leq \rho$.  
    Below, we refer to $\ai^\top \x \leq b_i$ as the \emph{feasibility constraints}, and $\ci^\top \x \leq d_i$ as the \emph{optimality constraints}.
    In fact, one can identify the $i$-th optimality constraint as $\ci = \G\y^{(i)}$ and $d_i=\rho$, where $\y^{(i)}$ is the $i$-th vertex of $\calY$.
    This is based on our construction of $\calX^*$ in the proof of Claim 1.
    Therefore, $K=|\calV(\calY)|$.  
	
Since \pref{eq: polytope condition} clearly holds for $\x\in\calX^*$, below, we focus on an $\x\in\calX\backslash \calX^*$, and let $\x^*\triangleq \Pi_{\calX^*}(\x)$. 

We say a constraint is \emph{tight} at $\x^*$ if $\ai^\top \x^*=b_i$ or $\ci^\top \x^*=d_i$. Below we assume that there are $\ell$ tight feasibility constraints at and $k$ tight optimality constraints at $\x^*$. Without loss of generality, we assume these tight constraints correspond to $i=1,\ldots, \ell$ and $i=1,\ldots, k$ respectively. That is, 
\begin{align*}
     \ai^\top \x^* &= b_i,  \qquad \text{for\ } i=1,\ldots, \ell, \\
     \ci^\top \x^* &= d_i, \qquad \text{for\ } i=1,\ldots, k. 
\end{align*}

\paragraph{Claim 2.} $\x$ violates at least one of the tight optimality constraint at $\x^*$. 
\begin{proof}[Proof of Claim 2]
     We prove this by contradiction. Suppose that $\x$ satisfies all $k$ tight optimality constraints at $\x^*$. Then $\x$ must violates some of the remaining $K-k$ optimality constraints (otherwise $\x\in\calX^*$). Assume that it violates constraints $K-n+1, \ldots, K$ for some $1\leq n\leq K-k$. Thus, we have the following: 
     \begin{align*}
         \ci^\top \x\leq d_i \quad &\text{for\ } i =1,\ldots K-n;\\
         \ci^\top \x> d_i \quad &\text{for\ }i=K-n+1, \ldots, K.  
     \end{align*}
     Recall that $\ci^\top \x^* \leq d_i$ for $i=1, \ldots, K-n$ and $\ci^\top \x^* < d_i$ for all $i=K-n+1,\ldots, K$. Thus, there exists some $\x'$ that lies strictly between $\x$ and $\x^*$ that makes all constraints hold (notice that $\x$ and $\x^*$ both satisfy all feasibility constraints), which contradicts with $\Pi_{\calX^*}(\x)=\x^*$. 
\end{proof}

\paragraph{Claim 3.} $\max_{\y'\in\calY} \left( \x^\top \G\y'-\rho\right) \geq \max_{i\in\{1,\ldots, k\}} \ci^\top (\x-\x^*)$. 
\begin{proof}[Proof of Claim 3]
     Recall that we identify $\ci$ with $\G\y^{(i)}$ and $d_i=\rho$. Therefore, 
     \begin{align*}
         \max_{\y'\in\calY} \left(\x^\top \G\y'-\rho\right) = \max_{i\in\{1,\ldots, |\calV(\calY)|\}} \left(\ci^\top \x - d_i\right) \geq \max_{i\in\{1,\ldots, k\}} \left(\ci^\top \x - d_i\right) = \max_{i\in\{1,\ldots, k\}} \ci^\top (\x-\x^*), 
     \end{align*}
     where the last equality is because $\ci^\top \x^*=d_i$ for $i=1,\ldots, k$. 
\end{proof}

Recall from linear programming literature \cite{davis2016a, davis2016b} that the \emph{normal cone} of $\calX^*$ at $\x^*$ is expressed as follows: 
	\begin{align*}
	\mathcal{N}_{\x^*}=\Big\{\x'-\x^*: \ \  \x'\in\mathbb{R}^M, \quad \Pi_{\calX^*}(\x')=\x^* \Big\} = \left\{ \sum_{i=1}^\ell p_i\ai + \sum_{i=1}^{k} q_i\ci: \quad p_i\geq 0, \quad q_i\geq 0 \right\}. 
	\end{align*}
	The normal cone of $\calX^*$ at $\x^*$ consists of all outgoing normal vectors of $\calX^*$ originated from $\x^*$. Clearly, $\x-\x^*$ belongs to $\calN_{\x^*}$. However, besides the fact that $\x-\x^*$ is a normal vector of $\calX^*$, we also have the additional constraints that $\x\in\calX$. We claim that in our case, $\x-\x^*$ lies in the following smaller cone (which is a subset of $\calN_{\x^*}$): 
	
	\paragraph{Claim 4. } $\x-\x^*$ belongs to
	\begin{align*}
	\mathcal{M}_{\x^*}= \Bigg\{ \sum_{i=1}^\ell p_i\ai + \sum_{i=1}^{k} q_i\ci: \; p_i\geq 0, \; q_i\geq 0,  \; \aj^\top \left(\sum_{i=1}^\ell p_i\ai + \sum_{i=1}^{k} q_i\ci\right)\leq 0, \ \ \forall j=1,\ldots, \ell \Bigg\}. 
	\end{align*}
		\begin{proof}[Proof of Claim 4. ]
		As argued above, $\x-\x^*\in\mathcal{N}_{\x^*}$, and thus $\x-\x^*$ can be expressed as $\sum_{i=1}^\ell p_i\ai + \sum_{i=1}^{k} q_i\ci$ with $p_i\geq 0, q_i\geq 0$. To prove that $\x-\x^*\in\mathcal{M}_{\x^*}$, we only need to prove that it satisfies the additional constraints, that is,
		\begin{align*}
		\ai^\top (\x-\x^*)\leq 0, \ \ \forall i=1,\ldots, \ell.  
		\end{align*}
		This is shown by noticing that for all $i=1,\ldots, \ell$, 
		\begin{align*}
		 \ai^\top (\x-\x^*) 
		 &= \left(\ai^\top \x^* - b_i\right) +  \ai^\top (\x-\x^*) \tag{the $i$-th constraint is tight at $\x^*$}\\
		 &= \ai^\top \left( \x^* + \x-\x^* \right) - b_i \\
		 &= \ai^\top \x - b_i \leq 0. \tag{$\x\in\mathcal{X}$}
		\end{align*}
	\end{proof}
	
\paragraph{Claim 5.} $\x-\x^*$ can be written as $\sum_{i=1}^{\ell} p_i\ai + \sum_{i=1}^k q_i \ci$ with  $0\leq p_i, q_i\leq C'\|\x-\x^*\|$ for all $i$ and some problem-dependent constant $C'<\infty$. 

\begin{proof}[Proof of Claim 5] 
    Notice that $\frac{\x-\x^*}{\|\x-\x^*\|}\in\calM_{\x^*}$ (because $\mathbf{0}\neq \x-\x^*\in\calM_{\x^*}$ and $\calM_{\x^*}$ is a cone). Furthermore, $\frac{\x-\x^*}{\|\x-\x^*\|}\in \{\bv\in\mathbb{R}^M: \|\bv\|_\infty\leq 1\}$. Therefore, $\frac{\x-\x^*}{\|\x-\x^*\|}\in \calM_{\x^*} \cap \{\bv\in\mathbb{R}^M: \|\bv\|_\infty\leq 1\}$, which is a bounded subset of the cone $\calM_{\x^*}$. 
    
    Below we argue that there exists a large enough $C'>0$ such that
    \begin{align*}
      & \Bigg\{ \sum_{i=1}^\ell p_i\ai + \sum_{i=1}^{k} q_i\ci: \; 0\leq p_i, q_i\leq C', \;\forall i \Bigg\} \quad  \supseteq  \quad \calM_{\x^*} \cap \{\bv\in\mathbb{R}^M: \|\bv\|_\infty\leq 1\} \quad \triangleq \quad \calP. 
    \end{align*}
    
    To see this, first note that $\calP$ is a polytope. For every vertex $\hatv$ of $\calP$, the smallest $C'$ such that $\hatv$ belongs to the left-hand side is the solution of the following linear programming: 
    \begin{align*}
         &\min_{p_i, q_i, C_{\hatv}'} C_{\hatv}' \quad 
         \textit{s.t.\ \ } \hatv = \sum_{i=1}^\ell p_i\ai +  \sum_{i=1}^k q_i \ci,  \quad 0\leq p_i, q_i\leq C_{\hatv}'.  
    \end{align*}
    Since $\widehat{v}\in\calM_{\x^*}$, this linear programming is always feasible and admits a finite solution $C_{\hatv}'<\infty$. Now let $C'=\max_{\hatv\in \calV(\calP)} C_{\hatv}'$, where $\calV(\calP)$ is the set of all vertices of $\calP$. Then since any $v\in\calP$ can be expressed as a convex combination of points in $\calV(\calP)$, $\bv$ can be also be expressed as $\sum_{i=1}^\ell  p_i\ai + \sum_{i=1}^k q_i\ci$ with $0\leq p_i, q_i\leq C'$.

    To sum up, $\frac{\x-\x^*}{\|\x-\x^*\|}$ can be represented as $\sum_{i=1}^\ell  p_i\ai + \sum_{i=1}^k q_i\ci$ with $0\leq p_i, q_i\leq C'$. This further implies that $\x-\x^*$ can be represented as $\sum_{i=1}^\ell  p_i\ai + \sum_{i=1}^k q_i\ci$ with $0\leq p_i, q_i\leq C'\|\x-\x^*\|$. 
    Notice that $C'$ only depends on the set of tight constraints at $\x^*$. 
\end{proof}

Finally, we are ready to combine all previous claims and prove the desired inequality. 

     Define $A_i\triangleq \ai^\top (\x-\x^*)$ and $C_i\triangleq \ci^\top (\x-\x^*)$. 
     By Claim 5, we can write $\x-\x^*$ as $\sum_{i=1}^\ell  p_i\ai + \sum_{i=1}^k q_i\ci$ with $0\leq p_i, q_i\leq C'\|\x-\x^*\|$, and thus, 
     \begin{align*}
          \sum_{i=1}^{\ell} p_iA_i + \sum_{i=1}^k q_i C_i =  \left(\sum_{i=1}^\ell  p_i\ai  + \sum_{i=1}^k q_i\ci\right)^\top  \left(\x-\x^*\right) = \|\x-\x^*\|^2.  
     \end{align*}
On the other hand, since $\x-\x^*\in \calM_{\x^*}$ by Claim 4, we have
\begin{align*}
     \sum_{i=1}^\ell p_iA_i = \sum_{i=1}^\ell p_i \ai^\top (\x-\x^*) \leq 0
\end{align*}
and 
\begin{align*}
     \sum_{i=1}^k q_i C_i \leq \left(\max_{i\in\{1,\ldots, k\}}C_i\right) \sum_{i=1}^k q_i \leq \left(\max_{i\in\{1,\ldots, k\}}C_i\right) kC'\|\x-\x^*\|, 
\end{align*}
where in the first inequality we use the fact $p_i\geq 0$, and in the second inequality we use the fact $\max_{i\in\{1,\ldots, k\}}C_i>0$ (by Claim 2) and $0 \leq q_i\leq C'\|\x-\x^*\|$. 

Combining the three inequalities above, we get 
\begin{align*}
     \max_{i\in\{1,\ldots, k\}}C_i \geq \frac{1}{kC'}\|\x-\x^*\|. 
\end{align*}
Then by Claim 3, 
     \begin{align*}
          \max_{\y'\in\calY} \left(\x^\top \G\y' - \rho\right) \geq \max_{i\in\{1,\ldots, k\}} C_i \geq \frac{1}{kC'}\|\x-\x^*\|. 
     \end{align*}
Note that $k$ and $C'$ only depend on the set of tight constraints at the projection point $\x^*$, and there are only finitely many different sets of tight constraints. 
Therefore, we conclude that there exists a constant $c > 0$ such that $\max_{\y'\in\calY} \left(\x^\top \G\y' - \rho\right) \geq c\|\x-\x^*\|$ holds for all $\x$ and $\x^*$, which completes the proof.
\end{proof}

%% file: appendixG.tex
%!TEX root=iclr2021_conference.tex
\section{Proof of \pref{thm: strongly convex} and \pref{thm: beta ge 0}}
\label{app:exmaples condition 2}
\begin{proof}[Proof of \pref{thm: strongly convex}]
	Suppose that $f$ is $\gamma$-strongly-convex in $\x$ and $\gamma$-strongly-concave in $\y$, and let $(\x^*, \y^*)\in\calZ^*$. Then for any $(\x, \y)$ we have 
	\begin{align*}
	&f(\x, \y) - f(\x^*, \y) \leq \nabla_x f(\x, \y)^\top (\x-\x^*) - \frac{\gamma}{2}\|\x-\x^*\|^2, \\
	&f(\x, \y^*) - f(\x, \y) \leq \nabla_y f(\x, \y)^\top (\y^*-\y) - \frac{\gamma}{2}\|\y-\y^*\|^2.
	\end{align*}
	Summing up the two inequalities, and noticing that $f(\x, \y^*)-f(\x^*, \y)\geq 0$ for any $(\x^*, \y^*)\in \calZ^*$, we get 
	\begin{align*}
	F(\z)^\top (\z-\z^*) \geq \frac{\gamma}{2}\|\z-\z^*\|^2,
	\end{align*}
	\modify{
	and therefore, for $\z\notin \calZ^*$, 
	\begin{align*}
	     \frac{F(\z)^\top (\z-\z^*)}{\|\z-\z^*\|} \geq \frac{\gamma}{2}\|\z-\z^*\|,
	\end{align*}
which implies \typetwo with $\beta = 0$ and $C = \gamma/2$.	}
\end{proof}

\begin{proof}[Proof of \pref{thm: beta ge 0}]
First, we show that $f$ has a unique Nash Equilibrium $\z^*=\left(\x^*, \y^*\right)=\left((0,1),(0,1)\right)$. As $f$ is a strictly monotone decreasing function with respect to $y_{1}$, we must have $y^*_{1}=0$ and $y^*_{2}=1$. In addition, if $\x=(0,1)$, $\max_{\y\in \calY}f(\x,\y)=-\min_{\y\in  \calY}y_{1}^{2n}=0$. If $\x\ne (0,1)$, then by choosing $\y^*=(0,1)$, $f(\x,\y^*)=x_{1}^{2n}>0$. Therefore, we have $\x^*=(0,1)$, which proves that the unique Nash Equilibrium is $\x^*=(0,1),\y^*=(0,1)$.

	Second, we show that $f$ satisfies \typetwo with $\beta=2n-2$. In fact, for any $\z=(\x,\y)\neq \z^*$, we have
	\begin{align*}
		F(\z)^\top(\z-\z^*) &= \begin{bmatrix}
		 2nx_{1}^{2n-1}-y_{1}\\ 
		 0\\
		 2ny_{1}^{2n-1}+x_{1}\\
		 0
		\end{bmatrix}^\top \begin{bmatrix}
		x_{1}\\ 
		x_{2}-1\\
		y_{1}\\
		y_{2}-1
		\end{bmatrix} \\
		& = 2n\left(x_{1}^{2n}+y_{1}^{2n}\right) \\
		& \ge 4n\cdot\left( \frac{x_{1}^2+y_{1}^2}{2}\right)^n \tag{Jensen's inequality} \\
		&=\frac{n}{2^{n-2}}\left(x_{1}^2+y_{1}^2\right)^n. \\
	\end{align*}
	Note that $\|\z-\z^*\|=\sqrt{x_{1}^2+(1-x_{2})^2+y_{1}^2+(1-y_{2})^2}=\sqrt{2x_{1}^2+2y_{1}^2}$. 
	\modify{
	Therefore, we have $\frac{F(\z)^\top(\z-\z^*)}{\|\z-\z^*\|}\ge\frac{n}{2^{2n-2}}\|\z-\z^*\|^{2n-1}$. %The equality holds when $x_{1}=y_{1}$.
	 This shows that $f$ satisfies \typetwo with $\beta = 2n-2$ and $C=\frac{n}{2^{2n-2}}$.
}
\end{proof}

%% file: appendixH.tex
\modify{
\section{Proof of \pref{thm: point convergence}}
\label{app:point convergence}
\begin{proof}[Proof of \pref{thm: point convergence}]
	As argued in \pref{sec:convergence}, with $\Theta_t = \|\zp_t - \Pi_{\calZ^*}(\zp_t)\|^2 + \tfrac{1}{16}\|\zp_t-\z_{t-1}\|^2$, $\zeta_t=\|\zp_{t+1}-\z_t\|^2 + \|\z_t-\zp_t\|^2$, we have (see \pref{eq: ogda recurr})
	\begin{align}
	     \Theta_{t+1} \leq \Theta_t - \tfrac{15}{16}\zeta_{t}.    \label{eq: tmp condition to holds}
	\end{align}
	Below, we relate $\zeta_t$ to $\Theta_{t+1}$ using the \rsi condition,
	and then apply \pref{lem: recursion lemma p > 1} to show
	\begin{align}
	\Theta_{t} \leq 
	\begin{cases}
	2\dist(\zp_1, \calZ^*)(1+C_5)^{-t} &\text{if $\beta = 0$,} \\
	\left[\left(1+4\left(\frac{4}{\beta}\right)^{\frac{1}{\beta}}\right)\dist(\zp_1, \calZ^*) + 2\left(\frac{2}{C_5\beta}\right)^{\frac{1}{\beta}}\right]t^{-\frac{1}{\beta}} &\text{if $\beta > 0$,}
	\end{cases} \label{eq:original_goal}
	\end{align}
	 where $C_5= \min\left\{\frac{16 \eta^2 C^2}{81}, \frac{1}{2}\right\}$ as defined in the statement of the theorem.
	This is enough to prove the theorem since 
		\begin{align*}
		\dist(\z_t, \calZ^*)  &\le \|\z_t-\Pi_{\calZ^*}(\zp_{t+1})\|^2 \\
		&\le 2\|\zp_{t+1}-\Pi_{\calZ^*}(\zp_{t+1})\|^2+2\|\zp_{t+1}-\z_t\|^2 \\
		&\le 32\Theta_{t+1} \leq 32\Theta_t.
	\end{align*}
	Next, we prove \pref{eq:original_goal}. We first show a simple fact by \pref{eq: tmp condition to holds}: 
	\begin{align}
	\|\zp_{t+1}-\z_t\|^2 \leq \zeta_t \leq \frac{16}{15}\Theta_t \leq \cdots \leq \frac{16}{15} \Theta_1.  \label{eq: simple simple}
	\end{align}
%	In the case that $\Omega_{t+1} > 0$, meaning that $\zp_{t+1}\notin \calZ^*$, and there is at least one $z'\in\calZ$ such that $F(z)^\top (z-z') > 0$. Thus we can use \pref{lem: sufficient decrease}, and get 
%	\begin{align}
%	\|\zp_{t+1}-z_t\|^2 + \|z_t-\zp_t\|^2
%	&\geq \frac{32\eta^2}{81} \max_{z'\in\calZ}  \frac{\left[F(\zp_{t+1})^\top (\zp_{t+1}-z')\right]_+^2}{\|\zp_{t+1}-z'\|^2}  \label{eq: tmptmpstep} 
%	\end{align}
	%\paragraph{Case 1. \typeone. } 
	Notice that
	\begin{align*}
	\zeta_t &\geq \frac{1}{2}\|\zp_{t+1}-\z_t\|^2 + \frac{1}{2} \left( \|\zp_{t+1}-\z_t\|^2 + \|\z_t-\zp_t\|^2 \right)\\
	&\geq \frac{1}{2}\|\zp_{t+1}-\z_t\|^2 + \frac{16\eta ^2}{81} \sup_{\z'\in\calZ}\frac{\left[F(\zp_{t+1})^\top (\zp_{t+1}-\z')\right]_+^2}{\|\zp_{t+1}-\z'\|^2}  \tag{ \pref{lem: sufficient decrease}} \\
	&\geq   \frac{1}{2}\|\zp_{t+1}-\z_t\|^{2} + \frac{16\eta ^2C^2}{81} \|\zp_{t+1}-\Pi_{\calZ^*}(\zp_{t+1})\|^{2(\beta+1)}    \tag{\typeone condition} \\
	&\geq  \min\left\{\frac{16 \eta^2 C^2}{81}, \frac{1}{2}\left(\frac{15}{16\Theta_1}\right)^\beta\right\}\left(\|\zp_{t+1}-\z_t\|^{2(\beta+1)} + \|\zp_{t+1}-\Pi_{\calZ^*}(\zp_{t+1})\|^{2(\beta+1)}\right)   
	\tag{by \pref{eq: simple simple}} \\
	&\geq  \min\left\{\frac{16 \eta^2 C^2}{2^\beta\cdot 81}, \frac{1}{2}\left(\frac{15}{32\Theta_1}\right)^\beta\right\} \left(\|\zp_{t+1}-\z_t\|^{2} + \|\zp_{t+1}-\Pi_{\calZ^*}(\zp_{t+1})\|^{2}\right)^{\beta+1}   \tag{by H\"{o}lder's inequality: $ 
          (a^{\beta+1} + b^{\beta+1})(1+1)^{\beta} \geq (a+b)^{\beta+1} 
     $}\\
	&\geq \min\left\{\frac{C_5}{2^\beta}, \frac{1}{2}\left(\frac{1}{4\Theta_1}\right)^\beta\right\} \Theta_{t+1}^{\beta+1}   \tag{recall that $C_5 = \min\{\frac{16\eta^2 C^2}{81}, \frac{1}{2}\}$} \\
	&= C'\Theta_{t+1}^{\beta+1}.   \tag{define $C'= \min\left\{\frac{C_5}{2^\beta}, \frac{1}{2}\left(\frac{1}{4\Theta_1}\right)^\beta\right\}$}
	\end{align*}
	Combining this with \pref{eq: tmp condition to holds}, we get 
	\begin{align}
	\Theta_{t+1}  
	&\leq \Theta_t -  C'\Theta_{t+1}^{\beta+1}
	\label{eq: key recursion}
	\end{align}
	When $\beta=0$, \pref{eq: key recursion} implies $\Theta_{t+1}\leq (1+C_5)^{-1}\Theta_t$, which immediately implies $\Theta_t \leq (1+C_5)^{-t+1}\Theta_1\leq 2\Theta_1(1+C_5)^{-t}$. %proving \pref{eq:conv-bilinear}.
	 When $\beta>0$, \pref{eq: key recursion} is of the form specified in \pref{lem: recursion lemma p > 1} with $p=\beta$ and $q = C'$. Note that the second required condition is satisfied: $C'(\beta+1) \Theta_1^\beta \leq \frac{\beta+1}{2\cdot 4^\beta}\leq 1$. Therefore, by the conclusion of \pref{lem: recursion lemma p > 1}, 
	 \begin{align*}
	     \Theta_t 
	     &\leq \max\left\{\Theta_1, \left(\frac{2}{C'\beta}\right)^{\frac{1}{\beta}}\right\}t^{-\frac{1}{\beta}} = \max\left\{\Theta_1, \left(\frac{2\cdot 2^\beta}{C_5\beta}\right)^{\frac{1}{\beta}}, 4\Theta_1\left(\frac{4}{\beta}\right)^{\frac{1}{\beta}}\right\}t^{-\frac{1}{\beta}} \\
	     &\leq  \left[\left(1+4\left(\frac{4}{\beta}\right)^{\frac{1}{\beta}}\right)\Theta_1 + 2\left(\frac{2}{C_5\beta}\right)^{\frac{1}{\beta}}\right]t^{-\frac{1}{\beta}}. 
	 \end{align*}
	 \pref{eq:original_goal} is then proven by noticing that $\Theta_1=\dist(\zp_1, \calZ^*)$.
%	
	%\paragraph{Case 2. \typetwo.}  In this case, if $\zp_{t+1}\notin \calZ^*$, then we have
	%\begin{align*}
	%    \zeta_t &\geq \frac{1}{2}\|\zp_{t+1}-\z_t\|^2 + \frac{1}{2} \left( \|\zp_{t+1}-\z_t\|^2 + \|\z_t-\zp_t\|^2 \right)\\
	%&\geq \frac{1}{2}\|\zp_{t+1}-\z_t\|^2 +\frac{16\eta^2}{81} \frac{\left[F(\zp_{t+1})^\top (\zp_{t+1}-\Pi_{\calZ^*}(\zp_{t+1}))\right]_+^2}{\|\zp_{t+1}-\Pi_{\calZ^*}(\zp_{t+1})\|^2}   \tag{ \pref{lem: sufficient decrease} with $\z'=\Pi_{\calZ^*}(\zp_{t+1})$} \\
	%&\geq   \frac{1}{2}\|\zp_{t+1}-\z_t\|^{2} + \frac{16\eta ^2C^2}{81} \|\zp_{t+1}-\Pi_{\calZ^*}(\zp_{t+1})\|^{2(\beta+1)}  \tag{\typetwo condition}
	%\end{align*}
	%The inequality above also trivially holds when $\zp_{t+1}\in \calZ^*$ since in this case $\|\zp_{t+1}-\Pi_{\calZ^*}(\zp_{t+1})\|= 0$. 
	%We have thus arrived at the exact same bound as in \textbf{Case 1}, and the rest of the proof is identical.
	\end{proof}

}

%% file: appendixI.tex
\section{Proof of \pref{thm: bilinear-general}}\label{app:proofthm9}

\begin{proof}[Proof of \pref{thm: bilinear-general}]
     Consider the following $2\times 2$ bilinear game with curved feasible sets: 
     \begin{align*}
          f(\x,\y) &= \x^\top \G \y = 
          \begin{bmatrix}
               {x_1}  & {x_2}
          \end{bmatrix}
          \begin{bmatrix}
               0 & -1 \\
               1  & 0
          \end{bmatrix}
          \begin{bmatrix}
               {y_1} \\ {y_2}
          \end{bmatrix},   \\
          \calX &= \left\{ \x: \quad 0\leq {x_1}\leq \frac{1}{2}, \quad 0\leq {x_2}\leq \frac{1}{4}, \quad  {x_2} \geq {x_1}^2 \right\}, \\
          \calY &= \left\{ \y: \quad  0\leq {y_1} \leq \frac{1}{2}, \quad 0\leq {y_2}\leq \frac{1}{4}, \quad  {y_2} \geq {y_1}^2 \right\}.
     \end{align*}
     Below, we use \textbf{Claim 1} - \textbf{Claim 5} to argue that if the two players start from $\x_0=\y_0=\xp_0=\yp_0=(\frac{1}{2}, \frac{1}{4})$, and use any constant learning rate $\eta\leq \frac{1}{64}$, then the convergence is sublinear in the sense that $\|\z_{t}-\z^*\| \geq \Omega(1/t)$. \modify{Then, in \textbf{Claim 6}, we show that in this example, \rsi holds with $\beta=3$. }

     \paragraph{Claim 1. }\label{claim:one} The unique equilibrium is $\x^*=\mathbf{0}$, $\y^*=\mathbf{0}$. 
%     \begin{proof}[Proof of Claim 1. ]

          When $\x=\mathbf{0}$, clearly $\max_{\y'\in\calY} f(\x,\y')=0$. When $\x\neq \mathbf{0}$, we prove $\max_{\y'\in\calY} f(\x,\y')>0$ below. If ${x_1}\neq 0$, we let ${y'_1}=\frac{1}{2}{x_1}$ and ${y'_2}=\frac{1}{4}{x_1}^2$ (which satisfies $\y'\in\calY$), and thus
          \begin{align*}
             f(\x,\y') &= {x_2}{y'_1} - {x_1}{y'_2} 
             = {x_1}^2 \cdot \frac{1}{2}{x_1} - {x_1}\cdot \frac{1}{4}{x_1}^2 
             = \frac{1}{4}{x_1}^3 > 0. 
          \end{align*}
          If ${x_1}=0$ but ${x_2}\neq 0$, we let ${y'_1}=\frac{1}{2}, {y'_2}=\frac{1}{4}$, and thus
          \begin{align*}
             f(\x,\y') &= {x_2}{y'_1} - {x_1}{y'_2} = \frac{1}{2} {x_2} > 0. 
          \end{align*}
          Thus, $\max_{\y'\in\calY} f(\x,\y')>0$ if $\x\neq \mathbf{0}$, and $\x^*=\mathbf{0}$ is the unique optimal solution for $\x$. By the symmetry between $\x$ and $\y$ (because $\G=-\G^\top$), we can also prove that the unique optimal solution for $\y$ is $\y^*=\mathbf{0}$. 
%     \end{proof}

%     Below, we show that if \alg starts from $\xp_1=x_1=\yp_1=y_1=(\frac{1}{2}, \frac{1}{4})$, then $\|\zp_t-z^*\|\geq \Omega(1/t)$.
     
     \paragraph{Claim 2.}\label{claim:two} Suppose that $\x_0=\y_0=\xp_0=\yp_0=(\frac{1}{2}, \frac{1}{4})$. Then, at any step $t\in [T]$, we have $\x_t = \y_t$ and $\xp_t=\yp_t$, and all $\x_t, \y_t, \xp_t, \yp_t$ belong to $\{\bu\in \mathbb{R}^2:  u_{2}=u_{1}^2\}$.  \\

     	We prove this by induction. The base case trivially holds. Suppose that for step $t$, we have $\x_t = \y_t$, $\xp_t=\yp_t$, and $\x_t, \y_t, \xp_t, \yp_t\in\{\bu\in \mathbb{R}^2:  u_{2}=u_{1}^2\}$. Then consider step $t+1$. According to the dynamic of \alg, we have
     	\begin{align}
     		&\xp_{t+1} = \Pi_{\calX}\left\{\xp_t-\eta
     		\begin{bmatrix}
     		-y_{t,2} \\ y_{t,1}
     		\end{bmatrix}\right\} = \Pi_{\calX}\left\{\begin{bmatrix}
     		\xpp_{t,1}+\eta y_{t,2} \\ \xpp_{t,2}-\eta y_{t,1}
     		\end{bmatrix}
     		\right\} \label{eq:xp_t+1},\\
     		&\x_{t+1} = \Pi_{\calX}\left\{\xp_{t+1}-\eta
     		\begin{bmatrix}
     		-y_{t,2} \\ y_{t,1}
     		\end{bmatrix}\right\} = \Pi_{\calX}\left\{\begin{bmatrix}
     		\xpp_{t+1,1}+\eta y_{t,2}, \\ \xpp_{t+1,2}-\eta y_{t,1}
     		\end{bmatrix}
     		\right\} ,\nonumber\\
     		&\yp_{t+1} = \Pi_{\calY}\left\{\yp_t+\eta
     		\begin{bmatrix}
     		x_{t,2} \\ -x_{t,1}
     		\end{bmatrix}\right\} = \Pi_{\calY}\left\{\begin{bmatrix}
     		\ypp_{t,1}+\eta x_{t,2} \\ \ypp_{t,2}-\eta x_{t,1}
     		\end{bmatrix}
     		\right\} ,\nonumber\\
     		&\y_{t+1} = \Pi_{\calY}\left\{\yp_{t+1}+\eta
     		\begin{bmatrix}
     		x_{t,2} \\ -x_{t,1}
     		\end{bmatrix}\right\} = \Pi_{\calY}\left\{\begin{bmatrix}
     		\ypp_{t+1,1}+\eta x_{t,2} \\ \ypp_{t+1,2}-\eta x_{t,1}
     		\end{bmatrix}
     		\right\}.\nonumber
     	\end{align}
     	
     	According to induction hypothesis, we have $\xp_{t+1}=\yp_{t+1}$, which further leads to $\x_{t+1}=\y_{t+1}$.
     	
     	Now we prove that for any $\begin{bmatrix}
     		{x_1} \\ {x_2}
     	\end{bmatrix}$ such that ${x_1}\ge 0$, ${x_2}\le \frac{1}{4}$ and ${x_2}<{x_1}^2$, $\begin{bmatrix}
     	\overline{x}_{1} \\ \overline{x}_{2}
     	\end{bmatrix}=\Pi_{\calX}\left\{\begin{bmatrix}
     	{x_1} \\ {x_2}
     	\end{bmatrix}\right\}$ satisfies that $\overline{x}_{1}^2=\overline{x}_{2}$. Otherwise, suppose that $\overline{x}_{1}^2<\overline{x}_{2}$. Then according to the intermediate value theorem, there exists $\begin{bmatrix}
     	\widetilde{x}_{1} \\ \widetilde{x}_{2}
     	\end{bmatrix}$ that lies in the line segment of $\begin{bmatrix}
     	{x}_{1} \\ {x}_{2}
     	\end{bmatrix}$ and $\begin{bmatrix}
     	\overline{x}_{1} \\ \overline{x}_{2}
     	\end{bmatrix}$ such that $\widetilde{x}_{1}^2=\widetilde{x}_{2}$. Moreover, as ${x_1}\ge 0$, $\widetilde{x}_{1}\ge 0$, ${x_2}\le \frac{1}{4}$, $\widetilde{x}_{2}\le\frac{1}{4}$, we know that $\begin{bmatrix}
     	\widetilde{x}_{1} \\ \widetilde{x}_{2}
     	\end{bmatrix}\in \calX$. Therefore, we have $\|\widetilde{\x}-\x\|< \|\overline{\x}-\x\|$, which leads to contradiction.  
     	
     	Now consider $\xp_{t+1}$. According to induction hypothesis, we have $(\xpp_{t,1}+\eta y_{t,2})^2\ge\xpp_{t,1}^2=\xpp_{t,2}\ge \xpp_{t,2}-\eta y_{t,1}$. If equalities hold, trivially we have $\xpp_{t+1,1}^2=\xpp_{t,1}^2 = \xpp_{t,2} = \xpp_{t+1,2}$ according to \pref{eq:xp_t+1}. Otherwise, as $\xpp_{t,1}+\eta y_{t,2}\ge 0$, $\xpp_{t,2}-\eta y_{t,1}\le \frac{1}{4}$, according to the analysis above, we also have $\xpp_{t+1,1}^2=\xpp_{t+1,2}$. Applying similar analysis to $\yp_{t+1}$, $\x_{t+1}$ and $\y_{t+1}$ finishes the induction proof.

 	 \paragraph{Claim 3.}\label{claim:three} 
 	 With $\eta\le \frac{1}{64}$, the following holds for all $t\geq 1$,
 	 \begin{align}
 	 		& x_{t,1} \in \left[\frac{1}{2}\xpp_{t,1}, 2\xpp_{t,1}\right] \label{eq:induct-stab},\\
 	 		& \xpp_{t,1}\in\left[\xpp_{t-1,1}-4\eta \xpp_{t-1,1}^2, \xpp_{t-1,1}+4\eta \xpp_{t-1,1}^2\right] \label{eq:induct-decr}.
 	 \end{align}

 	 	We prove the claim by induction on $t$.  	 		
  		The case $t=1$ trivially holds. Suppose that \pref{eq:induct-stab} and \pref{eq:induct-decr} hold at step $t$. Now consider step $t+1$. 
  		
  		\paragraph{Induction to get \pref{eq:induct-decr}. }
  		%We already know that $\xp_{t+1,2}=\xpp_{t+1,1}^2$ by Claim 2. 
  		According to Claim 2, we have 
  		\begin{align*}
  			\xp_{t+1} = \Pi_{\calX}\left\{\xp_t-\eta
  			\begin{bmatrix}
  			-y_{t,2} \\ y_{t,1}
  			\end{bmatrix}\right\} = \Pi_{\calX}\left\{\begin{bmatrix}
  			\xpp_{t,1}+\eta x_{t,1}^2 \\ \xpp_{t,1}^2-\eta x_{t,1}
  			\end{bmatrix}
  			\right\},
   		\end{align*}
   		and $\xp_{t+1} = (u, u^2)$ for some $u \in [0, 1/2]$.
  		Using the definition of the projection function, we have
  		\begin{align*}
  			\xpp_{t+1,1}= \argmin_{u\in[0,\frac{1}{2}]}\left\{\left(\xpp_{t,1}+\eta x_{t,1}^2-u\right)^2+\left(\xpp_{t,1}^2-\eta x_{t,1}-u^2\right)^2\right\}\triangleq \argmin_{u\in[0,\frac{1}{2}]}g(u).
  		\end{align*}
  		
  		Now we show that $\argmin_{u\in [0, \frac{1}{2}]}g(u)=\argmin_{u\in \mathbb{R}}g(u)$. Note that
  		\begin{align}\label{eq: nabla g}
                 \nabla g(u) = 2(u-\xpp_{t,1}-\eta x_{t,1}^2)+4u\left(u^2+\eta x_{t,1}-\xpp_{t,1}^2\right),
  		\end{align}
  		
  		Therefore, when $u>\frac{1}{2}$, using $x_{t,1}\le \frac{1}{2}$, we have  		
		\begin{align}\label{eq: ge1/2}
  			\nabla g(u) > -2\eta x_{t,1}^2+2\eta x_{t,1} \ge 0,
  		\end{align}
  		which means $g(u)>g(\frac{1}{2})$. 
  		On the other hand, when $u<0$, using $\xpp_{t,1}\le \frac{1}{2}$, we have 
  		\begin{align}\label{eq: le0}
  			\nabla g(u) < 2u-4u\xpp_{t,1}^2 \le u < 0,
  		\end{align}
  		which means $g(u)>g(0)$.
		Combining \pref{eq: ge1/2} and \pref{eq: le0}, we know that $\argmin_{u\in [0, \frac{1}{2}]}g(u)=\argmin_{u\in \mathbb{R}}g(u)$.
		Therefore, $\xpp_{t+1,1}$ is the unconstrained minimizer of convex function $g(u)$, which means $\nabla g(\xpp_{t+1,1})=0$.
  		%Note that for $u>\frac{1}{2}\ge \xpp_{t,1}$, we have \begin{align*}
  		%	\nabla g(u) &= 2(u-\xpp_{t,1}-\eta x_{t,1}^2)+4u\left(u^2+\eta x_{t,1}-\xpp_{t,1}^2\right) \\
  		%	&> -2\eta x_{t,1}^2+2\eta x_{t,1} \tag{$u>\frac{1}{2}\ge \xpp_{t,1}$} \\
  		%	&\ge 0.
  		%\end{align*}
  		%Therefore if $\nabla g(u)=0$, we have $u\le\frac{1}{2}$, which means $\nabla g(\xpp_{t+1,1})=0$. 
  		Below we use contradiction to prove that $\xpp_{t+1,1}\geq \xpp_{t,1}-4\eta \xpp_{t,1}^2$. 
  		If $\xpp_{t+1,1}<\xpp_{t,1}-4\eta \xpp_{t,1}^2$, we use \pref{eq: nabla g} and get  
		\begin{align*}
		\nabla g(\xpp_{t+1,1})&=2(\xpp_{t+1,1}-\xpp_{t,1}-\eta x_{t,1}^2)+4\xpp_{t+1,1}\left(\xpp_{t+1,1}^2+\eta x_{t,1}-\xpp_{t,1}^2\right)\\
		&< 2(-4\eta\xpp_{t,1}^2-\eta x_{t,1}^2)+4\xpp_{t+1,1}\left(\eta x_{t,1}-8\eta\xpp_{t,1}^3+16\eta^2\xpp_{t,1}^4\right) \\
		&\le -\frac{17}{2}\eta\xpp_{t,1}^2 +4\xpp_{t+1,1}\left(2\eta\xpp_{t,1}-8\eta\xpp_{t,1}^3+16\eta^2\xpp_{t,1}^4\right) \tag{\pref{eq:induct-stab}} \\
		&\le  -\frac{17}{2}\eta\xpp_{t,1}^2 +4\xpp_{t+1,1}\left(2\eta\xpp_{t,1}+16\eta^2\xpp_{t,1}^4\right)\\
		&\le -\frac{17}{2}\eta\xpp_{t,1}^2 +4(\xpp_{t,1}-4\eta\xpp_{t,1}^2)\left(2\eta\xpp_{t,1}+16\eta^2\xpp_{t,1}^4\right) \tag{$\xpp_{t+1,1}<\xpp_{t,1}-4\eta \xpp_{t,1}^2$}\\
		&= -\frac{1}{2}\eta\xpp_{t,1}^2+64\eta^2\xpp_{t,1}^5-32\eta^2\xpp_{t,1}^3-256\eta^3\xpp_{t,1}^6 \\
		&\le -\frac{1}{2}\eta\xpp_{t,1}^2-16\eta^2\xpp_{t,1}^3-256\eta^3\xpp_{t,1}^6 \tag{$\xpp_{t,1}\le \frac{1}{2}$}\\
		&\le 0,
		\end{align*} 
  		which leads to contradiction. 
  		%Therefore, we know that $\xpp_{t+1,1}\ge\xpp_{t,1}-4\eta \xpp_{t,1}^2$. 
  		Similarly, if $\xpp_{t+1,1}>\xpp_{t,1}+4\eta \xpp_{t,1}^2$, we have
  		\begin{align*}
  		\nabla g(\xpp_{t+1,1})&=2(\xpp_{t+1,1}-\xpp_{t,1}-\eta x_{t,1}^2)+4\xpp_{t+1,1}\left(\xpp_{t+1,1}^2+\eta x_{t,1}-\xpp_{t,1}^2\right)\\
  		&> 2(4\eta\xpp_{t,1}^2-\eta x_{t,1}^2)+4\xpp_{t+1,1}\left(\eta x_{t,1}+8\eta\xpp_{t,1}^3+16\eta^2\xpp_{t,1}^4\right) \\
  		&\ge 0. \tag{\pref{eq:induct-stab}} 
  		\end{align*} 
  	 The calculations above conclude that 
  	 \begin{align}\label{eq: induction-stab1}
  	 	\xpp_{t+1,1}\in\left[\xpp_{t,1}-4\eta \xpp_{t,1}^2, \xpp_{t,1}+4\eta \xpp_{t,1}^2\right].
  	 \end{align}
  	 
  	 \paragraph{Induction to get \pref{eq:induct-stab}. }
  	 Similarly, we have 
  	 \begin{align*}
  	 x_{t+1,1}= \argmin_{u\in[0,\frac{1}{2}]}\left\{\left(\xpp_{t+1,1}+\eta x_{t,1}^2-u\right)^2+\left(\xpp_{t+1,1}^2-\eta x_{t,1}-u^2\right)^2\right\}\triangleq \argmin_{u\in[0,\frac{1}{2}]}h(u),
  	 \end{align*} 
       \begin{align*}
            \nabla h(u) = 2(u-\xpp_{t+1,1} - \eta x_{t,1}^2) + 4u(u^2 + \eta x_{t,1} - \xpp_{t+1,1}^2), 
       \end{align*}
  	 and $\nabla h(x_{t+1,1})=0$. If $x_{t+1,1}<\frac{1}{2}\xpp_{t+1,1}$, we have
  	 \begin{align*}
  	 \nabla h(x_{t+1,1}) &=2(x_{t+1,1}-\xpp_{t+1,1}-\eta x_{t,1}^2)+4x_{t+1,1}\left(x_{t+1,1}^2+\eta x_{t,1}-\xpp_{t+1,1}^2\right) \\
  	 &< -\xpp_{t+1,1}-2\eta x_{t,1}^2-3x_{t+1,1}\xpp_{t+1,1}^2+2\eta\xpp_{t+1,1}x_{t,1} \tag{$x_{t+1,1}<\frac{1}{2}\xpp_{t+1,1}$} \\
  	 &\le 0. \tag{$\eta\le\frac{1}{64},x_{t,1}\le\frac{1}{2}$}
  	 \end{align*}
  	 If $x_{t+1,1} > 2\xpp_{t+1,1}$, we also have
  	 \begin{align*}
  	 \nabla h(x_{t+1,1}) &=2(x_{t+1,1}-\xpp_{t+1,1}-\eta x_{t,1}^2)+4x_{t+1,1}\left(x_{t+1,1}^2+\eta x_{t,1}-\xpp_{t+1,1}^2\right) \\
  	 &> 2\xpp_{t+1,1}-2\eta x_{t,1}^2+24\xpp_{t+1,1}^3+8\eta\xpp_{t+1,1}x_{t,1} \tag{$x_{t+1,1}>2\xpp_{t+1,1}$} \\
  	 &\ge 2\xpp_{t+1,1}-2\eta x_{t,1}^2+24\xpp_{t+1,1}^3+8\eta (\xpp_{t,1}-4\eta\xpp_{t,1}^2)x_{t,1} \tag{\pref{eq: induction-stab1}}\\
  	 &\ge 2\xpp_{t+1,1}-2\eta x_{t,1}^2+24\xpp_{t+1,1}^3+8\eta (\tfrac{1}{2}x_{t,1}-4\eta\xpp_{t,1}^2)x_{t,1} \tag{\pref{eq:induct-stab}} \\
  	 &=2\xpp_{t+1,1}+2\eta x_{t,1}^2+24\xpp_{t+1,1}^3-32\eta^2\xpp_{t,1}^2x_{t,1} \\
  	 &\ge 2\xpp_{t+1,1}+\frac{1}{4}\eta \xpp_{t,1}^2+24\xpp_{t+1,1}^3-32\eta^2\xpp_{t,1}^2x_{t,1} \tag{\pref{eq:induct-stab}}\\
  	 & \ge 0. \tag{$\eta\le \frac{1}{64}, x_{t,1}\le \frac{1}{2}$}
  	 \end{align*}
  	 Both lead to contradiction. 
  	 Therefore, we conclude that $x_{t+1}\in [\frac{1}{2}\xpp_{t+1,1},2\xpp_{t+1,1}]$, which finishes the induction proof.
 
 	 \paragraph{Claim 4. }  $x_{t,1}\ge \xpp_{t,1}-4\eta \xpp_{t,1}^2$, for all $t\ge 1$.
 	 
 	 The case $t=1$ holds trivially. For $t\ge 2$, we prove this by contradiction. Using the definition of the projection function, we have:
 	 
 	 \begin{align*}
 	 	x_{t+1,1} = \argmin_{u\in \left[0,\frac{1}{2}\right]}\left\{\left(\xpp_{t+1,1}+\eta x_{t,1}^2-u\right)^2+\left(\xpp_{t+1,1}^2-\eta x_{t,1}-u^2\right)^2\right\}\triangleq \argmin_{u\in \left[0, \frac{1}{2}\right]}h(u).
 	 \end{align*}
 	 
 	 Similar to the analysis in \textbf{Claim 3}, we have $\argmin_{u\in \left[0, \frac{1}{2}\right]}h(u) = \argmin_{u\in \mathbb{R}}h(u)$, which means that $\nabla h(x_{t+1,1})=0$. Note that $\eta\le \frac{1}{64}$ and $0\le\xpp_{t,1}\le \frac{1}{2}$, according to  \pref{eq:induct-stab} and \pref{eq:induct-decr}, we have
 	 \begin{align*}
 	 	\xpp_{t+1,1}\in\left[\xpp_{t,1}-4\eta \xpp_{t,1}^2, \xpp_{t,1}+4\eta \xpp_{t,1}^2\right]\subseteq \left[\frac{31}{32}\xpp_{t,1}, \frac{33}{32}\xpp_{t,1}\right],
 	 \end{align*}
  which means that
  	 \begin{align}
 	 	x_{t,1}\in \left[\frac{1}{2}\xpp_{t,1}, 2\xpp_{t,1}\right]\subseteq \left[\frac{16}{33}\xpp_{t+1,1}, \frac{64}{31}\xpp_{t+1,1}\right].\label{eq:x stab}
 	 \end{align}
 	 
 	 If $x_{t+1,1} < \xpp_{t+1,1} - 4\eta\xpp_{t+1,1}^2$, we show that $\nabla h(x_{t+1,1}) < 0$. In fact,
 	 
 	 \begin{align*}
 	 	\nabla h(x_{t+1,1})&=2(x_{t+1,1}-\xpp_{t+1,1}-\eta x_{t,1}^2)+4x_{t+1,1}\left(x_{t+1,1}^2+\eta x_{t,1}-\xpp_{t+1,1}^2\right)\\
 	 	&< 2(-4\eta\xpp_{t+1,1}^2-\eta x_{t,1}^2)+4x_{t+1,1}\left(\eta x_{t,1}-8\eta\xpp_{t+1,1}^3+16\eta^2\xpp_{t+1,1}^4\right) \\
 	 	&\le -\frac{42}{5}\eta\xpp_{t+1,1}^2 +4x_{t+1,1}\left(\frac{64}{31}\eta\xpp_{t+1,1}-8\eta\xpp_{t+1,1}^3+16\eta^2\xpp_{t+1,1}^4\right) \tag{\pref{eq:x stab}} \\
 	 	&\le -\frac{42}{5}\eta\xpp_{t+1,1}^2 +4x_{t+1,1}\left(\frac{64}{31}\eta\xpp_{t+1,1}+16\eta^2\xpp_{t+1,1}^4\right)\\
 	 	&< -\frac{42}{5}\eta\xpp_{t+1,1}^2 +4(\xpp_{t+1,1}-4\eta\xpp_{t+1,1}^2)\left(\frac{64}{31}\eta\xpp_{t+1,1}+16\eta^2\xpp_{t+1,1}^4\right)\\
 	 	&\le 64\eta^2\xpp_{t+1,1}^5-32\eta^2\xpp_{t+1,1}^3-256\eta^3\xpp_{t+1,1}^6 \\
 	 	&\le -16\eta^2\xpp_{t,1}^3-256\eta^3\xpp_{t,1}^6 \tag{$\xpp_{t,1}\le \frac{1}{2}$}\\
 	 	&\le 0,
 	 \end{align*} 
 	 which leads to contradiction. Therefore, we show that $x_{t,1}\ge \xpp_{t,1}-4\eta \xpp_{t,1}^2$ for all $t\ge 1$.
 	 
  	 \paragraph{Claim 5. }  If $\eta\le\frac{1}{64}$, we have $\|\z_t-\z^*\|\geq \Omega(1/t)$.

  	 Now we are ready to prove $\|\z_t-\z^*\|\ge\Omega(1/t)$. First we show $\xp_{t,1}\ge \frac{1}{2t}$ for all $t\ge 1$ by induction. The case $t=1$ trivially holds. Suppose that it holds at step $t$. Considering step $t+1$, we have
  	 \begin{align*}
  	 	\xpp_{t+1,1}&\ge \xpp_{t,1}-4\eta\xpp_{t,1}^2 \tag{\textbf{Claim 3}} \\
  	 	&\ge \xpp_{t,1}-\frac{1}{16}\xpp_{t,1}^2 \tag{$\eta\le\frac{1}{64}$} \\
  	 	& \ge \frac{1}{2t}-\frac{1}{64t^2} \tag{$\frac{1}{2t}\le \xpp_{t,1}\le \frac{1}{2}$, and $x-\frac{1}{16}x^2$ is increasing when $x\le 8$}\\
  	 	&\ge \frac{1}{2(t+1)}. \tag{$t\ge 1$}
  	 \end{align*}
   	 Therefore, $\xpp_{t,1}\ge \frac{1}{2t}$, $\forall t\ge 1$. This, by \textbf{Claim 4} and the analysis above, shows that
   	 \begin{align*}
   	 	x_{t,1}&\ge \xpp_{t,1}-4\eta\xpp_{t,1}^2 \ge \frac{1}{2(t+1)}.
   	 \end{align*}
  	 Note that according to \textbf{Claim 1}, $\x^*=\mathbf{0}$. Therefore, we have $\|\z_t-\z^*\|\ge x_{t,1}\ge\frac{1}{2(t+1)}$, which finishes the proof.   
  	 
  	 \modify{
  	 \paragraph{Claim 6.} In this example, \rsi holds with $\beta=3$. This can be seen by the following:

      \begin{align*}
       \max_{\z'\in\calZ} \frac{F(\z)^\top (\z-\z')}{\|\z-\z'\|} 
       &\geq  \max_{\z'\in\calZ} F(\z)^\top (\z-\z') \\
       &= \max_{\x'\in\calX, \y'\in\calY}  \left\{\x^\top \G\y' - \x'^\top \G\y\right\} \\
       &= \max_{\x'\in\calX, \y'\in\calY}  \left\{-x_1 y_2' + x_2 y_1' + x_1' y_2 - x_2' y_1\right\} \\
       &\geq -x_1x_2^2 + x_2^2 + y_2^2 - y_2^2 y_1   \tag{picking $y_1'=x_2, y_2'=x_2^2, x_1'=y_2, x_2'=y_2^2$ }   \\
       &\geq \frac{1}{2}x_2^2 + \frac{1}{2} y_2^2   \tag{$x_1, y_1\leq \frac{1}{2}$}  \\
       &\geq \frac{1}{4}\left(x_1^4 + x_2^4 + y_1^4 + y_2^4\right)   \tag{$x_2\geq \{x_1^2, x_2^2\}, y_2\geq \{y_1^2, y_2^2\}$} \\
       &\geq \frac{1}{16}\left(x_1^2 + x_2^2 + y_1^2 + y_2^2\right)^2 \tag{Cauchy-Schwarz} \\
       &= \frac{1}{16}\|\z-\z^*\|^4,     \tag{$\z^*=(0,0,0,0)$}
  	 \end{align*}
  	 which implies $\beta=3$. 
  	 }
\end{proof}